\documentclass[twoside]{article}

\usepackage[accepted]{aistats2023}

\usepackage[round]{natbib}
\usepackage[utf8]{inputenc} %
\usepackage[T1]{fontenc}    %
\usepackage[colorlinks=true,linkcolor=blue,citecolor=blue]{hyperref} 
\usepackage{multirow}
\usepackage{arydshln}
\usepackage{url}            %
\usepackage{booktabs}       %
\usepackage{amsfonts}       %
\usepackage{nicefrac}       %
\usepackage{microtype}      %

\usepackage{rotating}

\usepackage{amsmath, amssymb, amsthm}
\usepackage{dsfont}

\usepackage{algorithm}
\usepackage{algorithmic}

\usepackage{bbm}
\usepackage{graphicx}
\usepackage{xcolor}
\usepackage{enumitem}
\usepackage{dsfont}

\newcommand{\E}{\mathbb{E}}
\newcommand{\R}{\mathbb{R}}
\newcommand{\N}{\mathbb{N}}

\newcommand{\xb}{\mathbf{x}}
\newcommand{\Xb}{\mathbf{X}}
\newcommand{\zb}{\mathbf{z}}
\newcommand{\Zb}{\mathbf{Z}}
\newcommand{\Ub}{\mathbf{U}}
\newcommand{\Rb}{\mathbf{R}}
\newcommand{\Kb}{\mathbf{K}}
\newcommand{\Bb}{\mathbf{B}}
\newcommand{\Ab}{\mathbf{A}}
\newcommand{\Ib}{\mathbf{I}}
\newcommand{\Qb}{\mathbf{Q}}
\newcommand{\Wb}{\mathbf{W}}
\newcommand{\Vb}{\mathbf{V}}
\newcommand{\Lamb}{\mathbf{\Lambda}}
\newcommand{\wb}{\mathbf{w}}
\newcommand{\Ufair}{\mathbf{U}_{\text{fair}}}
\newcommand{\Ust}{\mathbf{U}_{\text{st}}}
\newcommand{\Idk}{\mathbf{I}_{k\times k}}
\newcommand{\nullb}{\mathbf{0}}

\makeatletter
\newcommand*{\transpose}{%
  {\mathpalette\@transpose{}}%
}
\newcommand*{\@transpose}[2]{%
  \raisebox{\depth}{$\m@th#1\intercal$}%
}
\makeatother

\DeclareMathOperator{\argmin}{argmin}
\DeclareMathOperator{\argmax}{argmax}

\newcommand{\charfct}{\mathds{1}}
\DeclareMathOperator{\Psymb}{Pr}
\DeclareMathOperator{\trace}{trace}
\DeclareMathOperator{\rank}{rank}
\DeclareMathOperator{\Var}{Var}
\DeclareMathOperator{\Cov}{Cov}

\newtheorem{proposition}{Proposition}

\begin{document}

\runningauthor{Matthäus Kleindessner, Michele Donini, Chris Russell, Muhammad Bilal Zafar }

\twocolumn[

\aistatstitle{Efficient 
fair PCA for fair representation learning}

\aistatsauthor{ Matthäus Kleindessner \And Michele Donini }

\aistatsaddress{ Amazon Web Services\\ Tübingen, Germany \And  Amazon Web Services\\ Berlin, Germany }

\aistatsauthor{ Chris Russell \And Muhammad Bilal Zafar }

\aistatsaddress{  Amazon Web Services\\ Tübingen, Germany \And Amazon Web Services\\ Berlin, Germany}

]

\begin{abstract}

We revisit the problem of fair principal component analysis (PCA), where the 
goal is to learn the 
best 
low-rank 
linear approximation 
of the data that obfuscates 
 demographic information. 
We propose a conceptually 
simple approach that allows for an analytic solution 
similar to standard PCA and can be kernelized.  
Our methods have the same complexity as standard PCA, or kernel PCA,  
and 
run much faster than existing methods for fair PCA 
based on semidefinite programming or manifold optimization, while achieving similar~results.

\end{abstract}

\newcommand{\abstparagraph}{2pt}

\section{INTRODUCTION}\label{sec:introduction}

Over the last decade, fairness in machine learning  \citep{barocas-hardt-narayanan} has become an established field. 
Numerous 
definitions of fairness, 
and algorithms trying to 
satisfy these,  
have been proposed. In 
the context of 
classification, two of the most prominent 
fairness 
notions are 
demographic parity \citep[DP;][]{kamiran2011} and equality of opportunity \citep[EO;][]{hardt2016equality}. DP   requires a classifier's prediction to be independent of a datapoint's demographic attribute (such as a person's gender or race), 
and EO requires the prediction to be independent of the 
attribute given that the 
datapoint's 
ground-truth label is positive. 
Formally, 
in the case of binary classification, 
\begin{align}\label{eq:fairness_definitions}
\begin{split}
\text{DP:}~~\Psymb(\hat{Y}=1|Z=z)&=\Psymb(\hat{Y}=1),\\
\text{EO:}\Psymb(\hat{Y}=1|Z=z,Y=1)&=\Psymb(\hat{Y}=1|Y=1),
\end{split}
\end{align}
where $\Psymb$ is a probability distribution over 
random variables~$Y,\hat{Y}\in\{0,1\}$ and $Z\in\mathcal{Z}$, with $Y$ representing the ground-truth label, $\hat{Y}$
representing 
the classifier's prediction and
$Z$ representing the demographic attribute.

An appealing approach to satisfy DP or EO is fair representation learning (e.g., \citealp{zemel2013}; see Section~\ref{sec:related_work} for related work): let $X\in\mathcal{X}$ denote a random vector representing features based on which predictions are made. The idea 
of fair representation learning 
is to learn a \emph{fair} 
feature 
representation $f:\mathcal{X}\rightarrow \mathcal{X}'$~such~that $f(X)$ is 
(approximately) 
independent of the demographic attribute~$Z$ (conditioned on $Y=1$ if one aims to satisfy EO). Once a fair representation is found, any model trained on this representation will 
also 
be fair. Of course, the representation still needs  to contain 
some 
information about $X$ in order~to~be~useful. 

Leaving fairness aside, one of the most prominent methods for 
representation learning (in its special form of dimensionality reduction)
is principal component analysis \citep[PCA; e.g.,][]{shalev2014understanding}. PCA projects the data onto a linear subspace such that the approximation error is minimized.
The key idea of our paper is to alter PCA such that it 
gives 
a fair representation. 
This idea is not new: \citet{olfat2019} and \citet{Lee2022} already proposed formulations of fair PCA that aim for the same goal. We discuss the differences between our paper and these works in detail in Section~\ref{sec:related_work}. In short, the differences are twofold: (i)~while the goal is the same, the derivations are different, and we consider our derivation to be simpler and more intuitive. (ii)~the different derivations lead to different algorithms, with our 
main 
algorithm being very similar to standard PCA. While our formulation allows for an analytical solution by means of eigenvector computations,  the methods by \citeauthor{olfat2019} and \citeauthor{Lee2022} rely on semidefinite programming 
or 
manifold optimization.
While our algorithm can be implemented in a few lines 
of code 
and runs very fast, with the same complexity as standard PCA, their algorithms rely on 
specialized libraries and suffer from a huge running time. 
We believe that 
because of these advantages 
our new 
derivation of fair PCA 
and 
our proposed approach 
add value
to the existing 
literature.

\paragraph{Outline} 
In Section~\ref{sec:methods}, we first review 
PCA and then derive our formulation of  fair PCA. We discuss extensions and variants, including a kernelized version, in Section~\ref{sec:extensions}. We provide a detailed discussion of related work in Section~\ref{sec:related_work}
and present extensive experiments in Section~\ref{sec:experiments}.
Some 
details and 
experiments are deferred to the appendix.

\paragraph{Notation}
For $n\in\N$, let $[n]=\{1,\ldots,n\}$. 
We generally denote scalars by non-bold letters, vectors by bold lower-case letters, and matrices by bold upper-case letters. 
All vectors $\xb\in\R^d\equiv\R^{d\times 1}$ are column vectors, except that we use $\nullb$ to denote 
both a column vector 
and a row vector 
(and also a matrix) 
of all zeros. Let $\xb^\transpose\in\R^{1\times d}$ be the transposed row vector of $\xb$.
We denote the Euclidean norm of $\xb$ by 
$\|\xb\|_2=\sqrt{\sum_{i} \xb_i^2}$. 
For a matrix~$\Xb\in\R^{d_1\times d_2}$, let $\Xb^\transpose\in\R^{d_2\times d_1}$ be 
its transpose.
$\Idk$ denotes the identity matrix of size~$k$. 
 For 
 $\Xb\in\R^{d\times d}$, let $\trace(\Xb)=\sum_{i=1}^d \Xb_{ii}$.

\section{FAIR PCA FOR FAIR REPRESENTATION LEARNING}\label{sec:methods}

We first review 
PCA and then 
derive 
our formulation of fair PCA. Our formulation is a relaxation of a strong constraint imposed on the PCA objective. We provide a natural interpretation of the relaxation and show that it is equivalent to the original constraint under a particular 
data~model.

\vspace{\abstparagraph}
\textbf{PCA}~~ 
We 
represent 
a dataset of $n$ points $\xb_1,\ldots,\xb_n\in \R^d$ 
as a matrix $\Xb\in\R^{d\times n}$, where the $i$-th column equals $\xb_i$. Given a target dimension~$k\in[d-1]$, 
PCA \citep[e.g.,][]{shalev2014understanding} finds 
a best-approximating projection 
of the dataset 
onto a $k$-dimensional linear subspace.  
That is, 
PCA 
 finds $\Ub\in\R^{d\times k}$ solving
 \begin{align}\label{eq:standard_PCA}
 \begin{split}
 &\argmin_{\Ub\in\R^{d\times k}:\, \Ub^\transpose\Ub=\Idk}\sum_{i=1}^n\|\xb_i-\Ub\Ub^\transpose\xb_i\|_2^2\\
 &~~~~\equiv \argmax_{\Ub\in\R^{d\times k}:\, \Ub^\transpose\Ub=\Idk} \trace(\Ub^\transpose\Xb\Xb^\transpose\Ub).
 \end{split}
 \end{align} 
 $\Ub^\transpose\xb_i\in \R^k$ is the projection of $\xb_i$ onto the subspace 
 spanned by the columns of $\Ub$ 
 viewed as a point in 
 the 
lower-dim 
 space 
 $\R^k$, and $\Ub\Ub^\transpose\xb_i\in \R^d$ is the projection viewed as a point in the original space~$\R^d$.  
 A solution~to~\eqref{eq:standard_PCA} is given by any 
 $\Ub$ 
 that comprises 
 as columns orthonormal eigenvectors, corresponding to the largest $k$ eigenvalues,~of~$\Xb\Xb^\transpose$.

\vspace{\abstparagraph}
\textbf{Our formulation of 
fair PCA}~~
In fair PCA, we aim to 
remove demographic information when projecting the dataset onto the $k$-dimensional linear subspace. We look for a best-approximating projection such that the projected data does not contain demographic information anymore: 
let $z_i\in\{0,1\}$ denote the demographic attribute of datapoint~$\xb_i$, which encodes membership in one of two demographic groups (we discuss how to extend our approach to multiple groups in Section~\ref{subsec:multiple_groups} 
and to multiple attributes in Section~\ref{subsec:multiple_dem_attributes}). Ideally, we would like that no classifier can predict $z_i$ when getting to see only the projection 
of $\xb_i$ onto the $k$-dimensional
subspace, 
that is 
we would want to solve
\begin{align}\label{eq:fair_PCA}
\begin{split}
& ~~~~~~~~~~\argmax_{\Ub\in\mathcal{U}} \trace(\Ub^\transpose\Xb\Xb^\transpose\Ub),\quad\text{where}\\
 &\mathcal{U}=\left\{\Ub\in\R^{d\times k}:\,\Ub^\transpose\Ub=\Idk~\text{and $\forall h: \R^k\rightarrow \R$,}\right.\\
 & ~~~~~~~\left.\text{ $h(\Ub^\transpose\xb_i)$ and $z_i$ are statistically independent}\right\}. 
\end{split}
\end{align} 

It is not hard to see, that for a given target dimension~$k$ the set $\mathcal{U}$ defined in \eqref{eq:fair_PCA} may be empty and hence Problem~\eqref{eq:fair_PCA} not well defined 
(see Appendix~\ref{app:not_well_defined} for an example). The reason is that linear projections are not flexible enough  to always remove all demographic information from a dataset.\footnote{
Also more powerful ``projections'' in methods for learning adversarially fair representations (cf. Section~\ref{sec:related_work}) have been found to fail 
removing all demographic information; that is, a sufficiently strong adversary 
can 
still predict demographic information from the supposedly fair representation \citep[e.g., ][]{balunovic2022}.
} 
As a remedy, we relax Problem~\eqref{eq:fair_PCA} by %
expanding 
the set~$\mathcal{U}$ in two ways: first, rather than preventing arbitrary functions~$h: \R^k\rightarrow \R$ from recovering $z_i$, we restrict our goal to linear functions of the form $h(\xb)= \wb^\transpose\xb+b$ 
(we provide a non-linear kernelized version of fair PCA in Section~\ref{subsec:kernelized_version} and another variant 
that can deal, 
to some extent, 
with non-linear~$h$ 
in Section~\ref{subsec:covariance_extension}); second, rather than requiring $h(\Ub^\transpose\xb_i)$ and $z_i$ to be independent, we only require 
the two variables 
to be uncorrelated, that is their covariance to be zero. This leaves us with the following problem:
\begin{align}\label{eq:fair_PCA_relaxed}
\begin{split}
& ~~~~~~~~~~\argmax_{\Ub\in\mathcal{U}'} \trace(\Ub^\transpose\Xb\Xb^\transpose\Ub),\quad\text{where}\\
 &\mathcal{U}'=\left\{\Ub\in\R^{d\times k}:\,\Ub^\transpose\Ub=\Idk~\text{and $\forall \wb\in\R^k,b\in\R$,}\right.\\
 & ~~~~~~~~~~~\text{ $\wb^\transpose\Ub^\transpose\xb_i+b$ and $z_i$ are uncorrelated, that is} \\
 & ~~~~~~~~~~~\left.\Cov(\wb^\transpose\Ub^\transpose\xb_i+b,z_i)=0\right\}.
\end{split}
\end{align} 

We show 
that Problem~\eqref{eq:fair_PCA_relaxed} is well defined. 
Conveniently, it can be solved analytically similarly to standard PCA: with $\bar{z}=\frac{1}{n} \sum_{i=1}^n z_i$ and $\zb=(z_1-\bar{z},\ldots,z_n-\bar{z})^\transpose \in\R^{n}$, 
\begin{align*}
\forall \wb\in\R^k,b\in\R: \text{$\wb^\transpose\Ub^\transpose\xb_i+b$ and $z_i$ are uncorr.}~\Leftrightarrow\\
\forall \wb\in\R^k,b\in\R: \sum_{i=1}^n (z_i-\bar{z})\cdot(\wb^\transpose\Ub^\transpose\xb_i+b)=0~\Leftrightarrow\\
\forall \wb: \wb^\transpose \Ub^\transpose\Xb\zb=0~\Leftrightarrow~
\Ub^\transpose\Xb\zb=\nullb~\Leftrightarrow~ \zb^\transpose\Xb^\transpose\Ub=\nullb.
\end{align*}
We assume that $\zb^\transpose\Xb^\transpose\neq \nullb$ (otherwise Problem~\eqref{eq:fair_PCA_relaxed} is the same as the standard PCA Problem~\eqref{eq:standard_PCA}). 
Let $\Rb\in\R^{d\times (d-1)}$ comprise as columns an orthonormal basis of 
the nullspace of $\zb^\transpose\Xb^\transpose$. 
Every $\Ub\in \mathcal{U}'$ can then be written as 
$\Ub=\Rb\Lamb$ for $\Lamb\in\R^{(d-1)\times k}$ with $\Lamb^\transpose\Lamb=\Idk$, and the objective of \eqref{eq:fair_PCA_relaxed} becomes $\trace(\Lamb^\transpose\Rb^\transpose\Xb\Xb^\transpose\Rb\Lamb)$, where we now maximize w.r.t.~$\Lamb$. The latter problem has exactly the form of \eqref{eq:standard_PCA} with $\Xb\Xb^\transpose$ replaced by $\Rb^\transpose\Xb\Xb^\transpose\Rb$, and we know that a solution is given by orthonormal eigenvectors, corresponding to the largest $k$ eigenvalues, of $\Rb^\transpose\Xb\Xb^\transpose\Rb$. Once we have 
$\Lamb$, we obtain a solution~$\Ub$ of \eqref{eq:fair_PCA_relaxed} by computing~$\Ub=\Rb\Lamb$. 
We summarize the procedure 
as our proposed formulation of fair PCA 
in Algorithm~\ref{alg:fair_PCA}.
Its 
running time 
is $\mathcal{O}(nd^2+d^3)$, which is the same as the running time of standard PCA.

\begin{algorithm}[t!]
   \caption{Fair PCA (for two demographic groups)
   }\label{alg:fair_PCA}
\begin{algorithmic}
   \STATE {\bfseries Input:} data matrix $\Xb\in\R^{d\times n}$; demographic attr.~$z_i\in\{0,1\}$, $i\in[n]$;
   target dimension
   $k\in[d-1]$

\vspace{1mm}
   \STATE {\bfseries Output:} a solution $\Ub$ to Problem~\eqref{eq:fair_PCA_relaxed}

   \begin{itemize}[leftmargin=*]
   \setlength{\itemsep}{-2pt}
   \item set $\zb=(z_1-\bar{z},\ldots,z_n-\bar{z})^\transpose$ with $\bar{z}=\frac{1}{n} \sum_{i=1}^n z_i$ 
\item compute an orthonormal basis of the nullspace of $\zb^\transpose\Xb^\transpose$ and build  matrix~$\Rb$ comprising the basis vectors as columns
  \item compute orthonormal eigenvectors, corresponding to the largest $k$ eigenvalues, of $\Rb^\transpose\Xb\Xb^\transpose\Rb$ and build matrix~$\Lamb$ comprising the eigenvectors as columns
  \item return $\Ub=\Rb\Lamb$
   \end{itemize}
\end{algorithmic}
\end{algorithm}

The derivation above yields a natural interpretation of the relaxed Problem~\eqref{eq:fair_PCA_relaxed}. It is easy to see that the condition~$\Ub^\transpose\Xb\zb=\nullb$ is equivalent to
\begin{align*}%
    \frac{1}{|\{i:z_i=0\}|}\sum_{i: z_i=0} \Ub^\transpose\xb_i = \frac{1}{|\{i:z_i=1\}|}\sum_{i: z_i=1} \Ub^\transpose\xb_i.
\end{align*}
Hence, fair PCA finds a best-approximating projection 
such 
that the projected data's group-conditional means coincide.  
This interpretation implies that for a special data-generating model the relaxed Problem~\eqref{eq:fair_PCA_relaxed} 
solved by fair PCA 
coincides with Problem~\eqref{eq:fair_PCA}, which we originally wanted~to~solve.

\begin{proposition}\label{prop:gaussian_data}
If 
datapoints are sampled from a mixture of two Gaussians with identical covariance \mbox{matrices} and the two Gaussians 
corresponding to 
demographic groups,  
then, in the limit of $n\rightarrow \infty$, 
\eqref{eq:fair_PCA} and~\eqref{eq:fair_PCA_relaxed}~are~equivalent.
\end{proposition}

\begin{proof}
Let $\mathbf{\mu}_0,\mathbf{\mu}_1\in\R^d$ be the means of the two Gaussians and $\mathbf{\Sigma}\in\R^{d\times d}$ their shared covariance matrix such that datapoints are distributed as 
$\xb|z=l \sim \mathcal{N}(\mu_l,\mathbf{\Sigma})$, $l\in\{0,1\}$. 
After projecting datapoints onto $\R^k$ using $\Ub$ 
we have 
$\Ub^\transpose\xb|z=l \sim \mathcal{N}(\Ub^\transpose\mu_l,\Ub^\transpose\mathbf{\Sigma} \Ub)$. 
For $\Ub\in\mathcal{U'}$ as defined in \eqref{eq:fair_PCA_relaxed} the interpretation 
from 
above shows that  $\Ub^\transpose\mu_0=\Ub^\transpose\mu_1$, and hence $h(\Ub^\transpose\xb)$ and $z$ are independent for 
any 
$h$. 
\end{proof}

\section{EXTENSIONS \& VARIANTS}\label{sec:extensions}

We discuss 
several 
extensions and variants of our formulation
 of fair PCA and our proposed algorithm 
 from 
 Section~\ref{sec:methods}.

\subsection{Trading Off Accuracy vs. Fairness}\label{subsec:tradeoff}

Requiring an ML model to be fair 
often 
leads to a loss in predictive performance. 
For example, in the case of 
DP 
as defined in \eqref{eq:fairness_definitions}
it is clear that any fair predictor cannot have perfect accuracy if 
$\Psymb(Y=1|Z=z_1)\neq \Psymb(Y=1|Z=z_2)$. Hence, it is desirable for bias mitigation methods to have a knob that one can turn 
to trade off accuracy vs. fairness. 
We introduce such a knob for fair PCA via 
the following strategy: if $\Ufair\in \R^{d\times k}$ denotes the projection matrix of fair PCA and $\Ust\in \R^{d\times k}$ the one of standard PCA, we concatenate the fair representation $\Ub^\transpose_{\text{fair}}\xb$ of a datapoint~$\xb$ with a rescaled version of the standard representation $\Ub^\transpose_{\text{st}}\xb$, that is we consider $(\Ub^\transpose_{\text{fair}}\xb;\lambda \cdot \Ub^\transpose_{\text{st}}\xb)\in \R^{2k}$ for some $\lambda \in [0,1]$. If $\lambda=0$, this representation contains only the information of 
fair PCA; 
if $\lambda=1$, it contains all the information of 
standard PCA (and hence, potentially, all the demographic information in the data).
For $0<\lambda\ll 1$, technically the new representation also contains all the information of 
standard PCA, 
but any ML model trained 
with weight regularization will have troubles to exploit that information and  will be approximately fair.\footnote{
To 
obtain 
some 
intuition, consider the following simple scenario: let $k=1$, so that  $\Ub^\transpose_{\text{fair}}\xb=:x_{\text{f}}\in\R$  and $\Ub^\transpose_{\text{st}}\xb=:x_{\text{s}}\in\R$, and assume that we train a linear model~$h:\R^2\rightarrow \R$, $h(u,v)=w_1\cdot u+w_2 \cdot v$, parameterized by $w_1$ and $w_2$, on the representation~$(x_{\text{f}};\lambda\cdot x_{\text{s}})$. 
With 
weight regularization,  
$w_1$ and $w_2$ are effectively bounded, and if $\lambda$ is small, 
$h(x_{\text{f}},\lambda\cdot x_{\text{s}})=w_1\cdot x_{\text{f}}+w_2\cdot\lambda \cdot x_{\text{s}}$
must mainly depend 
on the fair PCA representation~$x_{\text{f}}$ rather than the standard PCA representation~$x_{\text{s}}$.
} 
There 
is a 
risk of redundant information in the concatenated representation~$(\Ub^\transpose_{\text{fair}}\xb;\lambda \cdot \Ub^\transpose_{\text{st}}\xb)$, 
which could confuse the learning algorithm applied on top according to some papers on feature selection \citep[e.g.,][]{Koller1996,Yu2004}.
However,  
in our experiments in Section~\ref{subsec:experiments_bias_mitigation} this does not seem to be an issue
and 
we see that our proposed strategy   provides an effective way 
to trade~off~accuracy~vs.~\mbox{fairness}.

\subsection{Adaptation to Equal Opportunity}

Our 
formulation 
of fair PCA in Section~\ref{sec:methods} aimed at 
making the data representation independent of the demographic attribute, thus aiming for 
demographic parity 
fairness of 
arbitrary 
downstream classifiers. If we instead aim for 
equality of opportunity 
fairness, of downstream classifiers trained to solve a 
specific 
task (coming with ground-truth labels~$y_i$), 
we 
apply 
the procedure 
only to datapoints~$\xb_i$ with 
$y_i=1$. 

\subsection{Kernelizing Fair PCA}\label{subsec:kernelized_version} 

Fair PCA solves 
\begin{align}\label{eq:fair_PCA_simple}
\begin{split}
    \argmax_{\Ub\in\R^{d\times k}:\, \Ub^\transpose\Ub=\Idk}\trace(\Ub^\transpose\Xb\Xb^\transpose\Ub)\\
    \text{subject to}\quad \zb^\transpose\Xb^\transpose\Ub=\nullb.
    \end{split}
\end{align}
To kernelize fair PCA, we 
rewrite~\eqref{eq:fair_PCA_simple} fully in terms of the kernel matrix~$\Kb=\Xb^\transpose\Xb\in\R^{n\times n}$ and avoid 
using 
the data matrix~$\Xb$. 
By the representer theorem \citep{bernhard_representer_theorem}, the optimal $\Ub$ can be written as $\Ub=\Xb \Bb$ for some $\Bb\in\R^{n\times k}$. The objective $\trace(\Ub^\transpose\Xb\Xb^\transpose\Ub)$ then becomes $\trace(\Bb^\transpose\Xb^\transpose\Xb\Xb^\transpose\Xb\Bb)$, the constraint $\Ub^\transpose\Ub=\Idk$  becomes $\Bb^\transpose\Xb^\transpose\Xb\Bb=\Idk$, and the constraint $\zb^\transpose\Xb^\transpose\Ub=\nullb$ becomes $\zb^\transpose\Xb^\transpose\Xb\Bb=\nullb$. Hence, with $\Kb=\Xb^\transpose\Xb$, \eqref{eq:fair_PCA_simple}
is equivalent to
\begin{align}\label{fair_PCA_eq_K}
\begin{split}
    \argmax_{\Bb\in\R^{n\times k}:\, \Bb^\transpose\Kb\Bb=\Idk}\trace(\Bb^\transpose\Kb\Kb\Bb)\\
    \text{subject to}\quad \zb^\transpose\Kb\Bb=\nullb.
\end{split}
\end{align}
Let $\Rb\in\R^{n\times (n-1)}$ comprise as columns an orthonormal basis of 
the nullspace of $\zb^\transpose\Kb$. With $\Bb=\Rb\Lamb$ for $\Lamb\in \R^{(n-1)\times k}$, \eqref{fair_PCA_eq_K} is equivalent to
\begin{align}\label{fair_PCA_eq_K2}
    \argmax_{\Lamb:\,
    \Lamb^\transpose\Rb^\transpose\Kb\Rb\Lamb=\Idk}\trace(\Lamb^\transpose\Rb^\transpose\Kb\Kb\Rb\Lamb).
\end{align}
A solution $\Lamb$ is obtained by filling the columns of $\Lamb$ with the generalized eigenvectors, corresponding to the largest $k$ eigenvalues, that solve $\Rb^\transpose\Kb\Kb\Rb\Lamb=\Rb^\transpose\Kb\Rb\Lamb \Wb$, where $\Wb$ is a diagonal matrix containing the eigenvalues \citep{ghojogh2019}. 
When projecting datapoints onto the  linear subspace, we can write $\Ub^\transpose\Xb=\Bb^\transpose\Xb^\transpose\Xb=\Lamb^\transpose\Rb^\transpose\Kb$, and hence we have kernelized fair PCA.
We provide the pseudo code of kernelized fair PCA 
in Appendix~\ref{app:multiple_groups}. Its running time is $\mathcal{O}(n^3)$ when 
being 
given 
$\Kb$
as input, which is the same as the running time of standard kernel PCA.

\subsection{Multiple Groups}\label{subsec:multiple_groups}

We derive fair PCA for multiple demographic groups by 
means of a one-vs.-all approach:
assume that there are $m$ disjoint groups. For every datapoint~$\xb_i$ we consider $m$ many 
one-hot 
demographic attributes~$z_i^{(1)},\ldots,z_i^{(m)}$ with $z_i^{(l)}=1$ if $\xb_i$ belongs to group~$l$ and  $z_i^{(l)}=0$ otherwise. We now require that for all linear functions $h$, $h(\Ub^\transpose\xb_i)$ and $z_i^{(l)}$ are uncorrelated for all $l\in[m]$. This is equivalent to requiring that $\Zb^\transpose\Xb^\transpose\Ub=\nullb$, where $\Zb\in\R^{n\times m}$ and the $l$-th column of $\Zb$ equals $(z_1^{(l)}-\bar{z}^{(l)},\ldots,z_n^{(l)}-\bar{z}^{(l)})^\transpose$ with $\bar{z}^{(l)}=\frac{1}{n} \sum_{i=1}^n z_i^{(l)}$. The resulting optimization problem can be solved analogously to fair PCA for two groups as long as $k\leq d-m+1$, and for $m=2$ the formulation presented here is equivalent to the one of Section~\ref{sec:methods}. 
Also the interpretation provided there 
holds in an analogous way 
for multiple groups: 
fair PCA for multiple groups finds a best-approximating projection 
such 
that the projected data's group-conditional means coincide for all groups. 
We provide details and the pseudo code of fair PCA 
for 
multiple 
groups, 
also 
in 
its 
kernelized version, in Appendix~\ref{app:multiple_groups}.

\subsection{Multiple Demographic Attributes}\label{subsec:multiple_dem_attributes}

We can also adapt fair PCA 
to simultaneously obfuscate demographic information for multiple demographic attributes (e.g., gender \emph{and} race), each of them potentially defining multiple demographic groups: 
assume that there are 
$p$ many attributes, where the $r$-th attribute defines $m_r$ demographic groups. 
For $r\in[p]$, 
let $\Zb_r\in\R^{n\times m_r}$ be the matrix~$\Zb$ from Section~\ref{subsec:multiple_groups} for the $r$-th attribute. By stacking the matrices~$\Zb_r$ to form one matrix $\Zb_{\text{comb}}\in\R^{n\times (\sum_r m_r)}$ and replacing the matrix~$\Zb$ from Section~\ref{subsec:multiple_groups} or Algorithm~\ref{alg:fair_PCA_multi_groups} with $\Zb_{\text{comb}}$, we obtain fair PCA for multiple demographic attributes. The resulting %
algorithm is guaranteed to successfully terminate if 
$k\leq d - \sum_{r=1}^p m_r +p$.

\subsection{Higher-Order Variant: Equalizing Group-Conditional 
Covariance Matrices
}\label{subsec:covariance_extension}

Fair PCA finds a best-approximating projection that equalizes the group-conditional means.  
It is natural to ask whether one can additionally equalize group-conditional covariances in order to further exacerbate discriminability of the projected group-conditional distributions. For 
one demographic attribute with  
two 
demographic 
groups, this additional constraint would result in the following 
problem:
\begin{align}\label{eq:fair_PCA_with_covariance_constraint}
\begin{split}
    \argmax_{\Ub\in\R^{d\times k}:\, \Ub^\transpose\Ub=\Idk}\trace(\Ub^\transpose\Xb\Xb^\transpose\Ub)\\
    \text{s. t.}\quad \zb^\transpose\Xb^\transpose\Ub=\nullb~\wedge~\Ub^\transpose(\mathbf{\Sigma}_0-\mathbf{\Sigma}_1)\Ub=\nullb,
    \end{split}
\end{align}
where $\zb$ is the vector encoding group-membership as in Section~\ref{sec:methods} and $\mathbf{\Sigma}_0$ and $\mathbf{\Sigma}_1$ are the two group-conditional covariance matrices. Unfortunately, depending on $\mathbf{\Sigma}_0$ and $\mathbf{\Sigma}_1$, this problem may not have a solution (e.g., when the feature variances for one group are much bigger than for the other group and hence $\mathbf{\Sigma}_0-\mathbf{\Sigma}_1$ is positive or negative definite). However, 
for small $k$ 
 (or large $d$)
we can apply a simple strategy to solve \eqref{eq:fair_PCA_with_covariance_constraint} approximately.
After writing $\Ub=\Rb\Lamb$ as in Section~\ref{sec:methods}, the problem becomes
\begin{align}\label{eq:fair_PCA_with_covariance_constraint_substituted}
\begin{split}
    \argmax_{\Lamb\in\R^{(d-1)\times k}:\, \Lamb^\transpose\Lamb=\Idk}\trace(\Lamb^\transpose\Rb^\transpose\Xb\Xb^\transpose\Rb\Lamb)\\
    \text{subject to}\quad \Lamb^\transpose\Rb^\transpose(\mathbf{\Sigma}_0-\mathbf{\Sigma}_1)\Rb\Lamb=\nullb.
    \end{split}
\end{align}
For some parameter $l\in  \{k,\ldots,d-1\}$, we can compute the $l$~smallest (in magnitude) eigenvalues of $\Rb^\transpose(\mathbf{\Sigma}_0-\mathbf{\Sigma}_1)\Rb$ and corresponding orthonormal eigenvectors. Let $\Qb\in\R^{d-1 \times l}$ comprise these eigenvectors as columns. By substituting $\Lamb = \Qb\Vb$ for $\Vb\in \R^{l\times k}$ and solving
\begin{align*}
    \argmax_{\Vb\in\R^{l\times k}:\, \Vb^\transpose\Vb=\Idk}\trace(\Vb^\transpose\Qb^\transpose\Rb^\transpose\Xb\Xb^\transpose\Rb\Qb\Vb),
\end{align*}
which just requires to compute eigenvectors of $\Vb^\transpose\Qb^\transpose\Rb^\transpose\Xb\Xb^\transpose\Rb\Qb\Vb$, we optimize the objective of Problem~\eqref{eq:fair_PCA_with_covariance_constraint_substituted} while approximately satisfying its constraint. The running time of this procedure is 
$\mathcal{O}(nd^2+d^3)$ 
as for standard PCA. 
The smaller 
the parameter~$l$, 
the more we equalize the 
projected data's 
group-conditional covariance matrices. For $l=d-1$, our strategy becomes void and coincides with fair PCA as described in Section~\ref{sec:methods}. 
In our experiments in Section~\ref{sec:experiments} we choose $l=\max\{k,\lfloor 0.5d\rfloor\}$ or $l=\max\{k,\lfloor 0.85d\rfloor\}$ and observe good results. 
In particular, we see that the variant yields fairer non-linear downstream classifiers than fair PCA from Section~\ref{sec:methods}.  
An example  can be seen in Figure~\ref{fig:example_fair_PCA_same_covariance}: here, the data comes from a mixture of two Gaussians in $\R^{10}$ with highly different covariance matrices 
and $k=2$. 
Each Gaussian corresponds to 
one 
demographic group. We can see that fair PCA 
from 
Section~\ref{sec:methods} fails to obfuscate the demographic information since the group-conditional covariance matrices of the projected data are highly different (just as for the original data), while the variant 
of 
this section (with $l=5=\lfloor 0.5d \rfloor$) successfully obfuscates the demographic information.

\begin{figure}
    \centering
    \includegraphics[scale=0.25]{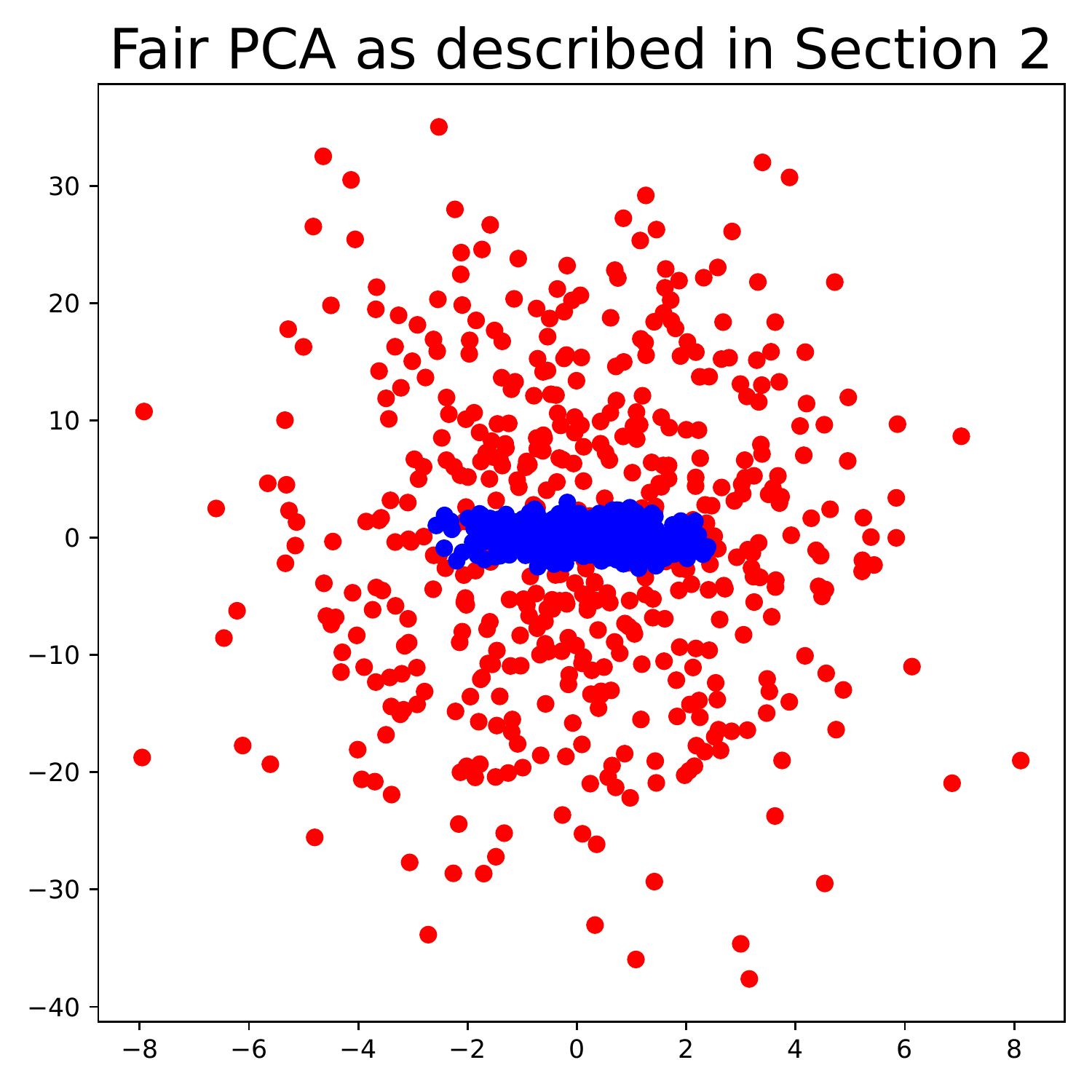}
    \includegraphics[scale=0.25]{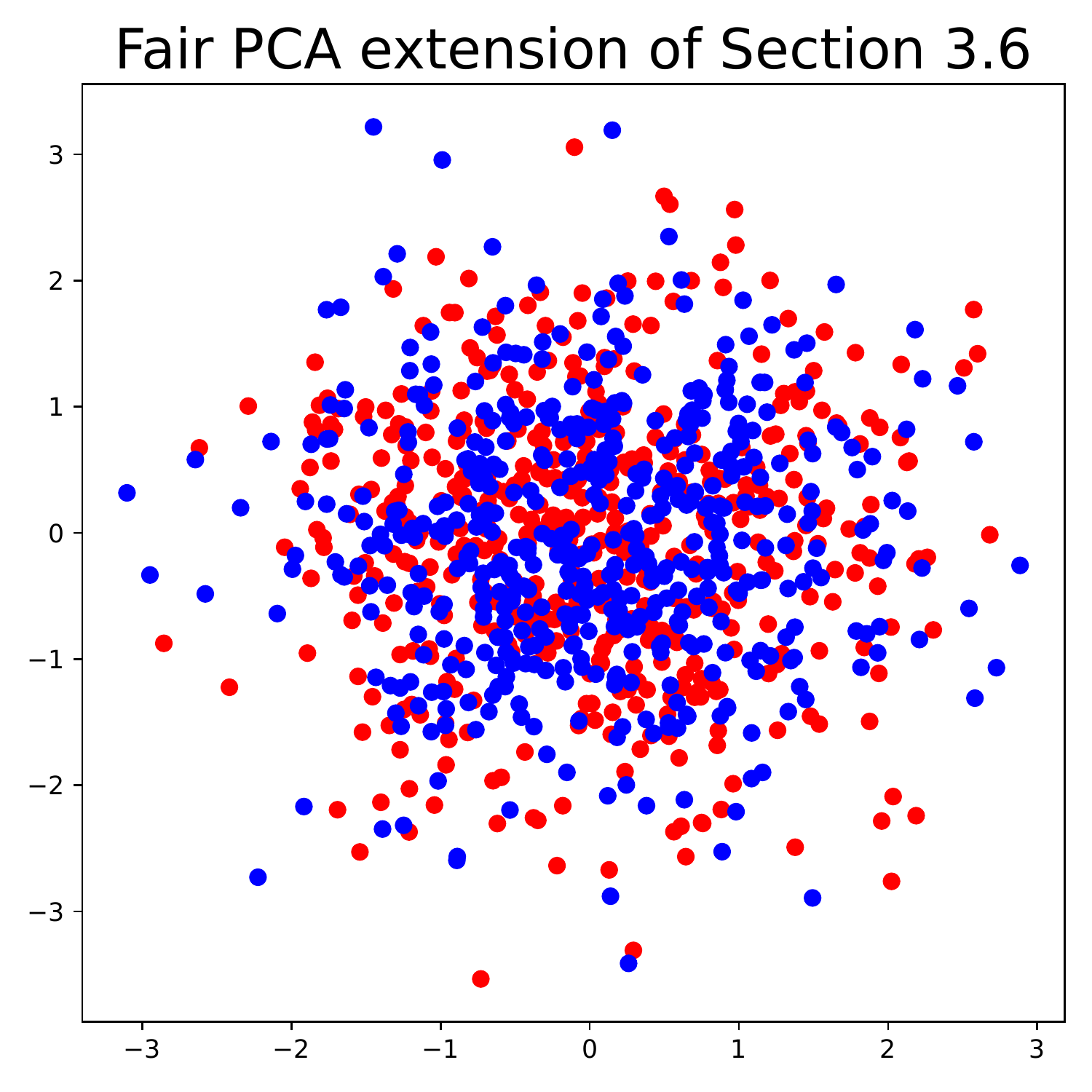}
    \caption{Fair PCA as described in Section~\ref{sec:methods} (left), which equalizes the group-conditional means, in comparison to the higher-order variant of Section~\ref{subsec:covariance_extension} (right), which additionally aims to equalize the group-conditional 
covariance matrices. Only the higher-order variant 
completely 
obfuscates 
the 
demographic information 
(encoded by color: 
red~vs.~blue).}
    \label{fig:example_fair_PCA_same_covariance}
\end{figure}

\section{RELATED WORK}\label{sec:related_work}

\paragraph{Fairness in machine learning (ML)}
Most works 
study the problem of fair classification 
\citep[e.g.,][]{zafar2019},  
but %
fairness has also been studied
for  
unsupervised learning tasks \citep[e.g.,][]{chierichetti2017fair}.  
Two of the most prominent definitions of fairness in classification are demographic parity \citep{kamiran2011}
and equal opportunity \citep{hardt2016equality} as introduced in Section~\ref{sec:introduction}. Methods for fair classification are commonly categorized into pre-processing, in-processing, and post-processing methods, depending on at which stage of the training pipeline they are applied \citep{alessandro2017}. In the following we %
discuss 
the 
works most closely related to our paper, 
all of which can generally be considered as pre-processing~methods.

\paragraph{Fair representation learning}
\citet{zemel2013} initiated the study of fair representation learning, where the goal is to learn an intermediate data representation that obfuscates demographic information while encoding other (non-demographic) information as well as possible. Once such a representation is found, any ML model trained on it 
should not be able to 
discriminate based on demographic information and hence 
be 
demographic parity fair. 
The approach of \citeauthor{zemel2013} 
learns prototypes and a probabilistic mapping of datapoints to these prototypes. Since then, numerous methods for fair representation learning have been proposed \citep[e.g.][]{louizos2016,moyer2018invariant,sarhan2020,balunovic2022,oh2022}, many of them formulating the problem as an adversarial game 
\citep[e.g.][]{edwards2016,Beutel2017DataDA,xie2017controllable,jia2018right,madras2018,raff2018gradientreversal,adel2019,alvi2019a,feng2019,song2019}
and some of them adapting their approach to 
aim for 
downstream classifiers to be equal opportunity fair \citep[e.g.][]{madras2018,song2019}.
In contrast to our proposed approach, none of these techniques allows for an  analytical solution and all of them require numerical optimization, which has often been found hard to perform, in particular for the adversarial approaches (cf. \citealp{feng2019}, Sec.~5, or \citealp{oh2022}, Sec.~2.2). 

\paragraph{Fair PCA 
 for fair representation learning 
 and 
 other 
 methods for linear guarding}
The methods discussed next are all methods for fair representation learning that bear some resemblance to our proposed approach.
Most closely related to our work are the papers by \citet{olfat2019}, \citet{Lee2022}, and \citet{Shao2022}. 

\citet{olfat2019} introduced a notion of fair PCA with the same goal 
that 
we are aiming for in our formulation, that is finding a best-approximating projection such that no linear classifier can predict demographic information from the projected data. They use Pinsker's inequality and an approximation of the group-conditional distributions  by two Gaussians 
to 
obtain 
an upper bound on the best linear classifier's accuracy. The upper bound is minimized when the projected data's group-conditional means and covariance matrices coincide. 
\citeauthor{olfat2019} then formulate a semidefinite program (SDP) to minimize the projection's reconstruction error while satisfying upper bounds on the differences in the projected data's group-conditional means and covariance matrices. This SDP approach has been criticized by \citet[][Section 5.1]{Lee2022} for its high runtime and its relaxation of the rank constraint to a trace constraint, 
``yielding sub-optimal outputs in
presence of (fairness) constraints, even to substantial order in
some cases''. In Section~\ref{sec:experiments} we rerun the experiments of \citeauthor{Lee2022} and also observe that the running time of the method by \citeauthor{olfat2019} is prohibitively high. Furthermore, we consider our derivation of fair PCA to be more intuitive since we do not %
rely on upper bounds or a Gaussian approximation.

Arguing that matching only group-conditional means and covariance matrices of the projected data might be too weak of a constraint,   \citet{Lee2022} define a version of fair PCA by 
requiring 
that the projected data's group-conditional distributions coincide. 
They use the maximum mean discrepancy to measure the deviation of the group-conditional distributions and a penalty method for manifold optimization to solve the resulting optimization problem. While running much faster than the method by \citet{olfat2019}, we 
find the running time of the method by \citeauthor{Lee2022} to be significantly higher than the running time of our proposed algorithms; still, in terms of the quality of the data representation our algorithms can 
compete.
\citeauthor{Lee2022} present their algorithm only for two demographic groups and 
it is unclear 
whether
it can be extended to more than~two~groups.

Concurrently with the writing of our paper, \citet{Shao2022} proposed the spectral attribute removal (SAL) algorithm to remove demographic information 
via 
a 
data projection. 
Their algorithm is based on the observation that a singular value decomposition of the cross-covariance matrix between feature vector~$\xb$ and demographic attribute~$z$ yields projections that maximize the covariance of $\xb$ and $z$. 
Although 
derived 
differently, 
it turns out that the SAL algorithm and our fair PCA method are closely related: SAL projects the data onto the subspace spanned by the columns of the matrix~$\Rb$ in our Algorithm~\ref{alg:fair_PCA}. Hence, for $k=d-1$ the two algorithms project the data onto the same subspace. However, SAL does not allow to choose an embedding dimension smaller than $d-1$.
While \citeauthor{Shao2022} also provide a kernelized variant of their algorithm, they do not provide the interpretation of matching group-conditional means or any extension to also match group-conditional covariances.

There are 
also 
papers that propose methods for linear guarding, that is finding a data representation from which no linear classifier can predict demographic information, that are not related to PCA: \citet{ravfogel2020} iteratively train a linear classifier to predict the demographic attribute and then project the data onto the classifier's nullspace; 
\citet{haghighatkhah2021} describe a procedure to find a projection such that the projected data is not linearly separable w.r.t. 
the 
demographic attribute anymore, but still linearly separable w.r.t. some other binary attributes; 
\citet{ravfogel2022} formulate 
the problem of linear guarding as 
a linear minimax game, where a projection matrix competes against the parameter vector of a 
linear model. In case of linear regression this game can be solved analytically, while for logistic regression and other linear models a relaxation of the game is solved  via alternate minimization and maximization.

\paragraph{Fair PCA for balancing reconstruction error}
A 
very
different notion of fair PCA was introduced by 
\citet{samira2018}, which views PCA as a standalone problem and 
wants 
to balance the excess reconstruction error across 
different 
demographic 
groups. This line of work, which is incomparable to our notion of fair PCA and the notions discussed above, has been 
extended 
by 
\citet{samira2019}, \citet{Pelegrina2021} and \citet{KamaniPCA}.

\paragraph{Information bottleneck method} As pointed out by one of the reviewers, there might be a 
closer relationship between our 
formulation 
of fair PCA and the information bottleneck method \citep{tishby1999}, where the goal is to find a compression of a signal variable $X$ while preserving information about a relevance variable~$Y$. In particular, when $X$ and $Y$ are jointly multivariate Gaussian variables, the optimal projection matrix is obtained by solving an eigenvalue problem involving the cross-covariance matrix~$\mathbf{\Sigma}_{XY}=(\E[(X_i-\E[X_i])(Y_j-\E[Y_j])])_{ij}$ \citep{chechik2005}.

\newcommand{\scaleExpMMDFairPCAa}{0.2}
\begin{figure*}[t]
    \centering
    \includegraphics[scale=\scaleExpMMDFairPCAa]{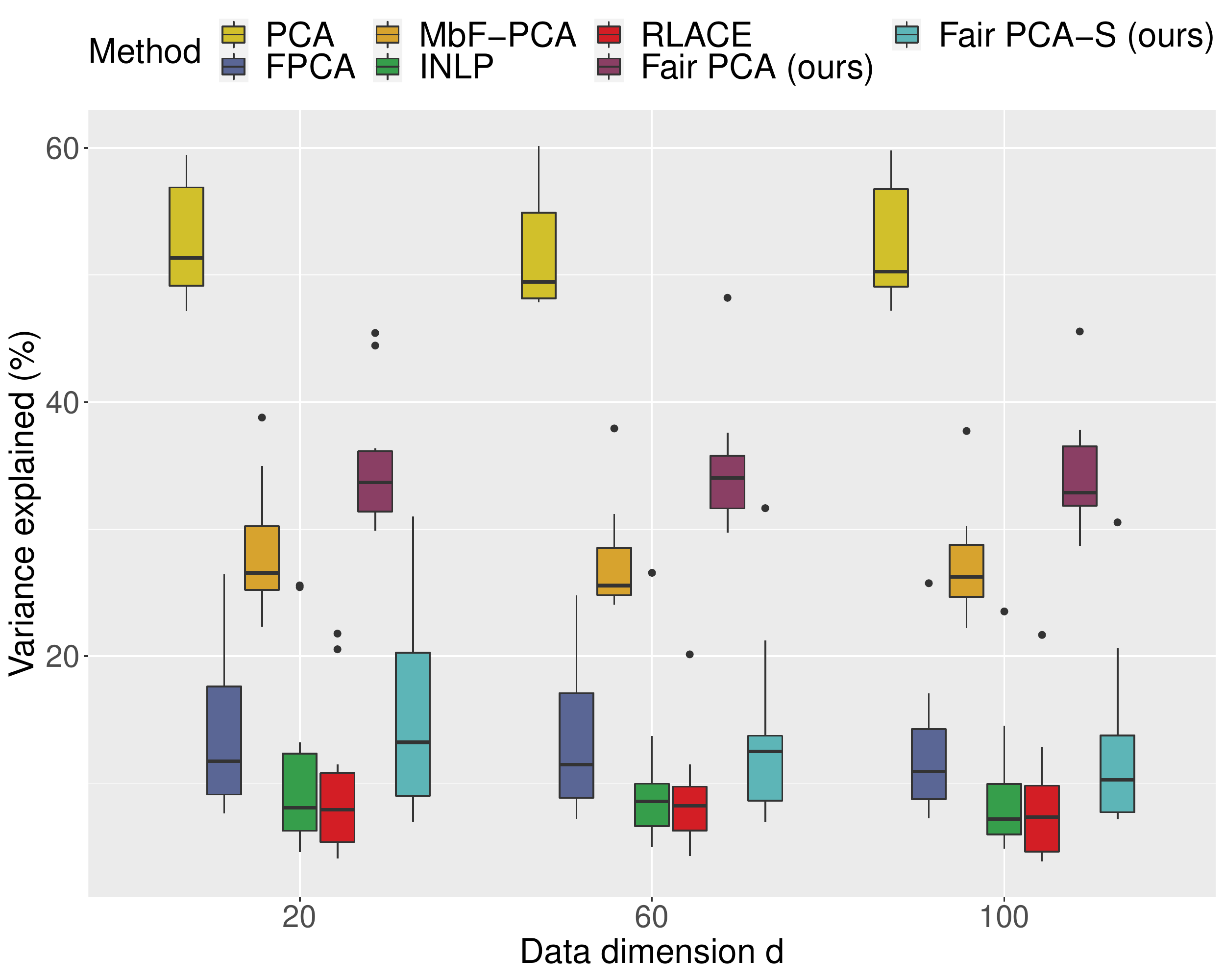}
    \hspace{5mm}
    \includegraphics[scale=\scaleExpMMDFairPCAa]{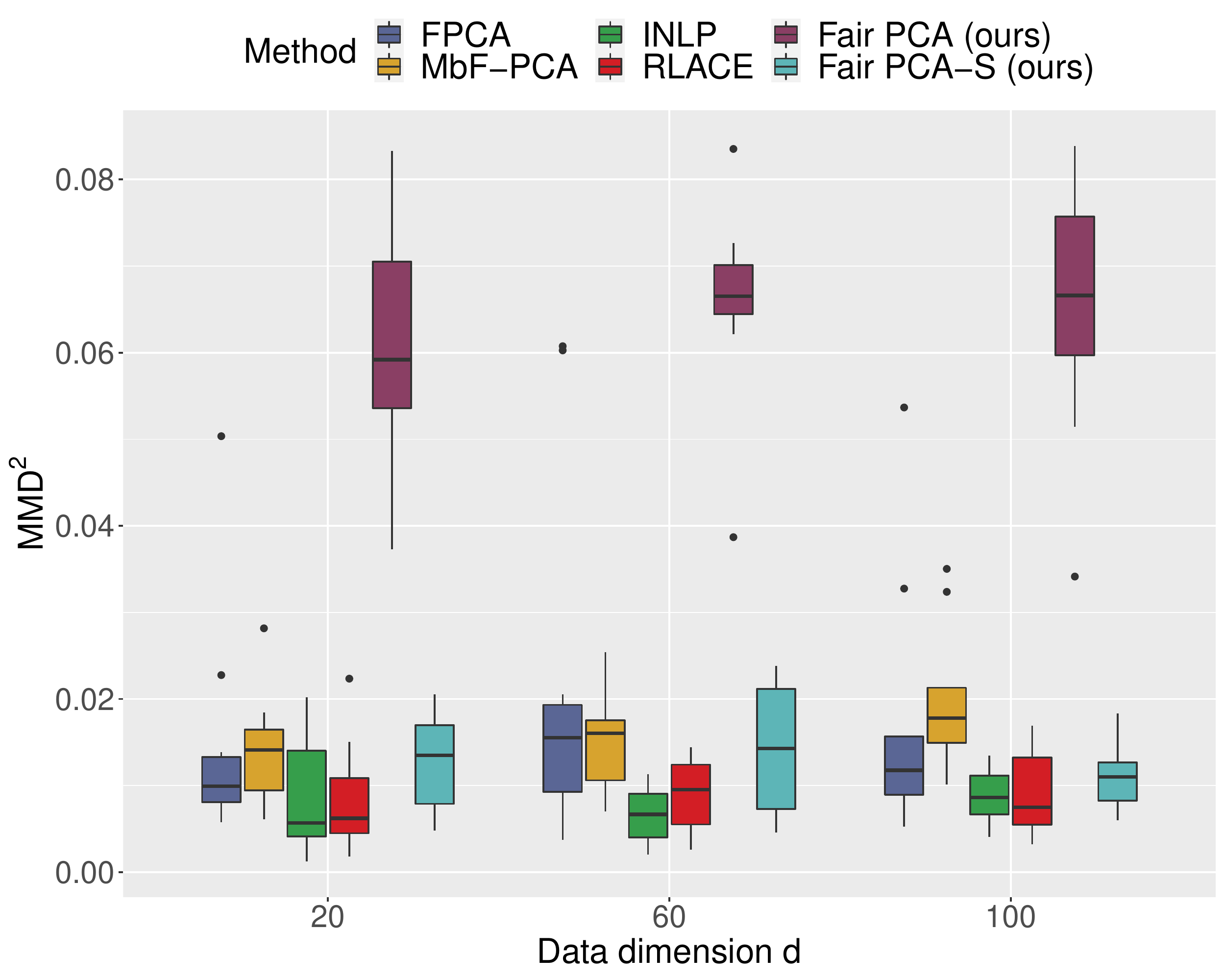}
    \hspace{5mm}
    \includegraphics[scale=\scaleExpMMDFairPCAa]{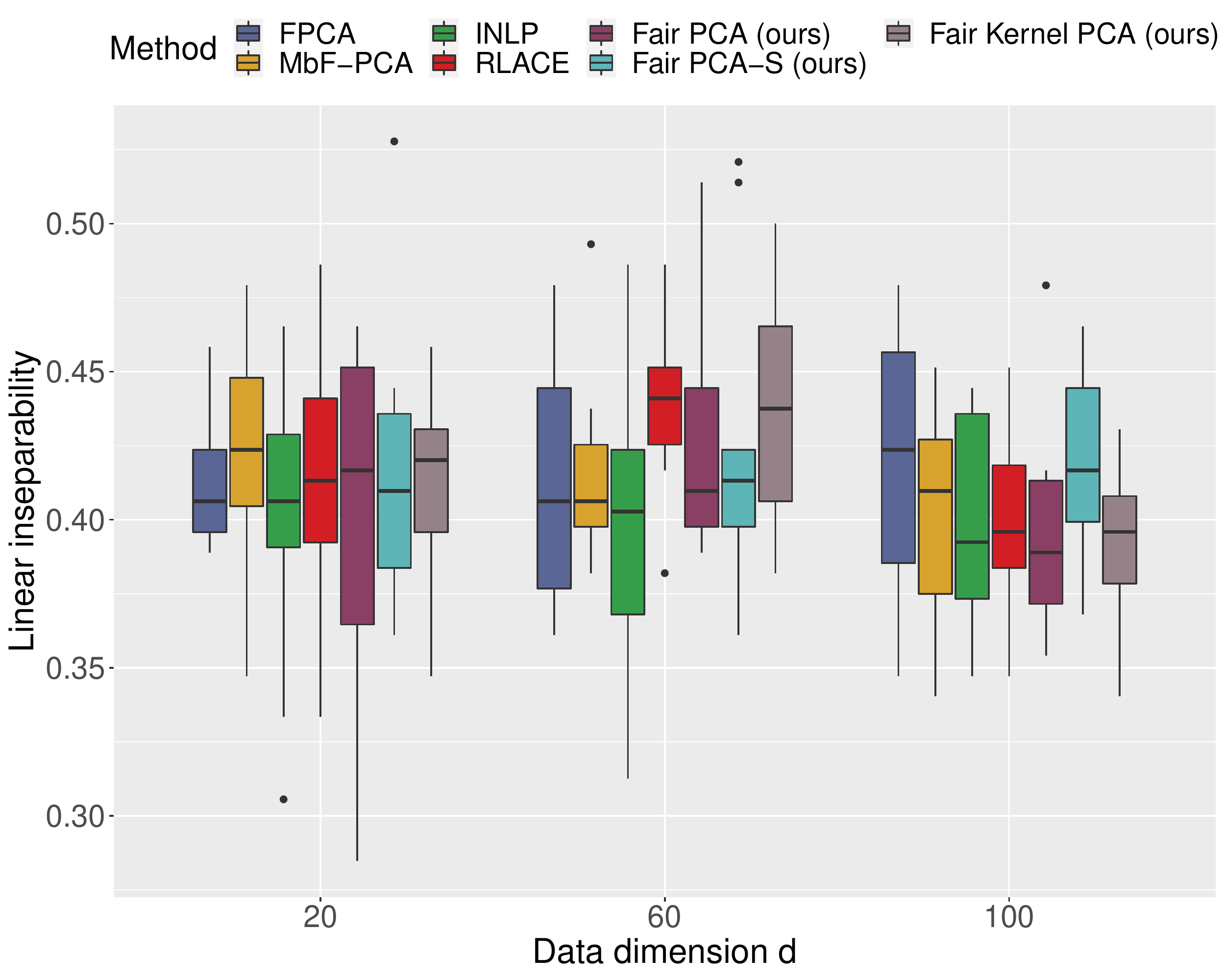}
    \caption{We compare our proposed algorithms to standard PCA and the methods by \citet{olfat2019} (FPCA), \citet{Lee2022} (MbF-PCA),  \citet{ravfogel2020} (INLP), and 
\citet{ravfogel2022} (RLACE), in the experimental setup of \citeauthor{Lee2022}. 
Variance explained (left plot; higher is better) measures how well the representation approximates the 
data; MMD$^2$ (middle plot; lower is better) and linear inseparability (right plot; higher is better) measure the fairness of the representation---see the running text for details. 
The left and middle plots do not show results for fair kernel PCA since 
its 
projection 
space 
lies in a 
reproducing kernel Hilbert space \citep[e.g.,][]{book_learning_with_kernels} 
and the two metrics are 
not comparable between fair kernel PCA and the other methods. 
The middle and right plots do not show results for standard PCA since its values are too high and low, respectively 
($\text{MMD}^2 > 0.5$ and $\text{linear inseparability}< 0.02$ for~all~data~dimensions).
}
    \label{fig:MMD_fair_pca_exp_1}
\end{figure*}

\section{EXPERIMENTS}\label{sec:experiments}

In this section, we present a number of experiments.\footnote{Code available on \url{https://github.com/amazon-science/fair-pca}.}
We first 
rerun and extend the experiments performed by \citet{Lee2022} in order to compare our algorithms to the existing methods for fair PCA by \citet{olfat2019} and \citet{Lee2022} and to the methods for linear guarding by \citet{ravfogel2020} and   
\citet{ravfogel2022}. We also apply our version of fair PCA to the CelebA dataset of facial images to illustrate 
its 
applicability 
to large high-dimensional~datasets. 
We then demonstrate the usefulness of our proposed algorithms 
as 
means of 
bias mitigation and compare their performance to 
the reductions approach of \citet{agarwal_reductions_approach}, which is the  state-of-the-art in-processing %
method implemented in Fairlearn (\url{https://fairlearn.org/}).
Some implementation details and details about 
datasets 
are \mbox{provided~in~Appendix~\ref{app:implementation_details}~and~\ref{app:details_about_datasets}.}

\setcounter{footnote}{1}
\stepcounter{footnote}

\newcommand{\scaleExpMMDFairPCAb}{0.19}
\begin{figure}[t]
    \centering
    \includegraphics[scale=\scaleExpMMDFairPCAb]{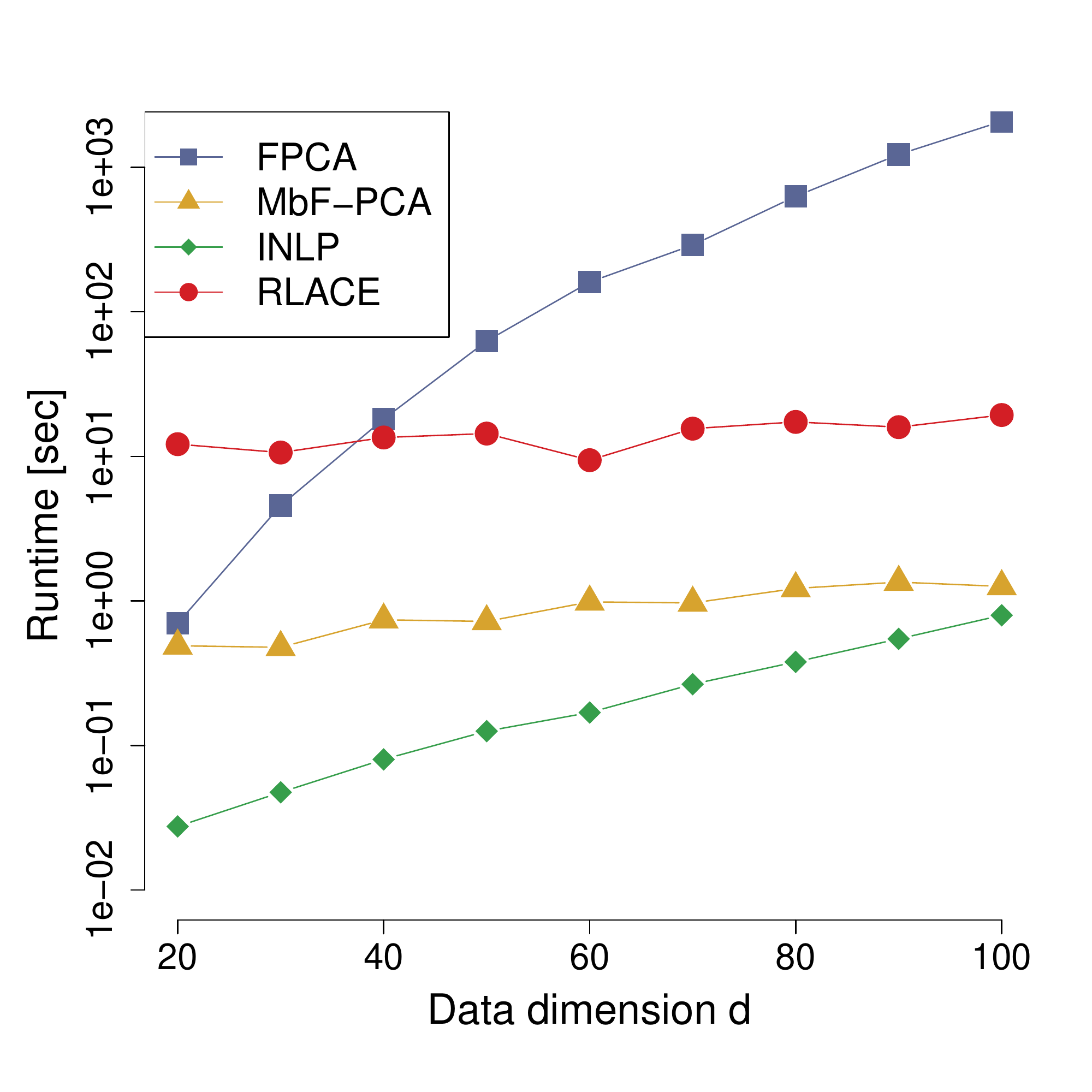}
    \hspace{2mm}
    \includegraphics[scale=\scaleExpMMDFairPCAb]{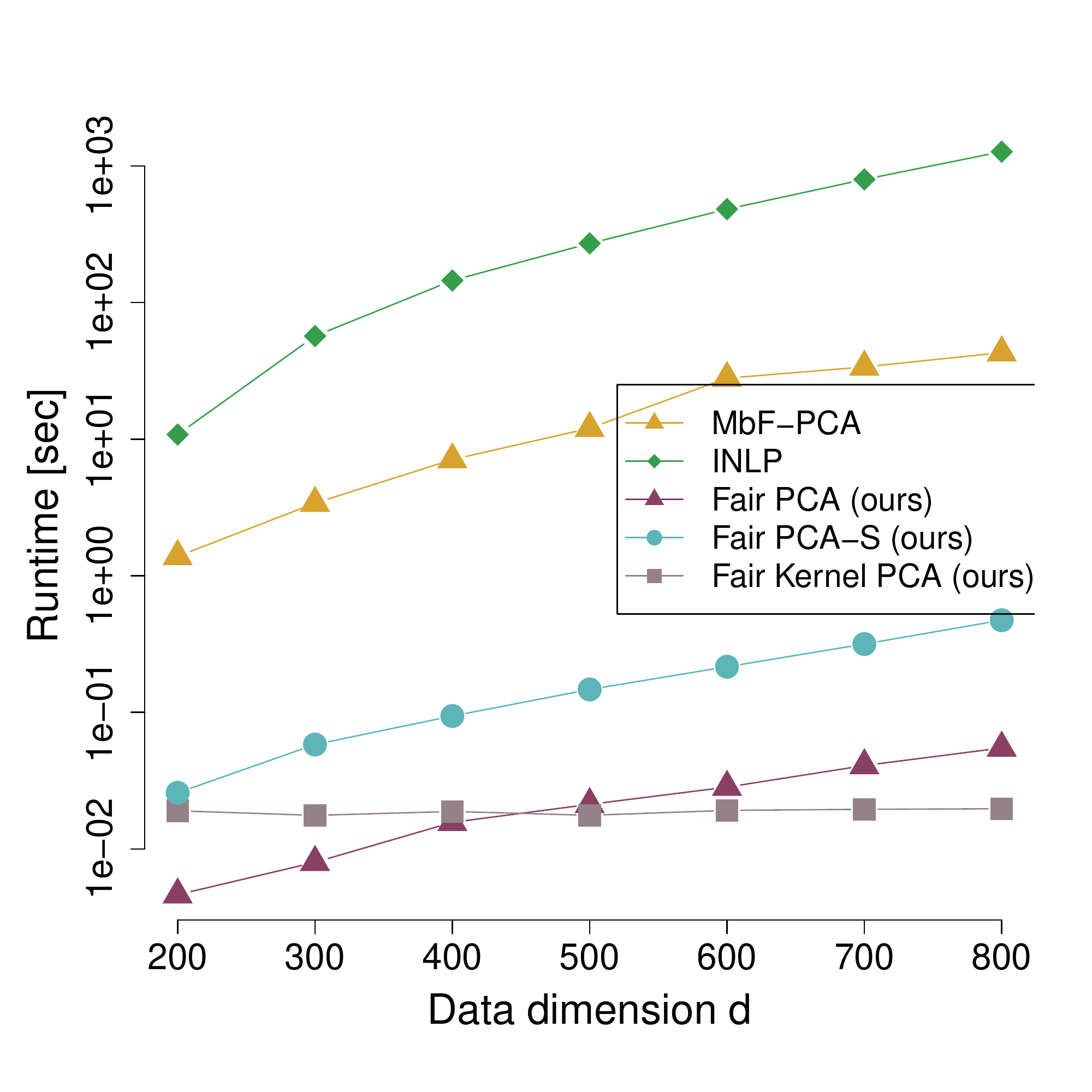}
\caption{The running time of the various methods as a function of the data dimension~$d$.
The 
target 
dimension~$k$ is 5 independent of $d$. Note the logarithmic y-axes and that 
the 
x-axes 
are different for the two plots.\protect\footnotemark}
    \label{fig:MMD_fair_pca_exp_2}
\end{figure}

\begin{table*}[t]
    \centering
\caption{Comparison of our proposed algorithms with standard PCA and the fair methods FPCA, MbF-PCA, INLP, and RLACE on the Adult Income dataset in the setup of \citet{Lee2022}. The top half shows the results for the target dimension~$k=2$, the lower half for $k=10$. Within each block of methods (standard PCA / fair competitors / our fair methods) results with the best mean values are shown in bold. For fair kernel PCA,
\%Var and MMD$^2$ are not~meaningful.}
    \label{tab:Adult}

    \vspace{3mm}
\renewcommand{\arraystretch}{1.2}
\begin{scriptsize}
\begin{tabular}{c|c|cccccccc}
\hline
\multicolumn{10}{c}{\normalsize{\textbf{Adult Income} [$\text{feature dim}=97$, $\Psymb(Y=1)=0.2489$]}}\\
\multirow{2}{*}{$k$} & \multirow{2}{*}{Algorithm} & \multirow{2}{*}{\%Var{\scriptsize ($\uparrow$)}} & \multirow{2}{*}{MMD$^2${\scriptsize ($\downarrow$)}} & \%Acc{\scriptsize ($\uparrow$)} & $\Delta_{DP}${\scriptsize ($\downarrow$)}  & \%Acc{\scriptsize ($\uparrow$)}  & $\Delta_{DP}${\scriptsize ($\downarrow$)}  & \%Acc{\scriptsize ($\uparrow$)}  & $\Delta_{DP}${\scriptsize ($\downarrow$)} \\
 & & & & \multicolumn{2}{c}{Kernel SVM} & \multicolumn{2}{c}{Linear SVM} & \multicolumn{2}{c}{MLP}\\\hline
\multirow{11}{*}{2} & PCA & $\mathbf{7.78_{0.77}}$ & $\mathbf{0.349_{0.026}}$ & $\mathbf{82.03_{1.09}}$ & $\mathbf{0.2_{0.05}}$ & $\mathbf{79.87_{1.11}}$ & $\mathbf{0.2_{0.04}}$ & $\mathbf{81.21_{1.22}}$ & $\mathbf{0.22_{0.03}}$ \\
\cdashline{2-10}
 & FPCA (0.1, 0.01) & $4.05_{0.93}$ & $0.016_{0.011}$ & $77.44_{2.81}$ & $0.04_{0.03}$ & $75.54_{1.98}$ & $\mathbf{0.0_{0.0}}$ & $76.19_{2.44}$ & $0.02_{0.03}$ \\
 & FPCA (0, 0.01) & $3.65_{0.92}$ & $0.005_{0.004}$ & $77.05_{3.02}$ & $\mathbf{0.01_{0.01}}$ & $75.51_{1.93}$ & $0.01_{0.02}$ & $76.24_{2.81}$ & $\mathbf{0.01_{0.01}}$ \\
 & MbF-PCA ($10^{-3}$) & $\mathbf{6.08_{0.59}}$ & $0.005_{0.004}$ & $\mathbf{79.46_{1.21}}$ & $0.02_{0.01}$ & $\mathbf{76.97_{1.71}}$ & $0.02_{0.01}$ & $\mathbf{78.6_{1.38}}$ & $0.02_{0.02}$ \\
 & MbF-PCA ($10^{-6}$) & $5.83_{0.54}$ & $0.005_{0.004}$ & $79.12_{1.08}$ & $\mathbf{0.01_{0.01}}$ & $76.7_{1.86}$ & $0.02_{0.02}$ & $77.69_{1.46}$ & $0.02_{0.01}$ \\
 & INLP & $2.09_{0.18}$ & $\mathbf{0.003_{0.001}} $& $75.94_{1.4}$ & $\mathbf{0.01_{0.01}} $& $75.11_{1.66}$ & $\mathbf{0.0_{0.0}} $& $75.3_{1.68}$ & $\mathbf{0.01_{0.01}}$ \\
 & RLACE & $1.98_{0.19}$ & $0.007_{0.008} $& $76.24_{1.37}$ & $0.02_{0.03} $& $75.11_{1.66}$ & $\mathbf{0.0_{0.0}} $& $75.8_{1.31}$ & $0.02_{0.02}$ \\
\cdashline{2-10}
 & Fair PCA & $\mathbf{6.37_{0.65}}$ & $0.009_{0.003}$ & $\mathbf{80.24_{1.57}}$ & $0.06_{0.02}$ & $\mathbf{77.26_{1.79}}$ & $0.02_{0.02}$ & $\mathbf{78.63_{1.36}}$ & $0.04_{0.02}$ \\
 & Fair Kernel PCA & $n/a_{}$ & $n/a_{}$ & $75.11_{1.66}$ & $\mathbf{0.0_{0.0}}$ & $75.11_{1.66}$ & $\mathbf{0.0_{0.0}}$ & $77.08_{1.78}$ & $0.03_{0.03}$ \\
 & Fair PCA-S (0.5) & $3.05_{0.3}$ & $\mathbf{0.002_{0.002}}$ & $75.85_{1.59}$ & $0.01_{0.01}$ & $75.11_{1.66}$ & $\mathbf{0.0_{0.0}}$ & $75.26_{1.63}$ & $0.01_{0.01}$ \\
 & Fair PCA-S (0.85) & $4.27_{0.34}$ & $0.003_{0.002}$ & $76.07_{1.34}$ & $0.01_{0.01}$ & $75.11_{1.66}$ & $\mathbf{0.0_{0.0}}$ & $75.01_{1.76}$ & $\mathbf{0.0_{0.01}}$  \\
\hline
\multirow{11}{*}{10} & PCA & $\mathbf{21.77_{1.95}}$ & $\mathbf{0.195_{0.006}}$ & $\mathbf{93.64_{0.87}}$ & $\mathbf{0.16_{0.01}}$ & $\mathbf{82.68_{0.96}}$ & $\mathbf{0.18_{0.02}}$ & $\mathbf{89.06_{2.07}}$ & $\mathbf{0.2_{0.03}}$ \\
\cdashline{2-10}
 & FPCA (0.1, 0.01) & $15.75_{1.14}$ & $0.006_{0.003}$ & $91.94_{0.84}$ & $0.13_{0.02}$ & $78.1_{2.15}$ & $0.03_{0.02}$ & $\mathbf{87.17_{1.1}}$ & $0.11_{0.04}$ \\
 & FPCA (0, 0.01) & $15.52_{1.12}$ & $0.004_{0.002}$ & $91.66_{0.92}$ & $0.13_{0.02}$ & $77.72_{2.06}$ & $0.03_{0.02}$ & $85.38_{2.08}$ & $0.09_{0.03}$ \\
 & MbF-PCA ($10^{-3}$) & $\mathbf{18.86_{1.46}}$ & $0.005_{0.002}$ & $\mathbf{93.06_{0.85}}$ & $0.15_{0.01}$ & $\mathbf{80.53_{1.31}}$ & $0.03_{0.02}$ & $86.83_{2.05}$ & $0.08_{0.03}$ \\
 & MbF-PCA ($10^{-6}$) & $12.36_{4.15}$ & $\mathbf{0.002_{0.001}}$ & $83.58_{3.58}$ & $\mathbf{0.05_{0.02}}$ & $75.11_{1.66}$ & $\mathbf{0.0_{0.0}}$ & $80.27_{3.55}$ & $\mathbf{0.04_{0.04}}$ \\
 & INLP & $10.79_{0.84}$ & $0.004_{0.001} $& $89.01_{1.19}$ & $0.1_{0.03} $& $75.11_{1.66}$ & $\mathbf{0.0_{0.0}} $& $85.24_{1.2}$ & $0.11_{0.03}$ \\
 & RLACE & $10.3_{0.49}$ & $0.007_{0.005} $& $90.96_{1.04}$ & $0.12_{0.04} $& $75.11_{1.66}$ & $\mathbf{0.0_{0.0}} $& $86.61_{1.69}$ & $0.11_{0.05}$\\
\cdashline{2-10}
 & Fair PCA & $\mathbf{19.62_{1.73}}$ & $0.014_{0.003}$ & $\mathbf{93.42_{0.8}}$ & $0.16_{0.01}$ & $\mathbf{81.31_{1.23}}$ & $0.04_{0.02}$ & $\mathbf{88.16_{1.45}}$ & $0.16_{0.03}$ \\
 & Fair Kernel PCA & $n/a_{}$ & $n/a_{}$ & $79.9_{1.54}$ & $\mathbf{0.04_{0.03}}$ & $78.34_{1.21}$ & $0.02_{0.02}$ & $80.52_{2.05}$ & $\mathbf{0.06_{0.03}}$ \\
 & Fair PCA-S (0.5) & $12.75_{1.31}$ & $\mathbf{0.004_{0.001}}$ & $86.85_{1.79}$ & $0.09_{0.02}$ & $75.11_{1.66}$ & $\mathbf{0.0_{0.0}}$ & $82.89_{2.01}$ & $0.07_{0.04}$ \\
 & Fair PCA-S (0.85) & $15.79_{1.05}$ & $0.005_{0.001}$ & $91.81_{1.05}$ & $0.15_{0.02}$ & $75.11_{1.66}$ & $\mathbf{0.0_{0.0}}$ & $87.07_{1.4}$ & $0.15_{0.02}$ \\
\hline
\end{tabular}
\end{scriptsize}

\end{table*}

\subsection{Comparison with 
Existing 
Methods for 
Fair PCA 
}\label{subsec:experiments_linear_guarding}

\paragraph{Experiments as in \citet{Lee2022}} 
We used the code provided by \citet{Lee2022} to rerun their experiments and perform a comparison with our proposed algorithms. We 
additionally 
compared to 
the methods by \citet{ravfogel2020} and \citet{ravfogel2022} 
using the code provided by 
those authors, where we set all parameters to their default values except for the maximum number of outer iterations for the second method, which we decreased from 75000 to 10000 to get a somewhat acceptable running time. 
We extended the experimental evaluation of \citeauthor{Lee2022} by reporting additional metrics, but other than that  did not modify their code or experimental setting  in any way.

In their first experiment (Section 8.2 in their paper), 
\citet{Lee2022} applied standard PCA, their method (referred to as MbF-PCA) 
and the method by \citet{olfat2019} 
(FPCA) 
to synthetic data 
sampled 
from a mixture of two Gaussians of varying dimension~$d$. The two Gaussians correspond to two demographic groups. The target dimension~$k$ is held constant at~$5$. We reran 
the code of \citeauthor{Lee2022} 
and additionally 
applied the methods of 
\citet{ravfogel2020} (INLP) and 
\citet{ravfogel2022} (RLACE) and
our algorithms for fair PCA, fair kernel PCA 
with a Gaussian kernel, 
and 
the variant of fair PCA that 
additionally 
aims to   
equalize group-conditional covariance matrices (referred to as \mbox{Fair PCA-S}; we set $l=\lfloor 0.85d\rfloor$---cf. Section~\ref{subsec:covariance_extension}).
\citeauthor{Lee2022} reported the fraction of explained variance of the projected data (i.e., $\trace(\Ub^\transpose\Xb\Xb^\transpose\Ub)/\trace(\Xb\Xb^\transpose)$ for the projection defined by $\Ub$---higher means better approximation of the original data), the 
squared 
maximum mean discrepancy (MMD$^2$) based on a Gaussian kernel between the two groups after the projection (lower means the representation is more fair) and the running time of the methods. We additionally report the error of a linear classifier trained to predict the demographic information from the projected data (higher means the representation is more fair; we refer to this metric as linear inseparability). Figure~\ref{fig:MMD_fair_pca_exp_1} and Figure~\ref{fig:MMD_fair_pca_exp_2} show the results,  
where the boxplots are obtained from considering ten random splits into training and test data and the runtime curves show an average over the ten splits. 
While standard PCA does best in approximating the 
original 
data (variance about 50\%), it does not yield a fair representation with high values for MMD$^2$ (more than 0.5) and low values for linear inseparability 
($\approx 0$).
Our algorithm for fair PCA does worse in approximating the data (variance above 30\%), but drastically reduces the unfairness of standard PCA (MMD$^2$ smaller than 0.07).  The other methods yield even lower values for MMD$^2$, but this comes at the cost of a worse approximation of the data. Our variant Fair PCA-S performs similarly to FPCA by \citet{olfat2019}. All methods except standard PCA perform similarly in terms of linear inseparability. The biggest difference is in 
the methods' 
running times: while FPCA runs for more than 2000 seconds,  RLACE for about 20 seconds, MbF-PCA for about 1.3 seconds, and INLP for about 0.8 seconds when the data dimension~$d$ is as small as 100, none of our algorithms runs for more than 0.5 seconds even when $d=800$. In the latter case,  RLACE runs for about 260 seconds, MbF-PCA for about 43 seconds, and INLP for about 1270 seconds. 
\footnotetext{We ran these experiments on a MacBook Pro with 2.6 GHz 6-Core Intel Core i7 processor 
and 16 GB 2667 MHz DDR4 memory. MbF-PCA is implemented in Matlab while all other methods are implemented %
in Python---by running time we mean wall time.} 
In Appendix~\ref{app:runtime_comparison}, we study the running time of the methods as a function of the target dimension~$k$ and observe that the running time of MbF-PCA drastically increases with~$k$ 
(about 290 seconds when $d=100$ and $k=50$). This shows that none of the existing methods can be applied when both $d$ and $k$ are large (such as in the experiment on the CelebA dataset below) and provides strong evidence for the benefit of~our~proposed~methods.

In their second experiment (Section 8.3 in their paper), 
\citet{Lee2022} applied standard PCA, MbF-PCA and FPCA to three real-world datasets: Adult Income and German Credit from the UCI repository \citep{Dua:2019}, and COMPAS \citep{angwin2016}. 
\citeauthor{Lee2022} ran MbF-PCA and FPCA for two different parameter configurations, indicated by the numbers in parentheses after a method's name in the tables below. Similarly, we ran our proposed method Fair PCA-S from Section~\ref{subsec:covariance_extension} for 
$l=\max\{k,\lfloor 0.5d\rfloor\}$  as well as $l=\max\{k,\lfloor 0.85d\rfloor\}$.
\citeauthor{Lee2022} reported the explained variance and MMD$^2$ as above. Furthermore, they reported the accuracy and the 
DP 
violation~$\Delta_{DP}:=|\Psymb(\hat{Y}=1|Z=0)-\Psymb(\hat{Y}=1|Z=1)|$ of a Gaussian kernel 
support vector machine (SVM)  %
trained to solve a downstream task on the projected data (e.g., for the Adult Income dataset the downstream task is to predict whether a person's income exceeds 
\$50k
or not). We additionally report 
the 
accuracy and $\Delta_{DP}$ of a linear SVM 
and a 
multilayer perceptron (MLP) with two hidden layers of size 10 and 5, respectively. Table~\ref{tab:Adult} provides the results for Adult Income; the tables for 
German Credit and COMPAS 
can be found in Appendix~\ref{app:tables}. 
The reported results are average results (together with standard deviations in subscript) over ten random splits into training and test data.  
We see that there is no single best method. Methods that allow for a high downstream accuracy tend to suffer from higher DP violation and the other way around. The parameters of FPCA and MbF-PCA allow to trade-off accuracy vs. fairness and so does the parameter~$l$ in Fair PCA-S (note that Fair PCA is equivalent to Fair PCA-S(1.0)). Except on COMPAS, whose data dimension is very small, Fair PCA-S always achieves smaller DP violation than fair PCA for the non-linear classifiers and fair kernel PCA achieves the smallest DP violation, among all methods, for the kernel SVM. 
Overall, we consider the results for our proposed methods to be similar as for the existing methods.

One of the reviewers asked for a comparison with the method of \citet{samira2018}, which aims to balance the excess reconstruction error of PCA across different demographic groups (cf. Section~\ref{sec:related_work}). We emphasize once more that this fairness notion is incomparable to ours \citep[also see the discussion in Appendix~A of][]{Lee2022}. Still, 
we provide the results for the method of \citet{samira2018} on the three real-world datasets in Appendix~\ref{app:comparison_samira}. As expected, their method yields much higher DP violations than our methods or~the~other~competitors.

\newcommand{\scaleBiasMitigation}{0.21}

\begin{figure*}[t!]
    \centering
    \includegraphics[scale=\scaleBiasMitigation]{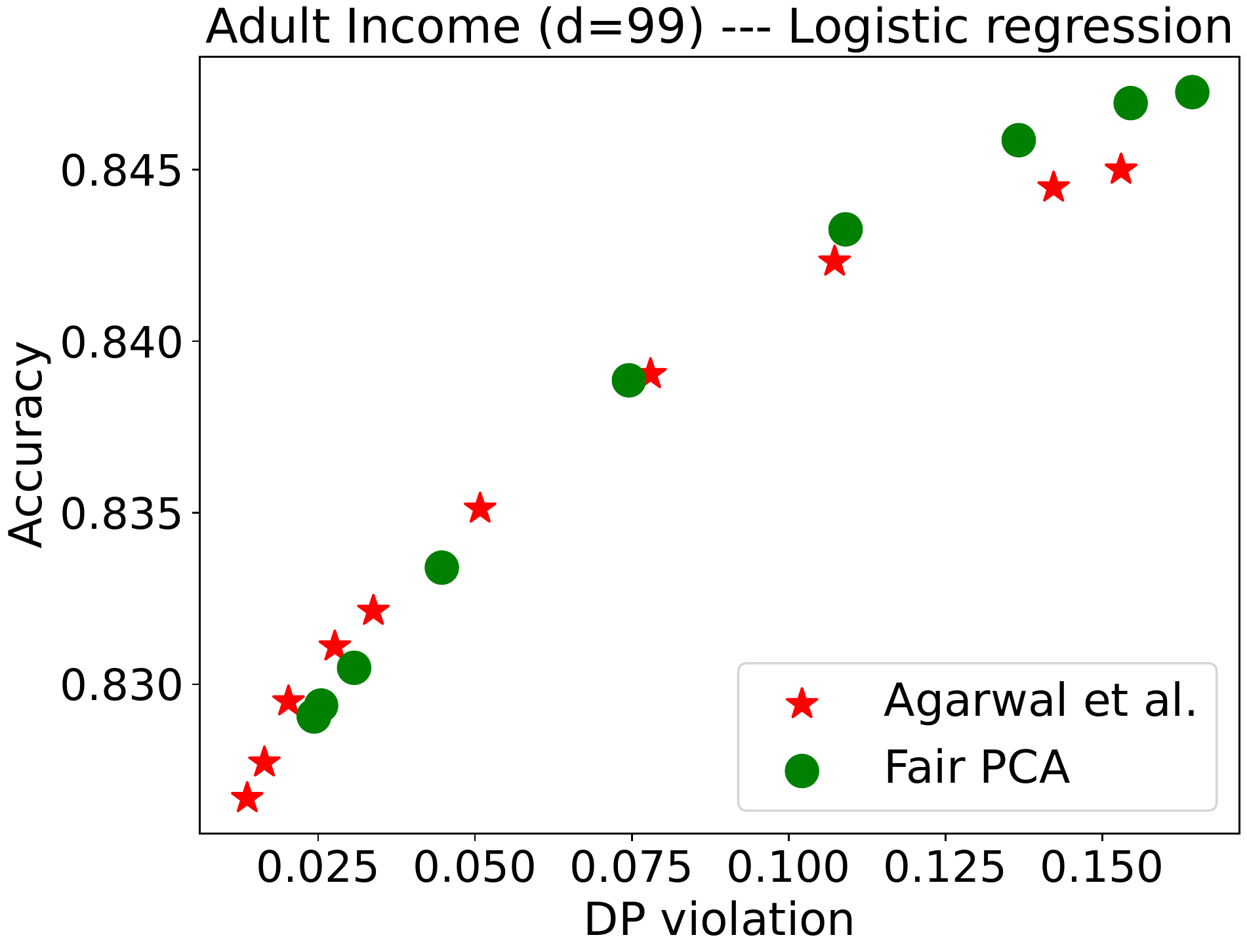}
    \includegraphics[scale=\scaleBiasMitigation]{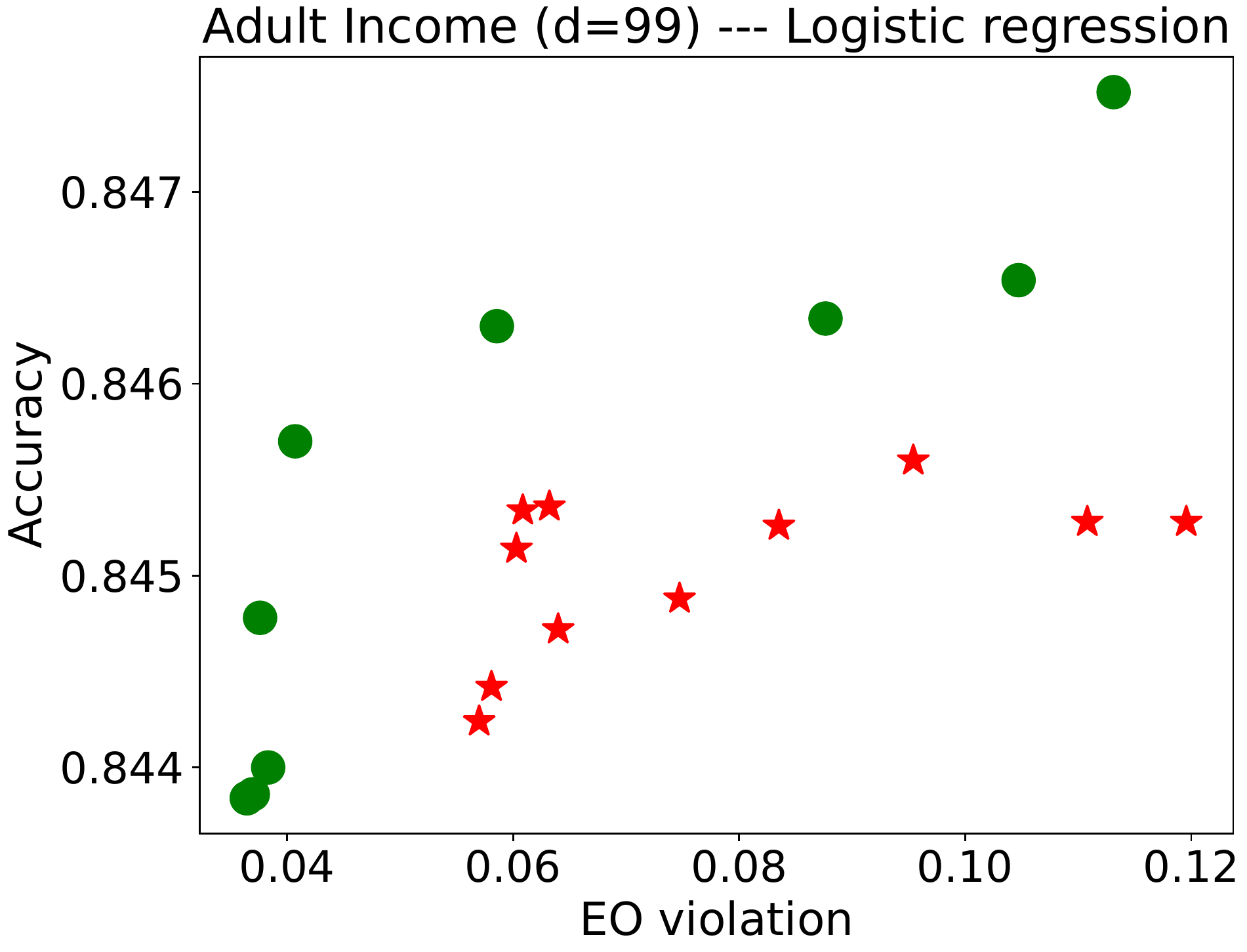}
    \includegraphics[scale=\scaleBiasMitigation]{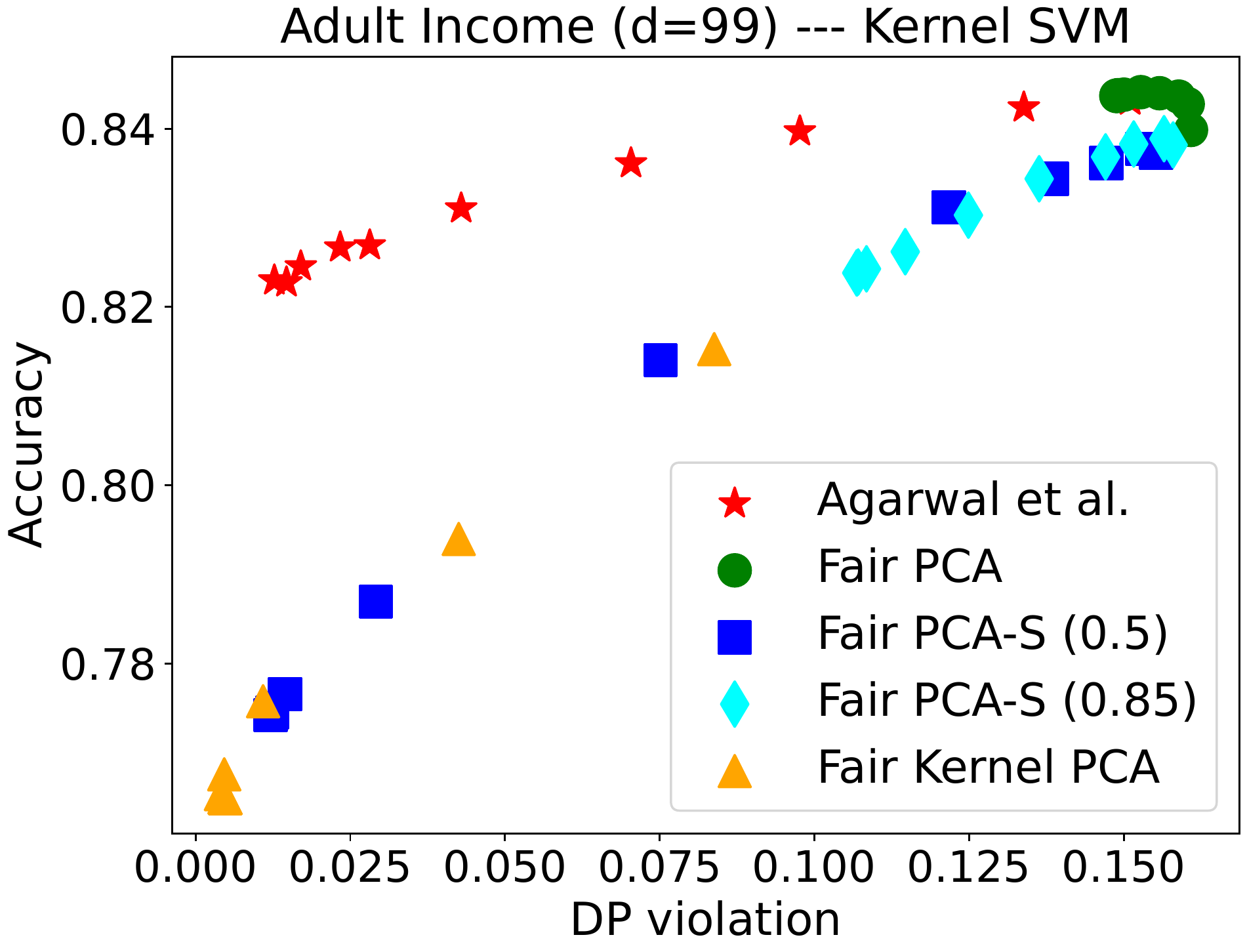}
    \includegraphics[scale=\scaleBiasMitigation]{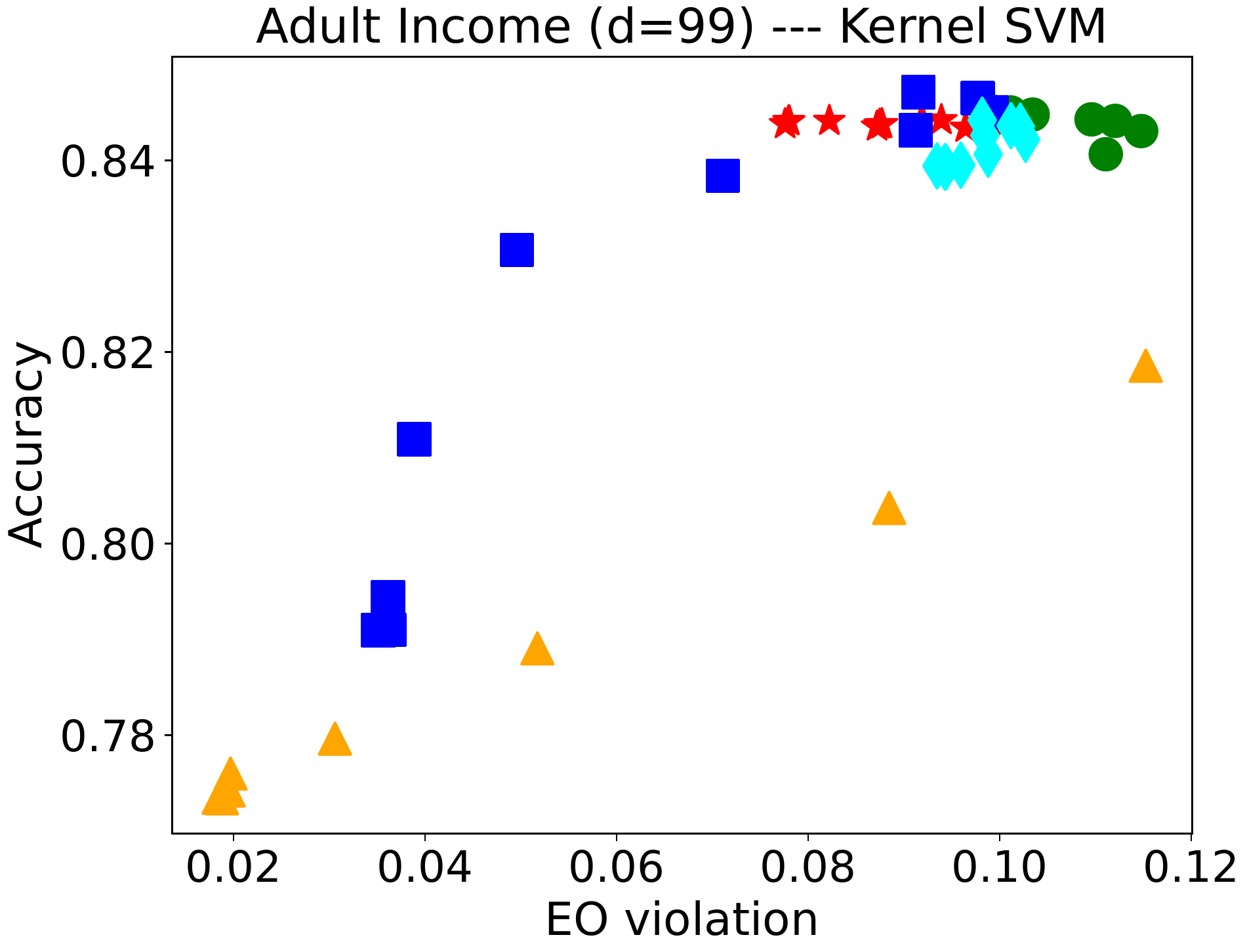}

    \vspace{3mm}
    \includegraphics[scale=\scaleBiasMitigation]{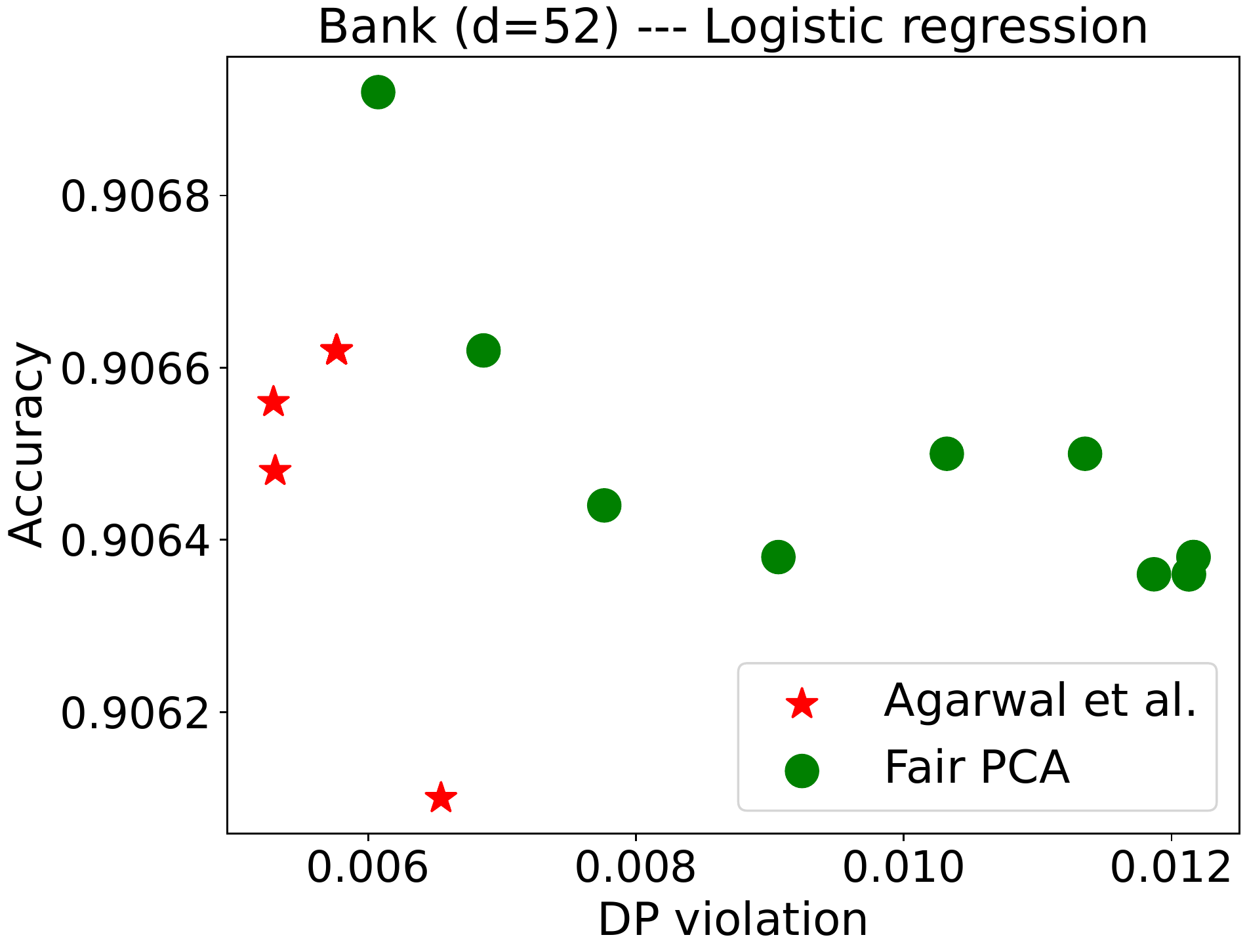}
    \includegraphics[scale=\scaleBiasMitigation]{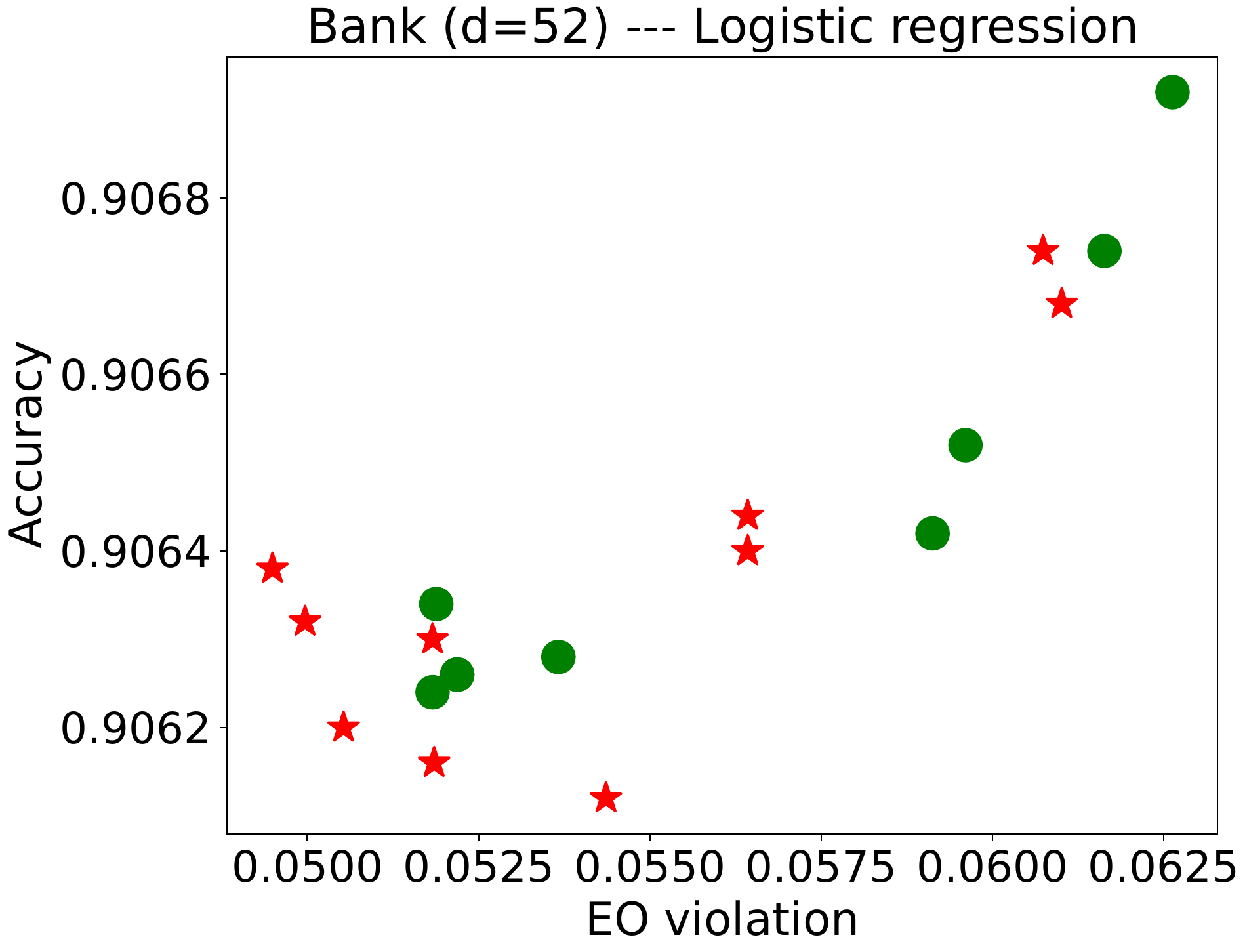}
    \includegraphics[scale=\scaleBiasMitigation]{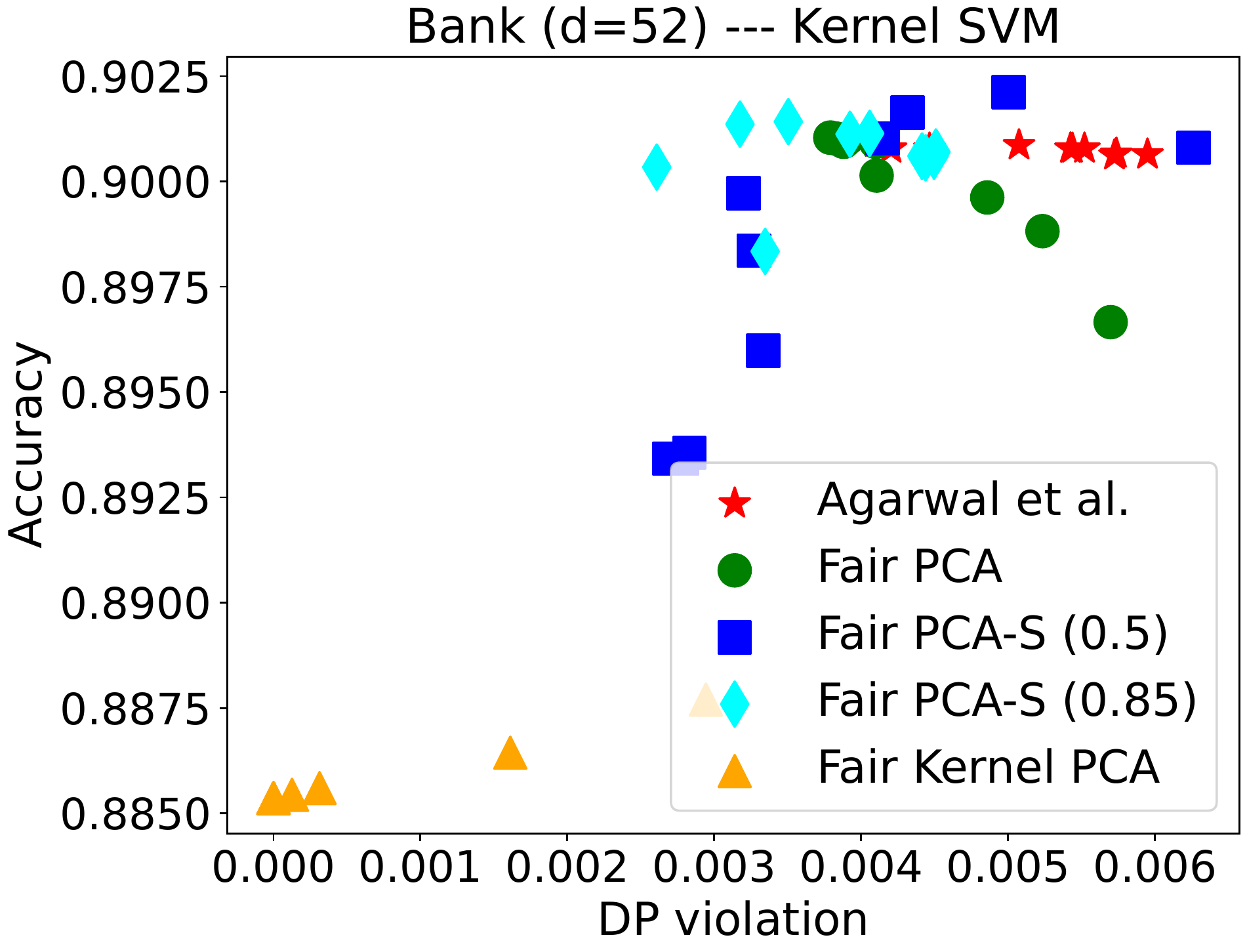}
    \includegraphics[scale=\scaleBiasMitigation]{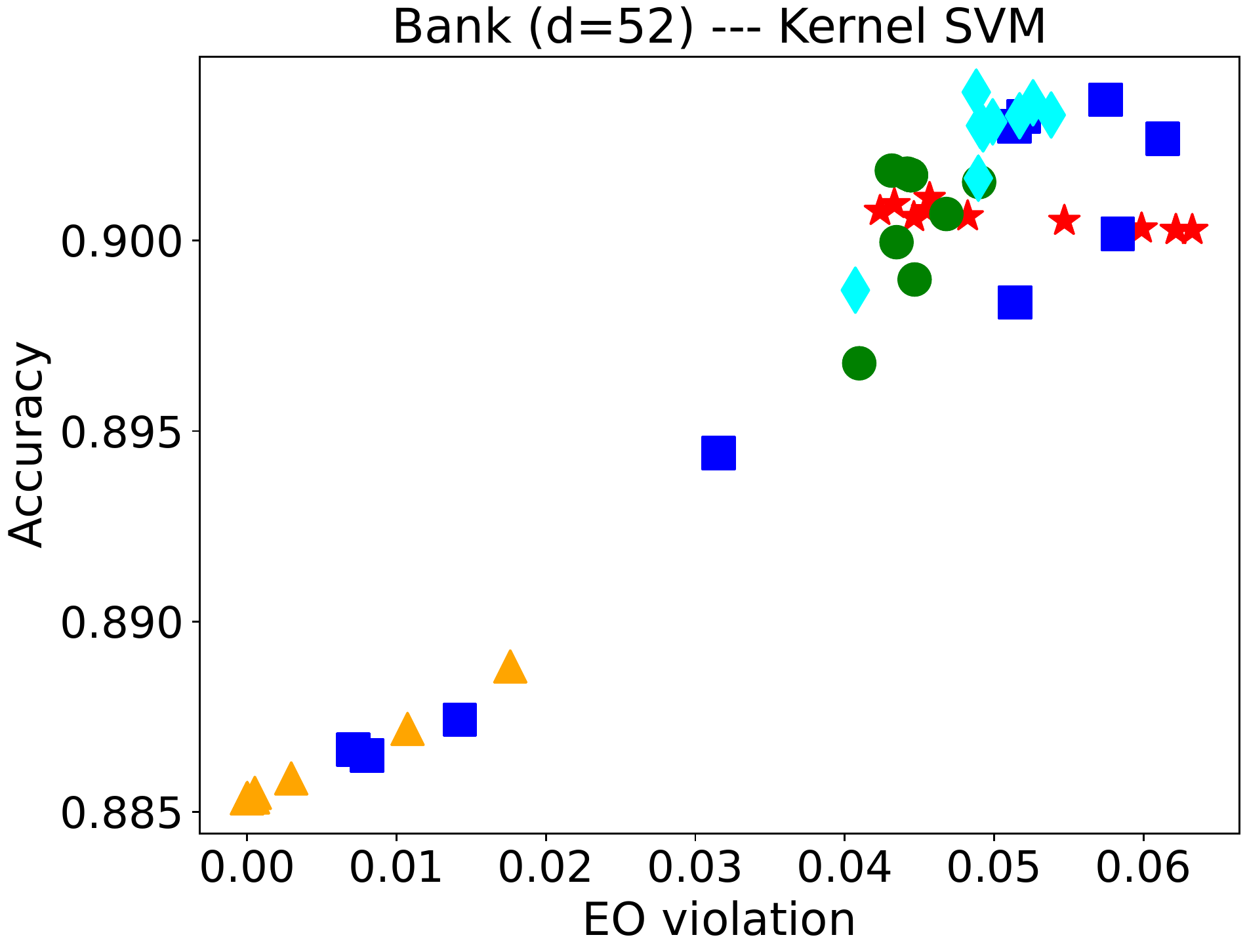}

    \caption{We compare our proposed algorithms, used as pre-processing methods for bias mitigation, to the state-of-the-art in-processing method of \citet{agarwal_reductions_approach}. 
    The first row shows results for the Adult Income dataset, the second row for the Bank Marketing dataset.
    Our algorithms generate 
    comparable 
    trade-off curves, but run much faster (cf. Appendix~\ref{appendix_agarwal_addendum}).}
    \label{fig:bias_mitigation_exp}
\end{figure*}

\newcommand{\scaleCelebA}{1.4cm}

\begin{figure}
    \centering
    \begin{turn}{90} 
     \begin{minipage}{\scaleCelebA}
    \begin{center}
    \begin{small}
    Original
    \\
    dim$=$6400
    \end{small}
    \end{center}
    \end{minipage}
    \end{turn}
     \hspace{-1pt}
    \includegraphics[height=\scaleCelebA]{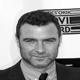}
    \includegraphics[height=\scaleCelebA]{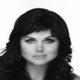}
    \includegraphics[height=\scaleCelebA]{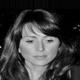}
    \includegraphics[height=\scaleCelebA]{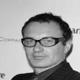}
    \includegraphics[height=\scaleCelebA]{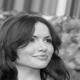}

    \vspace{-1mm}
    \par\noindent\rule{0.485\textwidth}{0.6pt}
    \vspace{-1mm}
    
    \centering
    \begin{turn}{90} 
    \begin{minipage}{\scaleCelebA}
    \begin{center}
    \begin{small}
    Fair PCA
    \\
    $k=6399$
    \end{small}
    \end{center}
    \end{minipage}
    \end{turn}
    \hspace{-1pt}
    \includegraphics[height=\scaleCelebA]{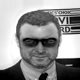}
    \includegraphics[height=\scaleCelebA]{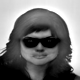}
    \includegraphics[height=\scaleCelebA]{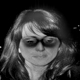}
    \includegraphics[height=\scaleCelebA]{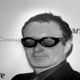}
    \includegraphics[height=\scaleCelebA]{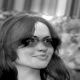}

     \vspace{2mm}
    \centering
    \begin{turn}{90} 
    \begin{minipage}{\scaleCelebA}
    \begin{center}
    \begin{small}
    Fair PCA
    \\
    $k=400$
    \end{small}
    \end{center}
    \end{minipage}
    \end{turn}
    \hspace{-1pt}
    \includegraphics[height=\scaleCelebA]{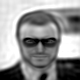}
    \includegraphics[height=\scaleCelebA]{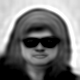}
    \includegraphics[height=\scaleCelebA]{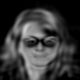}
    \includegraphics[height=\scaleCelebA]{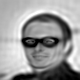}
    \includegraphics[height=\scaleCelebA]{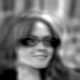}

    \caption{
    Fair PCA 
    applied to the CelebA dataset 
    to erase the concept of ``glasses''. 
    See App.~\ref{app:CelebA} for more examples.  
    }
    \label{fig:celeba_experiment}
\end{figure}

\paragraph{Applying fair PCA to CelebA similarly to \citet{ravfogel2022}}
Similarly to \citet{ravfogel2022}, we applied our fair PCA method to the CelebA dataset \citep{liu2015faceattributes} to erase concepts such as ``glasses'' or ``mustache'' from facial images. The CelebA dataset comprises 202599 pictures of faces of celebrities. 
We rescaled all images to $80\times 80$ grey-scale images 
and applied our Algorithm~\ref{alg:fair_PCA} 
to the flattened raw-pixel vectors, using one of the \emph{bald}, \emph{beard}, \emph{eyeglasses}, \emph{hat}, \emph{mustache}, or \emph{smiling} annotations 
as 
demographic attributes.
Figure~\ref{fig:celeba_experiment} shows 
some 
results for \emph{eyeglasses}; we provide more 
results, also for the other attributes, 
and a discussion 
in Appendix~\ref{app:CelebA}.
Due to their high running time, we were not able to apply the methods by \citet{olfat2019}, \citet{Lee2022},  \citet{ravfogel2020}, or  
\citet{ravfogel2022} to this 
large and high-dimensional dataset. 
However, results for the method of \citet{ravfogel2022} for a smaller resolution~can~be~found~in~their~paper.

\subsection{Comparison with 
\citet{agarwal_reductions_approach}}\label{subsec:experiments_bias_mitigation}

We compare our proposed algorithms as means of bias mitigation to the  state-of-the-art in-processing method of \citet{agarwal_reductions_approach}. While our algorithms 
learn a fair representation and 
perform standard training (without fairness considerations) 
on top of that representation 
to learn a fair classifier, 
the approach of \citeauthor{agarwal_reductions_approach} modifies the training procedure.
Concretely, their approach solves a sequence of cost-sensitive classification problems. We apply the various methods to 
the Adult Income and the Bank Marketing dataset \citep{moro2014},  
which are both available on the UCI repository \citep{Dua:2019}.  The goal for each method is to produce good accuracy vs. fairness trade-off curves---every point on a trade-off curve corresponds to a specific classifier. Note that the approach of \citeauthor{agarwal_reductions_approach} yields randomized classifiers, which is 
problematic if a classifier strongly affects humans' lives \citep{cotter_neurips2019}. 
For our algorithms we deploy the strategy of Section~\ref{subsec:tradeoff} to produce the trade-off curves.
Figure~\ref{fig:bias_mitigation_exp} shows the results.
All results are average results obtained from considering ten random draws of train and test data (see Appendix~\ref{app:details_about_datasets} for details).
The plots show on the y-axis the accuracy of a classifier and on the 
x-axis its fairness violation, which is $\Delta_{DP}=|\Psymb(\hat{Y}=1|Z=0)-\Psymb(\hat{Y}=1|Z=1)|$ as in Section~\ref{subsec:experiments_linear_guarding} when aiming for DP and $\Delta_{EO}:=|\Psymb(\hat{Y}=1|Z=0, Y=1)-\Psymb(\hat{Y}=1|Z=1, Y=1)|$ when aiming for EO. 
In the first and the second plot of each row we learn a logistic regression classifier, aiming to satisfy DP or EO.
We see that fair PCA produces similar curves as the method by \citeauthor{agarwal_reductions_approach} (note that in the bottom left plot $\Delta_{DP}$ is very small for all classifiers). However, fair PCA runs much faster: including the classifier training, 
fair PCA runs for 0.04 seconds on average while the method by \citeauthor{agarwal_reductions_approach} runs for 4.6 seconds (see Appendix~\ref{appendix_agarwal_addendum} for details).
These plots 
do not show results for Fair PCA-S and fair kernel PCA since they cannot compete (we provide those results in Appendix~\ref{appendix_agarwal_addendum}). Fair PCA-S and fair kernel PCA can compete when training a kernel SVM classifier though (third and fourth plot of each row).

\section{DISCUSSION}\label{sec:discussion}

We provided a new derivation 
of 
fair PCA,
aiming for 
a fair representation that does not contain demographic information. Our derivation is simple and 
allows for 
efficient algorithms based on eigenvector computations similar to standard PCA. Compared to existing methods for fair PCA, our proposed algorithms run 
much 
faster while achieving similar results. 
In a comparison with a state-of-the-art in-processing bias mitigation method we saw that our algorithms provide a 
significantly
faster alternative \mbox{to~train~fair~classifiers.}

\newpage

\bibliography{bibfile,bibfile2,bibfile3}

\clearpage
\appendix

\thispagestyle{empty}

\onecolumn

\section*{APPENDIX}

\section{ADDENDUM TO SECTION~\ref{sec:methods}}

\subsection{Problem~\eqref{eq:fair_PCA} May Not Be Well Defined}\label{app:not_well_defined}

Let $n=2n'$, $\xb_1=\xb_2=\ldots=\xb_{n'}=\nullb\in\R^2$, and $\xb_{n'+1}, \ldots \xb_{2n'}\in\R^2$ be equidistantly spread on a circle with center~$\nullb$. Let $z_1=\ldots=z_{n'}=0$,  $z_{n'+1}= \ldots =z_{2n'}=1$, and $k=1$. Any projection onto a 1-dimensional linear subspace maps $\xb_1,\ldots,\xb_{n'}$ to $\nullb$ 
and $\xb_{n'+1}, \ldots, \xb_{2n'}$ onto a line through $\nullb$ such that half the points of  $\xb_{n'+1}, \ldots, \xb_{2n'}$ lie 
on one side of $\nullb$ and the other half lies on the 
other
side of $\nullb$ (at most two of $\xb_{n'+1}, \ldots, \xb_{2n'}$ might map to $\nullb$). The function $h: \R \rightarrow\R$
with 
$h(x)=\charfct[x\neq 0]$ 
(almost) perfectly predicts $z_i$ from the projected points, showing that the set $\mathcal{U}$ defined in \eqref{eq:fair_PCA} can be empty if we require $h(\Ub^\transpose\xb_i)$ and $z_i$ to be independent for \emph{all} functions~$h$.

The same example shows that $\mathcal{U}$ can be empty if we require $h(\Ub^\transpose\xb_i)$ and $z_i$ to be \emph{uncorrelated} (rather than independent) for all functions~$h$.

It also shows that $\mathcal{U}$ can be empty if we require $h(\Ub^\transpose\xb_i)$ and $z_i$ to be  independent for all \emph{linear}  functions~$h$ (rather than all functions~$h$): for $h: \R \rightarrow\R$
with $h(x)=x$,   $h(\Ub^\transpose\xb_i)$ and $z_i$ are clearly dependent.

This shows that we have to relax Problem~\eqref{eq:fair_PCA} in two ways in order to arrive at a well defined problem.

\begin{algorithm}[t!]
   \caption{Fair PCA (for multiple demographic groups)
   }\label{alg:fair_PCA_multi_groups}
\begin{algorithmic}
   \STATE {\bfseries Input:} data matrix $\Xb\in\R^{d\times n}$; demographic attributes~$z_i^{(1)},\ldots,z_i^{(m)}\in\{0,1\}$, $i\in[n]$, where $z_i^{(l)}$ encodes membership of the $i$-th datapoint in the $l$-th group;
   target dimension
   $k\in[d-m+1]$

\vspace{1mm}
   \STATE {\bfseries Output:} a solution $\Ub$ to the multi-group version of Problem~\eqref{eq:fair_PCA_relaxed}

   \begin{itemize}[leftmargin=*]
   \setlength{\itemsep}{-2pt}
   \item set $\Zb\in\R^{n\times m}$ with the $l$-th column of $\Zb$ equaling $(z_1^{(l)}-\bar{z}^{(l)},\ldots,z_n^{(l)}-\bar{z}^{(l)})^\transpose$ with $\bar{z}^{(l)}=\frac{1}{n} \sum_{i=1}^n z_i^{(l)}$ 
\item compute an orthonormal basis of the nullspace of $\Zb^\transpose\Xb^\transpose$ and build  matrix~$\Rb$ comprising the basis vectors as columns
  \item compute orthonormal eigenvectors, corresponding to the largest $k$ eigenvalues, of $\Rb^\transpose\Xb\Xb^\transpose\Rb$ and build matrix~$\Lamb$ comprising the eigenvectors as columns
  \item return $\Ub=\Rb\Lamb$
   \end{itemize}
\end{algorithmic}
\end{algorithm}

\section{ADDENDUM TO SECTION~\ref{sec:extensions}}

\subsection{Fair PCA for Multiple Demographic Groups}\label{app:multiple_groups}

In fair PCA for multiple groups we want to solve
\begin{align}\label{eq:fair_PCA_multi_groups_app}
    \argmax_{\Ub\in\R^{d\times k}:\, \Ub^\transpose\Ub=\Idk}\trace(\Ub^\transpose\Xb\Xb^\transpose\Ub)\quad \text{subject to}\quad \Zb^\transpose\Xb^\transpose\Ub=\nullb,
\end{align}
where $\Zb\in\R^{n\times m}$ and the $l$-th column of $\Zb$ equals $(z_1^{(l)}-\bar{z}^{(l)},\ldots,z_n^{(l)}-\bar{z}^{(l)})^\transpose$ with $\bar{z}^{(l)}=\frac{1}{n} \sum_{i=1}^n z_i^{(l)}$ and $z_i^{(l)}=1$ if $\xb_i$ belongs to group~$l$ and  $z_i^{(l)}=0$ otherwise. Assuming that no group is empty, the rank of $\Zb$ 
is $m-1$ as $\sum_{l=1}^m \Zb^{(l)}_i=0$, $i\in[n]$, and in any linear combination of $(m-1)$ many columns of $\Zb$ equaling zero all coefficients must be zero. Hence, $\rank(\Zb^\transpose\Xb^\transpose)\leq \rank(\Zb^\transpose)=\rank(\Zb)=m-1$ and the nullspace of $\Zb^\transpose\Xb^\transpose$ has dimension at least $d-m+1$.
Let $\Rb\in\R^{d\times s}$ with $s\geq d-m+1$ comprise as columns an orthonormal basis of 
the nullspace of $\Zb^\transpose\Xb^\transpose$. We can then substitute $\Ub=\Rb\Lamb$ for $\Lamb\in\R^{s\times k}$. The constraint $\Ub^\transpose\Ub=\Idk$ becomes $\Lamb^\transpose\Lamb=\Idk$, and the objective $\trace(\Ub^\transpose\Xb\Xb^\transpose\Ub)$ becomes $\trace(\Lamb^\transpose\Rb^\transpose\Xb\Xb^\transpose\Rb\Lamb)$. Hence, we can compute $\Lamb$ by computing  eigenvectors, corresponding to the largest $k$ eigenvalues, of $\Rb^\transpose\Xb\Xb^\transpose\Rb$. This requires $k\leq s$, which is guaranteed to hold for 
$k\leq d-m+1$.

If $m=2$, then the first and the second column of $\Zb$ coincide up to  multiplication by $-1$ and the nullspace of $\Zb^\transpose\Xb^\transpose$ is the same as if we removed one of the two columns from $\Zb$. This shows that for two groups, fair PCA as presented here is equivalent to fair PCA as presented in Section~\ref{sec:methods}.  

Finally, the interpretation of fair PCA provided in Section~\ref{sec:methods} also applies to the case of multiple groups: $\Zb^\transpose\Xb^\transpose\Ub = \nullb$ is equivalent to  
\begin{align*}
\frac{1}{|\{i:\xb_i\in \text{group $l$}\}|}\sum_{i:\, \xb_i\,\in \,\text{group $l$}} \Ub^\transpose\xb_i = \frac{1}{|\{i:\xb_i\notin \text{group $l$}\}|}\sum_{i:\, \xb_i\,\notin\, \text{group $l$}} \Ub^\transpose\xb_i,\quad l=1,\ldots,m,
\end{align*}
which in turn is equivalent to the projected data's group-conditional means to coincide for all groups. Hence, 
an analogous version of Proposition~\ref{prop:gaussian_data} holds true for multiple groups.

The pseudo code of fair PCA for multiple demographic groups is provided in Algorithm~\ref{alg:fair_PCA_multi_groups}. The pseudo code of fair kernel PCA for multiple demographic groups is provided in Algorithm~\ref{alg:fair_PCA_kernelized}.

\begin{algorithm}[t]
   \caption{Fair Kernel  PCA (for multiple demographic groups)
   }\label{alg:fair_PCA_kernelized}
\begin{algorithmic}
   \STATE {\bfseries Input:} kernel matrix $\Kb\in\R^{n\times n}$ with $\Kb_{ij}=k(\xb_i,\xb_j)$ for some kernel function~$k$; demographic~attributes $z_i^{(1)},\ldots,z_i^{(m)}\in\{0,1\}$, $i\in[n]$, where $z_i^{(l)}$ encodes membership of $\xb_i$ in the $l$-th group;
   target dimension~$k\in[n-m+1]$; \emph{optional:} kernel matrix $\mathbf{\hat{K}}\in\R^{n\times n'}$ with $\mathbf{\hat{K}}_{ij}=k(\xb_i,\xb'_j)$, $i\in[n],j\in[n']$, %
    for test data~$\xb'_1,\ldots,\xb'_{n'}$

\vspace{1mm}
   \STATE {\bfseries Output:} $k$-dimensional representation of the training data~$\xb_1,\ldots,\xb_n$; \emph{optional:} $k$-dimensional representation of the test data~$\xb'_1,\ldots,\xb'_{n'}$

   \begin{itemize}[leftmargin=*]
   \setlength{\itemsep}{-2pt}
   \item set $\Zb\in\R^{n\times m}$ with the $l$-th column of $\Zb$ equaling $(z_1^{(l)}-\bar{z}^{(l)},\ldots,z_n^{(l)}-\bar{z}^{(l)})^\transpose$ with $\bar{z}^{(l)}=\frac{1}{n} \sum_{i=1}^n z_i^{(l)}$ 
\item compute an orthonormal basis of the nullspace of $\Zb^\transpose\Kb$ and build  matrix~$\Rb$ comprising the basis vectors as columns
  \item compute orthonormal eigenvectors, corresponding to the largest $k$ eigenvalues, of the generalized eigenvalue problem $\Rb^\transpose\Kb\Kb\Rb\Lamb=\Rb^\transpose\Kb\Rb\Lamb \newblock$; here, the matrix~$\Lamb$ comprises the eigenvectors as columns and $\Wb$ is a diagonal matrix containing the eigenvalues
  \item return $\Lamb^\transpose\Rb^\transpose\Kb$ as the representation of the training data;  \emph{optional:} return $\Lamb^\transpose\Rb^\transpose\mathbf{\hat{K}}$ as the representation of the test data
   \end{itemize}
\end{algorithmic}
\end{algorithm}

\section{ADDENDUM TO SECTION~\ref{sec:experiments}}

\subsection{Implementation Details}\label{app:implementation_details}

\paragraph{General details}

\begin{itemize}

\item \textbf{Solving generalized eigenvalue problem for fair kernel PCA:} Fair kernel PCA requires to solve a generalized eigenvalue problem of the form $\Ab\xb=\lambda \Bb\xb$ for square matrices~$\Ab$ and $\Bb$
that are given as input. 
In fair kernel PCA, 
$\Bb$ is 
guaranteed to be symmetric positive semi-definite, but not 
necessarily 
positive definite 
(and so is $\Ab$). We use the function \texttt{eigsh} from SciPy (\url{https://docs.scipy.org/doc/scipy/reference/generated/scipy.sparse.linalg.eigsh.html}) to solve the generalized eigenvalue problem. While \texttt{eigsh} allows for a 
positive semi-definite $\Bb$, it requires a parameter~\texttt{sigma} to use the shift-invert mode in this case (this is in contrast to the function~\texttt{eig} in Matlab, which does not require such a parameter and automatically chooses the best algorithm to solve the  generalized eigenvalue problem in case of a singular $\Bb$; cf. \url{https://de.mathworks.com/help/matlab/ref/eig.html}). In order to avoid having to look for an appropriate value of \texttt{sigma}, when \texttt{eigsh} would require its specification, we simply add $10^{-5}\cdot \Ib$ to $\Bb$, where $\Ib$ is the identity matrix, to guarantee that $\Bb$ is positive definite.  This is a common practice in the context of kernel methods to avoid numerical instabilities \citep[see, e.g., ][Section 1.2]{williams2000}.

    \item \textbf{Bandwith for fair kernel PCA:} When running our 
proposed 
fair kernel PCA algorithm with a Gaussian kernel, we set the 
parameter~$\gamma$ of the kernel function (cf. \url{https://scikit-learn.org/stable/modules/generated/sklearn.metrics.pairwise.rbf_kernel.html#sklearn.metrics.pairwise.rbf_kernel}) to $1/(d\cdot \Var(\text{training data}))$, where $d$ is the dimension of the data (i.e., number of features)  and $\Var(\text{training data})$ the variance of the flattened training data array. This value of $\gamma$ is the default value in Scikit-learn's kernel SVM implementation (cf. \url{https://scikit-learn.org/stable/modules/generated/sklearn.svm.SVC.html}).
\end{itemize}

\paragraph{Details for the experiments of Section~\ref{subsec:experiments_linear_guarding}}

We used the experimental setup and code of \citet{Lee2022}. Hence, most implementation details can be found in their paper or code repository. In addition, we provide the following details:

\begin{itemize}

\item \textbf{Training additional classifiers for evaluating representations:} In addition to the metrics reported by \citet{Lee2022}, we reported the 
accuracy and $\Delta_{DP}$ of a linear support vector machine (SVM)  
and a 
multilayer perceptron (MLP). We trained the linear SVM in Matlab using the function \texttt{fitcsvm} (\url{https://de.mathworks.com/help/stats/fitcsvm.html}) and the MLP using Scikit-learn's \texttt{MLPClassifier} class (\url{https://scikit-learn.org/stable/modules/generated/sklearn.neural_network.MLPClassifier.html}) with all parameters set to the default values (except for \texttt{hidden\_layer\_sizes}, \texttt{max\_iter}, and \texttt{random\_state} for the MLP).

\end{itemize}

\paragraph{Details for the experiments of Section~\ref{subsec:experiments_bias_mitigation}} 

\begin{itemize}
    \item \textbf{Data normalization:} We normalized the data to have zero mean and unit variance on the training data.

\item \textbf{Target dimension for our methods:} As target dimension~$k$ we chose $k=d-1$, where $d$ is the data dimension, for fair PCA and fair kernel PCA, $k=\lfloor d/4\rfloor$ for Fair PCA-S (0.5), and $k=\lfloor d/2\rfloor$ for Fair PCA-S (0.85).

\item \textbf{Controlling accuracy vs. fairness trade-off:} For our methods, we deployed the strategy described in Section~\ref{subsec:tradeoff} to trade off accuracy vs. fairness. For the reductions approach of \citet{agarwal_reductions_approach}, we controlled the trade-off by varying the parameter \texttt{difference\_bound} in the classes \texttt{DemographicParity} or \texttt{TruePositiveRateParity}, which implement the fairness constraints. For all methods, we used 11 parameter values for generating the trade-off curves. For our methods, we set the fairness parameter~$\lambda$ of Section~\ref{subsec:tradeoff} to $(i/10)^3$, $i=0,1,\ldots, 10$. For the approach of \citeauthor{agarwal_reductions_approach} we set 
\texttt{difference\_bound} to 
 0.001, 0.005, 0.01, 0.015, 0.02, 0.03, 0.05, 0.07, 0.1, 0.15, 0.2.

\item \textbf{Regularization parameters:} We trained the logistic regression classifier using Scikit-learn (\url{https://scikit-learn.org/stable/modules/generated/sklearn.linear_model.LogisticRegression.html}) with regularization parameter~$C=1/(2\cdot \text{size of training data}\cdot  0.01)$ and the kernel SVM classifier using Scikit-learn (\url{https://scikit-learn.org/stable/modules/generated/sklearn.svm.SVC.html}) with regularization parameter~$C=1/(2\cdot \text{size of training data}\cdot 0.00005)$. By default, both classifiers are trained with $l_2$-regularization.

\end{itemize}

\newcommand{\scaleDistributionPlots}{0.18}
\newcommand{\abstDistributionPlots}{-3mm}
\begin{figure}[t]
    \centering
    \includegraphics[scale=\scaleDistributionPlots]{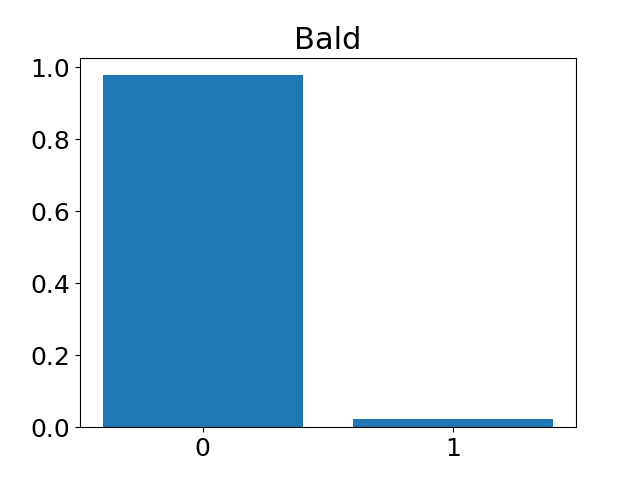}
    \hspace{\abstDistributionPlots}
    \includegraphics[scale=\scaleDistributionPlots]{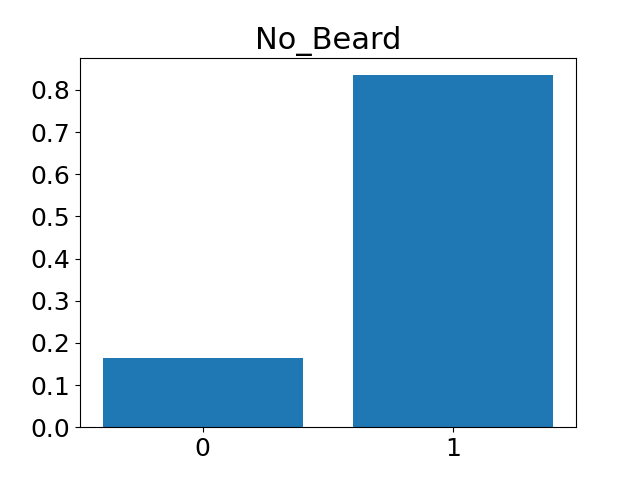}
    \hspace{\abstDistributionPlots}
    \includegraphics[scale=\scaleDistributionPlots]{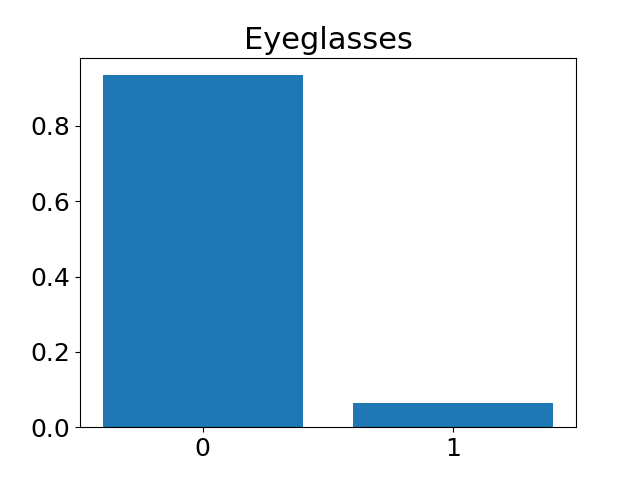}
    \hspace{\abstDistributionPlots}
    \includegraphics[scale=\scaleDistributionPlots]{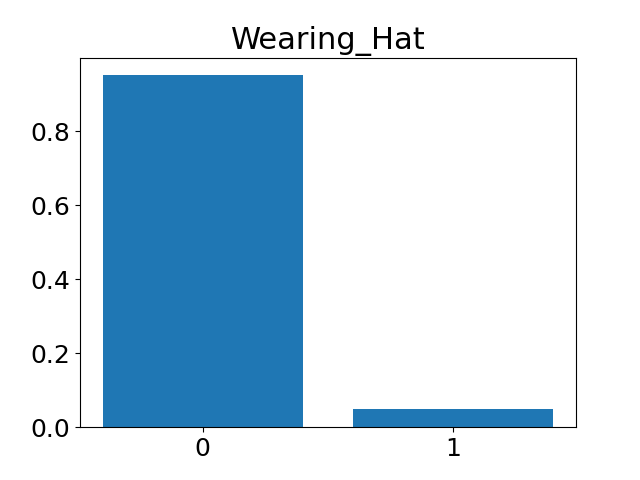} 
    \hspace{\abstDistributionPlots}
    \includegraphics[scale=\scaleDistributionPlots]{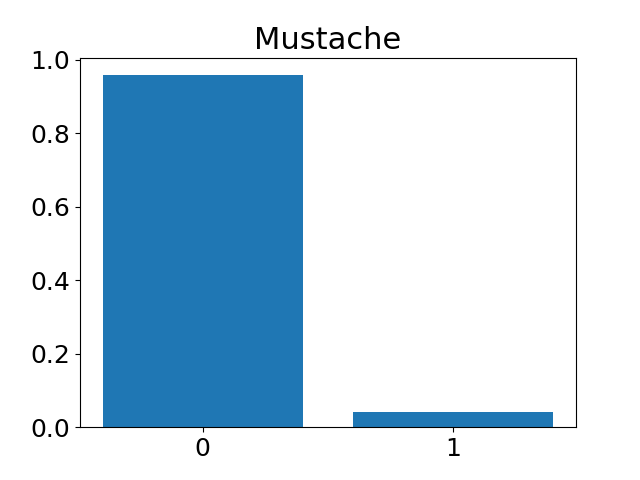}
    \hspace{\abstDistributionPlots}
    \includegraphics[scale=\scaleDistributionPlots]{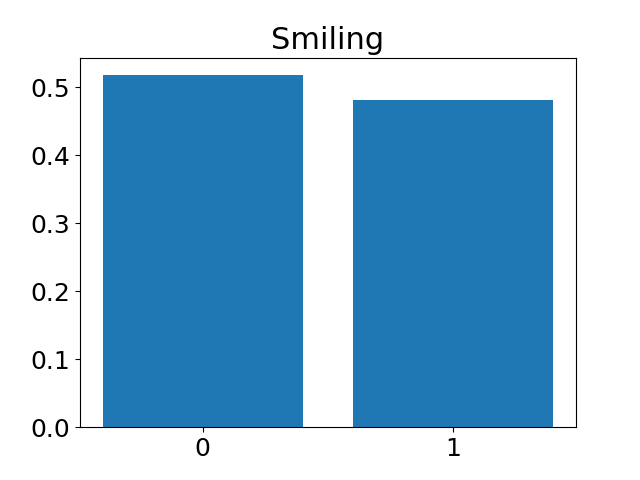}
    \caption{Distributions of the 
    attributes \emph{bald}, \emph{beard}, \emph{eyeglasses}, \emph{hat}, \emph{mustache}, and \emph{smiling} in the CelebA dataset.}
    \label{fig:CelebA_distribution_attributes}
\end{figure}

\subsection{Details about Datasets}\label{app:details_about_datasets}

\paragraph{Adult Income dataset 
\citep{Dua:2019}
} The Adult Income dataset is available on the UCI repository \citep{Dua:2019}. Each record comprises 14 features (before one-hot encoding categorical ones) for an individual, such as their education or marital status, and the task is to predict whether 
an individual makes 
more than \$50k per year 
or not (distribution: 23.9\% yes - 76.1\% no). In Section~\ref{subsec:experiments_linear_guarding}, we used the dataset as provided by  \citet{Lee2022}. They removed the features ``fnlwgt'' and ``race'', and they subsampled the dataset to comprise 2261 records (cf. Appendix~I.3 in their paper). They used the binary feature ``sex'' as demographic attribute (distribution: 66.8\% male - 33.2\% female). In our comparison with the method of \citet{agarwal_reductions_approach} presented in Section~\ref{subsec:experiments_bias_mitigation}, we also used  ``sex'' as demographic attribute; however, we did not remove any features and we 
randomly 
subsampled the dataset to comprise 5000 records for training and 5000 different records for evaluation (i.e., computing a classifier's accuracy and fairness violation). 
In the runtime comparison of Appendix~\ref{appendix_agarwal_addendum} we used between 1000 and 40000 
randomly sampled 
records for training.

\paragraph{Bank Marketing dataset \citep{moro2014,Dua:2019}}
The Bank Marketing dataset is available on the UCI repository \citep{Dua:2019}. There are four versions available. We worked with the file \texttt{bank-additional-full.csv}. Each record comprises 20 features (before one-hot encoding categorical ones) for an individual, and the task is to predict whether 
an individual subscribes a term deposit  
or not (distribution: 11.3\% yes - 88.7\% no). We used 
a person's binarized age 
(older than 40 vs. not older than 40)
as demographic attribute (distribution: 42.3\% older than 40 - 57.7\% not older than 40), and we 
randomly 
subsampled the dataset to comprise 5000 records for training and 5000 different records for evaluation.

\paragraph{CelebA dataset \citep{liu2015faceattributes}}  The CelebA dataset comprises 202599 pictures of faces of celebrities
together with 
40 binary 
attribute 
annotations for each picture. The dataset comes in two versions: one that provides in-the-wild images, which may not only show a person's face, but also their upper body, and one that provides aligned-and-cropped images, which only show a person's face. For our experiment, we used the latter one. We used one of  the \emph{bald}, \emph{beard}, \emph{eyeglasses}, \emph{hat}, \emph{mustache}, or \emph{smiling} annotations as demographic attributes. The distributions of these attributes can be seen in Figure~\ref{fig:CelebA_distribution_attributes}.

\paragraph{COMPAS dataset \citep{angwin2016}} The COMPAS dataset is available on \url{https://github.com/propublica/compas-analysis}. We used the dataset as provided by \citet{Lee2022}. They subsampled the dataset to comprise 2468 datapoints, removed the features ``sex'' and ``c\_charge\_desc'', and used the feature ``Race'' for defining the demographic attribute (cf. Appendix~I.1 in their paper).

\paragraph{German Credit \citep{Dua:2019}} The German Credit dataset  is available on the UCI repository \citep{Dua:2019}. It  comprises 1000 datapoints. We used the dataset as provided by \citet{Lee2022}. They removed the features ``sex'' and ``personal\_status'' and used the feature ``Age'' for defining the demographic attribute (cf. Appendix~I.2 in their paper).

\begin{figure}[t]
    \centering
    \includegraphics[scale=0.3]{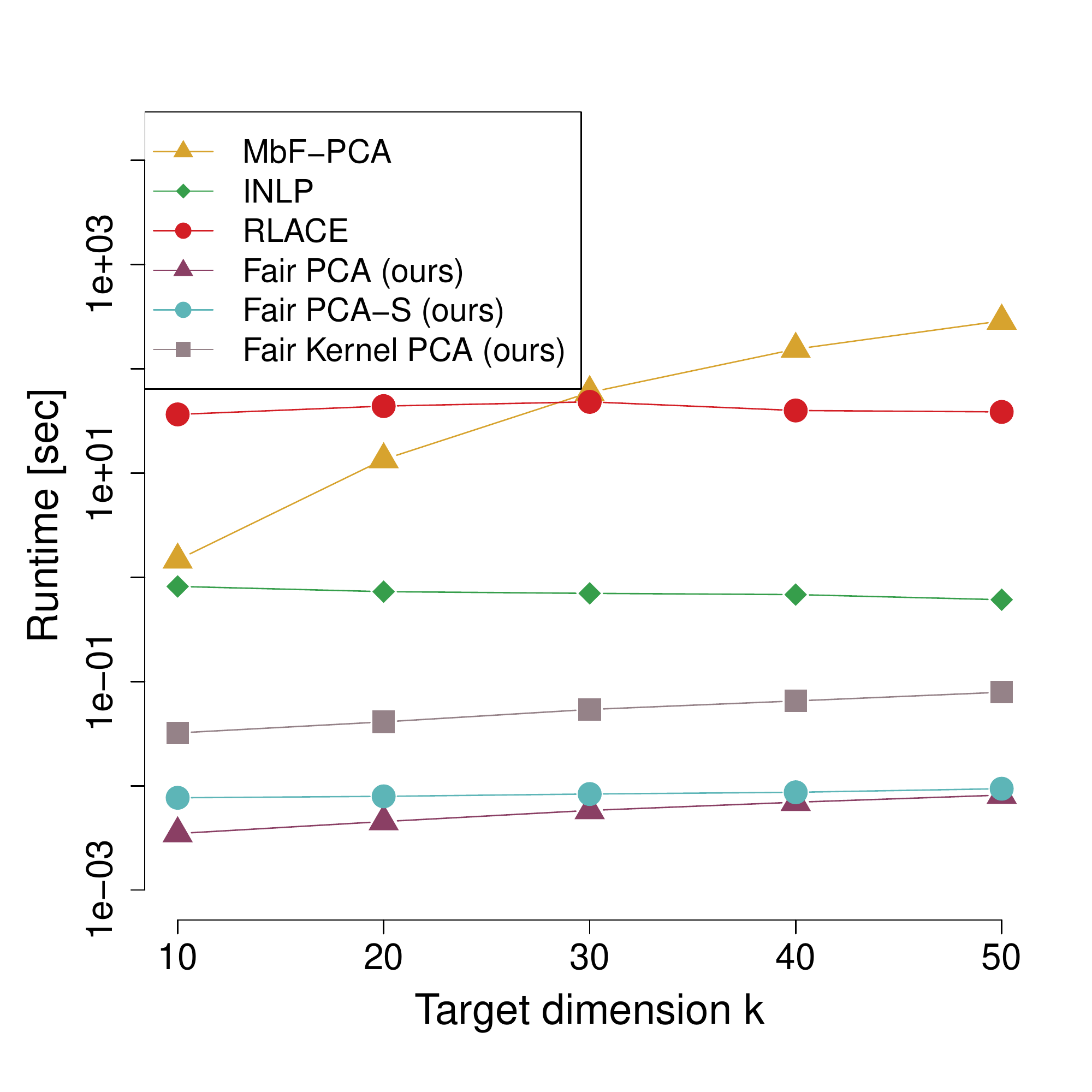}
    \caption{The running time of  MbF-PCA \citep{Lee2022}, INLP \citep{ravfogel2020}, RLACE 
\citep{ravfogel2022} and our proposed methods as a function of the target dimension~$k$. The data dimension~$d$ equals 100. Note the logarithmic~y-axis.}
    \label{fig:runtime_as_function_of_k}
\end{figure}

\subsection{Another Runtime Comparison}\label{app:runtime_comparison}

Figure~\ref{fig:runtime_as_function_of_k} provides a comparison of the running times of the various methods (except for FPCA, which we  already have seen to run extremely slow in Figure~\ref{fig:MMD_fair_pca_exp_2}) in a 
related, but different 
setting as in the experiments of Section~\ref{subsec:experiments_linear_guarding} on the synthetic data. The data is generated in the same way as in Section~\ref{subsec:experiments_linear_guarding}, but now we vary the target dimension~$k$ and hold the data dimension~$d$ constant at 100. We see that while the running time of our proposed methods only moderately increases with~$k$, the running time of MbF-PCA drastically increases with~$k$. Note that the running time of INLP even decreases with~$k$, which is by 
the 
design of the method (cf. Section~\ref{sec:related_work}).

\subsection{Tables for German Credit and COMPAS}\label{app:tables}

Table~\ref{tab:German} and Table~\ref{tab:COMPAS} 
provide the results of the experiments of Section~\ref{subsec:experiments_linear_guarding} on the real  data for the German Credit 
dataset 
and 
the 
COMPAS dataset, respectively. 
Note that the dimension of the COMPAS dataset is rather small---in particular, for $k=10$, we do not expect Fair PCA-S (0.5) or Fair PCA-S (0.85) to behave any differently than fair PCA since $l=\max\{k,\lfloor f\cdot d\rfloor\}=k=d-1$ for both $f=0.5$ and $f=0.85$.

\begin{table*}[t]
    \centering
\caption{Similar table as Table~\ref{tab:Adult} for the German Credit dataset.}
    \label{tab:German}

    \vspace{3mm}
\renewcommand{\arraystretch}{1.15}
\begin{scriptsize}
\begin{tabular}{c|c|cccccccc}
\hline
\multicolumn{10}{c}{\normalsize{\textbf{German Credit} [$\text{feature dim}=57$, $\Psymb(Y=1)=0.3020$]}}\\
\multirow{2}{*}{$k$} & \multirow{2}{*}{Algorithm} & \multirow{2}{*}{\%Var{\scriptsize ($\uparrow$)}} & \multirow{2}{*}{MMD$^2${\scriptsize ($\downarrow$)}} & \%Acc{\scriptsize ($\uparrow$)} & $\Delta_{DP}${\scriptsize ($\downarrow$)}  & \%Acc{\scriptsize ($\uparrow$)}  & $\Delta_{DP}${\scriptsize ($\downarrow$)}  & \%Acc{\scriptsize ($\uparrow$)}  & $\Delta_{DP}${\scriptsize ($\downarrow$)} \\
 & & & & \multicolumn{2}{c}{Kernel SVM} & \multicolumn{2}{c}{Linear SVM} & \multicolumn{2}{c}{MLP}\\\hline
\multirow{11}{*}{2} & PCA & $\mathbf{11.42_{0.45}}$ & $\mathbf{0.147_{0.047}}$ & $\mathbf{76.87_{1.32}}$ & $\mathbf{0.12_{0.06}}$ & $\mathbf{69.8_{1.21}}$ & $\mathbf{0.0_{0.0}}$ & $\mathbf{71.7_{1.83}}$ & $\mathbf{0.09_{0.09}}$ \\
\cdashline{2-10}
 & FPCA (0.1, 0.01) & $7.43_{0.56}$ & $0.017_{0.009}$ & $72.17_{1.04}$ & $0.03_{0.02}$ & $\mathbf{69.8_{1.21}}$ & $\mathbf{0.0_{0.0}}$ & $70.27_{1.38}$ & $0.01_{0.01}$ \\
 & FPCA (0, 0.01) & $7.33_{0.54}$ & $0.015_{0.01}$ & $71.77_{1.52}$ & $0.03_{0.02}$ & $\mathbf{69.8_{1.21}}$ & $\mathbf{0.0_{0.0}}$ & $69.83_{1.49}$ & $\mathbf{0.0_{0.01}}$ \\
 & MbF-PCA ($10^{-3}$) & $\mathbf{10.34_{0.57}}$ & $0.019_{0.014}$ & $\mathbf{74.87_{1.92}}$ & $0.04_{0.04}$ & $\mathbf{69.8_{1.21}}$ & $\mathbf{0.0_{0.0}}$ & $\mathbf{71.43_{2.08}}$ & $0.04_{0.05}$ \\
 & MbF-PCA ($10^{-6}$) & $9.38_{0.3}$ & $0.016_{0.009}$ & $73.97_{1.59}$ & $0.03_{0.02}$ & $\mathbf{69.8_{1.21}}$ & $\mathbf{0.0_{0.0}}$ & $70.83_{1.66}$ & $0.03_{0.03}$ \\
 & INLP & $2.99_{0.39}$ & $\mathbf{0.007_{0.004}} $& $70.93_{1.27}$ & $\mathbf{0.02_{0.02}} $& $\mathbf{69.8_{1.21}}$ & $\mathbf{0.0_{0.0}} $& $70.17_{1.56}$ & $0.01_{0.02}$ \\
 & RLACE & $3.62_{0.27}$ & $0.042_{0.027} $& $71.5_{1.75}$ & $\mathbf{0.02_{0.02}} $& $\mathbf{69.8_{1.21}}$ & $\mathbf{0.0_{0.0}} $& $70.23_{1.5}$ & $0.02_{0.01}$ \\
\cdashline{2-10}
 & Fair PCA & $\mathbf{10.85_{0.55}}$ & $0.025_{0.016}$ & $\mathbf{75.6_{1.89}}$ & $0.06_{0.05}$ & $\mathbf{69.8_{1.21}}$ & $\mathbf{0.0_{0.0}}$ & $\mathbf{72.03_{1.98}}$ & $0.04_{0.05}$ \\
 & Fair Kernel PCA & $n/a_{}$ & $n/a_{}$ & $69.8_{1.21}$ & $\mathbf{0.0_{0.0}}$ & $\mathbf{69.8_{1.21}}$ & $\mathbf{0.0_{0.0}}$ & $69.8_{1.21}$ & $\mathbf{0.0_{0.0}}$ \\
 & Fair PCA-S (0.5) & $4.73_{0.43}$ & $\mathbf{0.01_{0.006}}$ & $72.47_{3.14}$ & $0.02_{0.02}$ & $\mathbf{69.8_{1.21}}$ & $\mathbf{0.0_{0.0}}$ & $70.93_{2.21}$ & $0.01_{0.01}$ \\
 & Fair PCA-S (0.85) & $7.43_{0.42}$ & $0.018_{0.011}$ & $72.93_{2.05}$ & $0.02_{0.02}$ & $\mathbf{69.8_{1.21}}$ & $\mathbf{0.0_{0.0}}$ & $70.0_{1.83}$ & $0.02_{0.02}$ \\
\hline
\multirow{11}{*}{10} & PCA & $\mathbf{38.24_{0.92}}$ & $\mathbf{0.13_{0.018}}$ & $\mathbf{99.93_{0.13}}$ & $\mathbf{0.12_{0.07}}$ & $\mathbf{74.8_{1.93}}$ & $\mathbf{0.15_{0.11}}$ & $\mathbf{96.87_{2.08}}$ & $\mathbf{0.11_{0.08}}$ \\
\cdashline{2-10}
 & FPCA (0.1, 0.01) & $29.85_{0.82}$ & $0.02_{0.005}$ & $\mathbf{99.93_{0.13}}$ & $0.12_{0.07}$ & $71.13_{2.75}$ & $0.02_{0.03}$ & $\mathbf{96.77_{1.8}}$ & $0.1_{0.07}$ \\
 & FPCA (0, 0.01) & $29.74_{0.84}$ & $0.02_{0.005}$ & $\mathbf{99.93_{0.13}}$ & $0.12_{0.07}$ & $70.87_{2.38}$ & $0.02_{0.04}$ & $96.4_{1.77}$ & $0.1_{0.07}$ \\
 & MbF-PCA ($10^{-3}$) & $\mathbf{34.07_{1.0}}$ & $0.019_{0.007}$ & $\mathbf{99.93_{0.13}}$ & $0.12_{0.07}$ & $\mathbf{73.7_{2.58}}$ & $0.05_{0.04}$ & $96.67_{1.04}$ & $0.11_{0.06}$ \\
 & MbF-PCA ($10^{-6}$) & $16.82_{1.11}$ & $\mathbf{0.011_{0.007}}$ & $94.37_{2.63}$ & $0.12_{0.06}$ & $70.1_{0.63}$ & $\mathbf{0.0_{0.0}}$ & $80.07_{3.52}$ & $\mathbf{0.06_{0.05}}$ \\
 & INLP & $15.5_{0.92}$ & $\mathbf{0.011_{0.002}} $& $98.83_{0.79}$ & $\mathbf{0.11_{0.07}} $& $69.8_{1.21}$ & $\mathbf{0.0_{0.0}} $& $94.2_{2.26}$ & $0.1_{0.06}$ \\
 & RLACE & $17.24_{0.75}$ & $0.03_{0.023} $& $99.73_{0.29}$ & $0.12_{0.07} $& $70.97_{2.31}$ & $0.02_{0.03} $& $95.43_{2.89}$ & $0.12_{0.07}$ \\
\cdashline{2-10}
 & Fair PCA & $\mathbf{36.63_{1.04}}$ & $0.022_{0.008}$ & $\mathbf{99.93_{0.13}}$ & $0.12_{0.07}$ & $\mathbf{74.1_{2.23}}$ & $0.05_{0.04}$ & $96.03_{2.26}$ & $0.11_{0.04}$ \\
 & Fair Kernel PCA & $n/a_{}$ & $n/a_{}$ & $70.1_{1.18}$ & $\mathbf{0.0_{0.01}}$ & $69.8_{1.21}$ & $\mathbf{0.0_{0.0}}$ & $74.07_{2.43}$ & $\mathbf{0.06_{0.03}}$ \\
 & Fair PCA-S (0.5) & $20.51_{0.79}$ & $\mathbf{0.013_{0.006}}$ & $99.87_{0.22}$ & $0.12_{0.08}$ & $71.6_{2.83}$ & $0.02_{0.04}$  & $95.13_{2.52}$ & $0.08_{0.06}$ \\
 & Fair PCA-S (0.85) & $28.83_{0.82}$ & $0.018_{0.007}$ & $\mathbf{99.93_{0.13}}$ & $0.12_{0.07}$ & $71.87_{2.62}$ & $0.03_{0.04}$& $\mathbf{96.27_{2.0}}$ & $0.09_{0.07}$ \\
\hline
\end{tabular}
\end{scriptsize}
\end{table*}

\begin{table*}[t]
    \centering
\caption{Similar table as Table~\ref{tab:Adult} for the COMPAS dataset. }
    \label{tab:COMPAS}

    \vspace{3mm}
\renewcommand{\arraystretch}{1.15}
\begin{scriptsize}
\begin{tabular}{c|c|cccccccc}
\hline
\multicolumn{10}{c}{\normalsize{\textbf{COMPAS} [$\text{feature dim}=11$, $\Psymb(Y=1)=0.4548$]}}\\
\multirow{2}{*}{$k$} & \multirow{2}{*}{Algorithm} & \multirow{2}{*}{\%Var{\scriptsize ($\uparrow$)}} & \multirow{2}{*}{MMD$^2${\scriptsize ($\downarrow$)}} & \%Acc{\scriptsize ($\uparrow$)} & $\Delta_{DP}${\scriptsize ($\downarrow$)}  & \%Acc{\scriptsize ($\uparrow$)}  & $\Delta_{DP}${\scriptsize ($\downarrow$)}  & \%Acc{\scriptsize ($\uparrow$)}  & $\Delta_{DP}${\scriptsize ($\downarrow$)} \\
 & & & & \multicolumn{2}{c}{Kernel SVM} & \multicolumn{2}{c}{Linear SVM} & \multicolumn{2}{c}{MLP}\\\hline
\multirow{11}{*}{2} & PCA & $\mathbf{39.28_{4.91}}$ & $\mathbf{0.092_{0.009}}$ & $\mathbf{64.53_{1.38}}$ & $\mathbf{0.29_{0.08}}$ & $\mathbf{56.69_{1.52}}$ & $\mathbf{0.2_{0.09}}$ & $\mathbf{61.77_{2.81}}$ & $\mathbf{0.28_{0.06}}$ \\
\cdashline{2-10}
 & FPCA (0.1, 0.01) & $\mathbf{35.06_{4.9}}$ & $0.012_{0.007}$ & $61.65_{1.11}$ & $0.1_{0.06}$ & $56.23_{1.19}$ & $0.04_{0.03}$ & $57.61_{1.67}$ & $0.08_{0.04}$ \\
 & FPCA (0, 0.01) & $34.43_{4.76}$ & $0.011_{0.006}$ & $60.86_{1.03}$ & $0.11_{0.06}$ & $55.9_{1.26}$ & $0.03_{0.03}$ & $56.9_{1.88}$ & $0.09_{0.03}$ \\
 & MbF-PCA ($10^{-3}$) & $34.24_{3.68}$ & $0.006_{0.003}$ & $\mathbf{64.78_{0.96}}$ & $0.12_{0.05}$ & $56.92_{2.7}$ & $0.07_{0.06}$ & $\mathbf{60.53_{1.46}}$ & $0.1_{0.06}$ \\
 & MbF-PCA ($10^{-6}$) & $13.52_{2.76}$ & $0.002_{0.002}$ & $58.26_{1.29}$ & $0.03_{0.02}$ & $55.01_{0.9}$ & $0.01_{0.03}$ & $56.15_{1.52}$ & $0.04_{0.04}$ \\
 & INLP & $0.42_{1.25}$ & $\mathbf{0.0_{0.0}}$ & $54.95_{1.51}$ & $\mathbf{0.01_{0.02}}$ & $54.52_{0.7}$ & $\mathbf{0.0_{0.0}}$ & $54.95_{1.51}$ & $\mathbf{0.01_{0.02}}$ \\
 & RLACE & $19.18_{4.03}$ & $0.008_{0.007} $& $63.36_{1.96}$ & $0.1_{0.06} $& $\mathbf{59.64_{3.04}}$ & $0.06_{0.05} $& $62.16_{2.67}$ & $0.07_{0.04}$\\
\cdashline{2-10}
 & Fair PCA & $\mathbf{35.56_{4.52}}$ & $0.019_{0.007}$ & $62.82_{0.88}$ & $0.11_{0.08}$ & $54.55_{0.76}$ & $0.03_{0.04}$ & $60.65_{2.29}$ & $0.13_{0.11}$ \\
 & Fair Kernel PCA & $n/a_{}$ & $n/a_{}$ & $57.8_{1.82}$ & $\mathbf{0.08_{0.06}}$ & $54.74_{1.21}$ & $\mathbf{0.02_{0.04}}$ & $57.67_{1.57}$ & $\mathbf{0.05_{0.04}}$ \\
 & Fair PCA-S (0.5) & $25.11_{5.14}$ & $\mathbf{0.006_{0.004}}$ & $\mathbf{64.1_{1.49}}$ & $0.15_{0.06}$ & $\mathbf{58.08_{2.92}}$ & $0.07_{0.05}$ &  $\mathbf{62.65_{1.75}}$ & $0.14_{0.07}$ \\
 & Fair PCA-S (0.85) & $35.42_{4.49}$ & $0.027_{0.005}$ & $60.93_{0.6}$ & $0.14_{0.04}$ & $55.43_{0.94}$ & $0.06_{0.08}$ & $56.63_{1.09}$ & $0.2_{0.1}$ \\
\hline
\multirow{11}{*}{10} & PCA & $\mathbf{100.0_{0.0}}$ & $\mathbf{0.241_{0.005}}$ & $\mathbf{73.14_{1.16}}$ & $\mathbf{0.21_{0.06}}$ & $\mathbf{64.78_{0.99}}$ & $\mathbf{0.18_{0.05}}$ & $\mathbf{69.81_{1.8}}$ & $\mathbf{0.23_{0.08}}$ \\
\cdashline{2-10}
 & FPCA (0.1, 0.01) & $87.79_{1.21}$ & $0.015_{0.003}$ & $72.25_{0.88}$ & $\mathbf{0.16_{0.06}}$ & $64.7_{1.67}$ & $0.06_{0.05}$ & $\mathbf{69.73_{1.95}}$ & $0.15_{0.07}$ \\
 & FPCA (0, 0.01) & $87.44_{1.28}$ & $0.015_{0.002}$ & $\mathbf{72.32_{0.88}}$ & $\mathbf{0.16_{0.07}}$ & $64.82_{1.64}$ & $\mathbf{0.05_{0.04}}$ & $68.52_{1.25}$ & $0.08_{0.07}$ \\
 & MbF-PCA ($10^{-3}$) & $87.75_{1.29}$ & $\mathbf{0.013_{0.002}}$ & $72.19_{0.88}$ & $\mathbf{0.16_{0.06}}$ & $64.97_{1.53}$ & $0.08_{0.04}$ & $68.89_{1.61}$ & $0.11_{0.06}$ \\
 & MbF-PCA ($10^{-6}$) & $87.75_{1.29}$ & $\mathbf{0.013_{0.002}}$ & $72.19_{0.88}$ & $\mathbf{0.16_{0.06}}$ & $\mathbf{65.01_{1.49}}$ & $0.08_{0.04}$ & $68.14_{1.14}$ & $\mathbf{0.07_{0.05}}$ \\
 & INLP & $\mathbf{91.09_{0.88}}$ & $0.034_{0.005}$ & $71.4_{0.9}$ & $0.17_{0.03}$ & $64.93_{0.84}$ & $0.18_{0.04} $& $68.19_{1.46}$ & $0.2_{0.04}$ \\
 & RLACE & $87.47_{1.27}$ & $0.015_{0.002} $& $72.29_{0.85}$ & $\mathbf{0.16_{0.06}} $& $64.75_{1.73}$ & $\mathbf{0.05_{0.03}} $& $68.3_{2.07}$ & $0.1_{0.07}$ \\
\cdashline{2-10}
 & Fair PCA & $\mathbf{87.44_{1.28}}$ & $\mathbf{0.015_{0.002}}$ & $\mathbf{72.32_{0.9}}$ & $\mathbf{0.16_{0.06}}$ & $64.72_{1.69}$ & $\mathbf{0.05_{0.03}}$ & $67.94_{1.43}$ & $\mathbf{0.09_{0.07}}$ \\
 & Fair Kernel PCA & $n/a_{}$ & $n/a_{}$ & $65.96_{1.12}$ & $0.26_{0.07}$ & $64.33_{0.8}$ & $\mathbf{0.05_{0.04}}$ & $66.41_{1.03}$ & $0.14_{0.07}$ \\
 & Fair PCA-S (0.5) & $\mathbf{87.44_{1.28}}$ & $\mathbf{0.015_{0.002}}$ & $72.31_{0.88}$ & $\mathbf{0.16_{0.06}}$ & $\mathbf{64.75_{1.7}}$ & $\mathbf{0.05_{0.03}}$ & $\mathbf{69.24_{2.02}}$ & $0.12_{0.06}$ \\
 & Fair PCA-S (0.85) & $\mathbf{87.44_{1.28}}$ & $\mathbf{0.015_{0.002}}$ & $72.31_{0.88}$ & $\mathbf{0.16_{0.06}}$ & $\mathbf{64.75_{1.7}}$ & $\mathbf{0.05_{0.03}}$ & $\mathbf{69.24_{2.02}}$ & $0.12_{0.06}$ \\
\hline
\end{tabular}
\end{scriptsize}
\end{table*}

\clearpage

\begin{table*}[t!]
    \centering
\caption{We applied the fair PCA method of \citet{samira2018} to the three real-world datasets considered in Section~\ref{subsec:experiments_linear_guarding}. The tables do not provide the metrics \%Var and MMD$^2$ since the method of \citeauthor{samira2018} is not guaranteed to yield an embedding of the desired target dimension, and hence their method and the 
methods 
studied in 
Section~\ref{subsec:experiments_linear_guarding} are not comparable w.r.t. 
\%Var and MMD$^2$. 
}
    \label{tab:samadi}

    \vspace{3mm}
\renewcommand{\arraystretch}{1.2}

\begin{tabular}{c|c|cccccc}
\hline
 \multicolumn{8}{c}{\normalsize{\textbf{Adult Income} [$\text{feature dim}=97$, $\Psymb(Y=1)=0.2489$]}}\\
  \multirow{2}{*}{$k$} & \multirow{2}{*}{Algorithm} &  \%Acc{\scriptsize ($\uparrow$)} & $\Delta_{DP}${\scriptsize ($\downarrow$)}  & \%Acc{\scriptsize ($\uparrow$)}  & $\Delta_{DP}${\scriptsize ($\downarrow$)}  & \%Acc{\scriptsize ($\uparrow$)}  & $\Delta_{DP}${\scriptsize ($\downarrow$)} \\
    & & \multicolumn{2}{c}{Kernel SVM} & \multicolumn{2}{c}{Linear SVM} & \multicolumn{2}{c}{MLP}\\\hline
2 & Fair PCA of \citet{samira2018} &  $81.59_{1.12}$ & $0.14_{0.03}$ & $ 80.8_{1.09}$ & $0.13_{0.04}$ & $82.02_{1.03}$ & $0.18_{0.04}$ \\
 \hline
 10 & Fair PCA of \citet{samira2018} &  $87.78_{0.87}$ & $0.18_{0.02}$ & $83.2_{1.09}$ & $0.13_{0.03}$ & $91.34_{2.04}$ & $0.18_{0.03}$ \\
 \hline

\hline
 \multicolumn{8}{c}{\normalsize{\textbf{German Credit} [$\text{feature dim}=57$, $\Psymb(Y=1)=0.3020$]}}\\
  \multirow{2}{*}{$k$} & \multirow{2}{*}{Algorithm} &  \%Acc{\scriptsize ($\uparrow$)} & $\Delta_{DP}${\scriptsize ($\downarrow$)}  & \%Acc{\scriptsize ($\uparrow$)}  & $\Delta_{DP}${\scriptsize ($\downarrow$)}  & \%Acc{\scriptsize ($\uparrow$)}  & $\Delta_{DP}${\scriptsize ($\downarrow$)} \\
    & & \multicolumn{2}{c}{Kernel SVM} & \multicolumn{2}{c}{Linear SVM} & \multicolumn{2}{c}{MLP}\\\hline
2 & Fair PCA of \citet{samira2018} &  $73.73_{1.38}$ & $0.05_{0.03}$ & $ 69.97_{0.82}$ & $0.0_{0.01}$ & $72.97_{1.4}$ & $0.04_{0.02}$ \\
 \hline
 10 & Fair PCA of \citet{samira2018} &  $98.73_{0.7}$ & $0.1_{0.07}$ & $76.8_{1.9}$ & $0.09_{0.07}$ & $98.13_{0.69}$ & $0.12_{0.07}$ \\
 \hline

\hline
 \multicolumn{8}{c}{\normalsize{\textbf{COMPAS} [$\text{feature dim}=11$, $\Psymb(Y=1)=0.4548$]}}\\
  \multirow{2}{*}{$k$} & \multirow{2}{*}{Algorithm} &  \%Acc{\scriptsize ($\uparrow$)} & $\Delta_{DP}${\scriptsize ($\downarrow$)}  & \%Acc{\scriptsize ($\uparrow$)}  & $\Delta_{DP}${\scriptsize ($\downarrow$)}  & \%Acc{\scriptsize ($\uparrow$)}  & $\Delta_{DP}${\scriptsize ($\downarrow$)} \\
    & & \multicolumn{2}{c}{Kernel SVM} & \multicolumn{2}{c}{Linear SVM} & \multicolumn{2}{c}{MLP}\\\hline
2 & Fair PCA of \citet{samira2018} &  $63.89_{1.86}$ & $0.2_{0.05}$ & $ 57.94_{2.22}$ & $0.12_{0.05}$ & $63.95_{3.98}$ & $0.18_{0.03}$ \\
 \hline
 10 & Fair PCA of \citet{samira2018} &  $73.12_{1.17}$ & $0.21_{0.06}$ & $64.79_{0.96}$ & $0.18_{0.05}$ & $69.46_{1.04}$ & $0.27_{0.09}$ \\
 \hline
\end{tabular}
\end{table*}

\subsection{Comparison with \citet{samira2018}}\label{app:comparison_samira}

We applied the fair PCA method of \citet{samira2018} to the three real-world datasets considered in Section~\ref{subsec:experiments_linear_guarding}. As discussed in Section~\ref{sec:related_work} and Section~\ref{subsec:experiments_linear_guarding}, the fairness notion underlying the method of \citeauthor{samira2018} is incomparable to our notion of fair PCA. \citeauthor{samira2018} provide theoretical guarantees for an algorithm that relies on solving a semidefinite program (SDP), but then propose to use a multiplicative weight update method for solving the SDP approximately in order to speed up computation. We observed that this can result in embedding dimensions that are much  larger than the desired target dimension. We used the code provided by \citeauthor{samira2018} without modifications; in particular, we used the same parameters for the multiplicative weight update algorithm  as they used in their experiment on the LFW dataset.

Table~\ref{tab:samadi} provides the results. We see that the downstream classifiers trained on the fair PCA representation of \citeauthor{samira2018} have roughly similar values of accuracy and DP violation as standard PCA. Clearly,  the DP violations are much higher than for our methods or the other competitors. 
Since the dimension of the fair PCA representation of \citeauthor{samira2018} is not guaranteed to equal the desired target dimension~$k$, we do not report the metrics \%Var and MMD$^2$ in Table~\ref{tab:samadi}.

\subsection{Fair PCA Applied to the CelebA Dataset}\label{app:CelebA}

Figures~\ref{fig:celeba_experiment_glasses} to~\ref{fig:celeba_experiment_hat} show examples of original CelebA images (top row of each figure) together with the results of applying fair PCA (middle and bottom row) for the various demographic attributes. 
We 
see that fair PCA 
adds something looking like glasses / a mustache / a beard to 
the 
faces, making it hard to tell whether an original face 
features 
those,  
and successfully obfuscates 
the demographic information for these attributes. Still the projected faces resemble the original ones to a good extent. For the attribute ``smiling'', fair PCA also succeeds in obfuscating the demographic information, but the whole faces become more 
perturbed 
and less similar to the original ones. For the attributes ``bald'' and ``hat'', fair PCA 
appears 
to fail, and we can tell for all of the faces under consideration that they \emph{do not} feature baldness / a hat. We suspect that 
the reason for this might be 
the high diversity of hats 
or 
\emph{non-bald} faces
(see 
Figure~\ref{fig:celebA_hat_examples} 
for 
some example images).

\begin{figure*}[t!]
    \centering
    \includegraphics[scale=\scaleBiasMitigation]{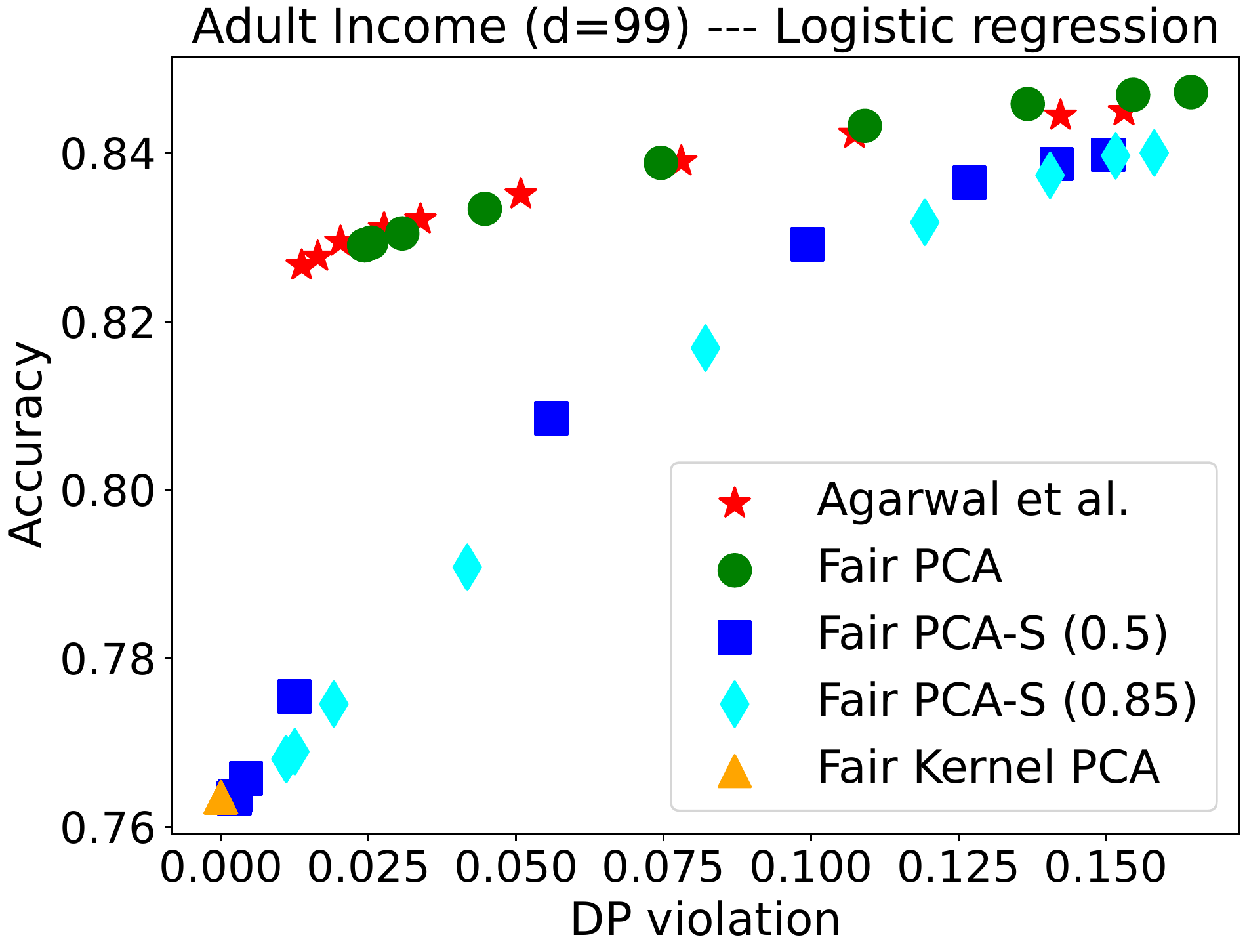}
    \includegraphics[scale=\scaleBiasMitigation]{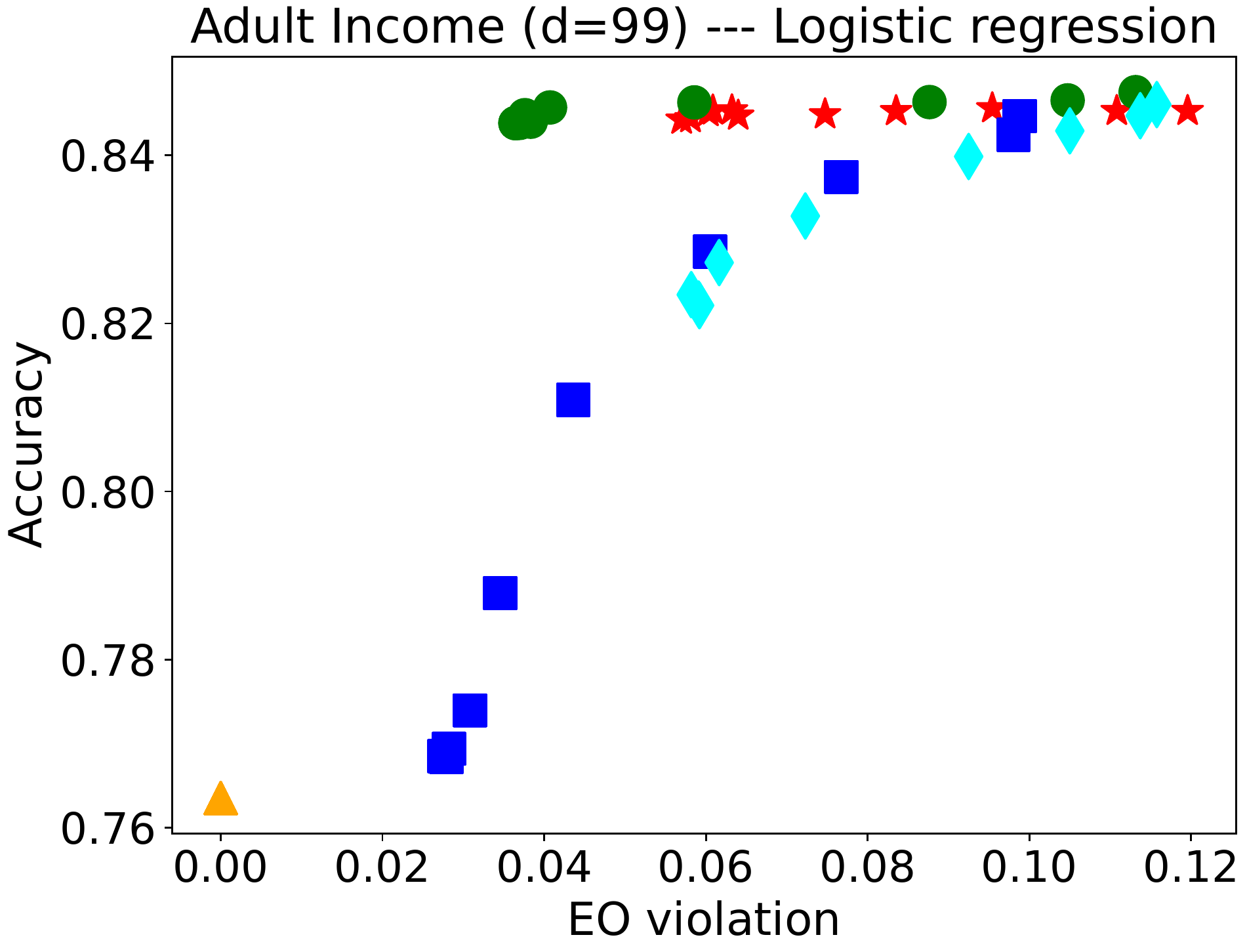}
    \includegraphics[scale=\scaleBiasMitigation]{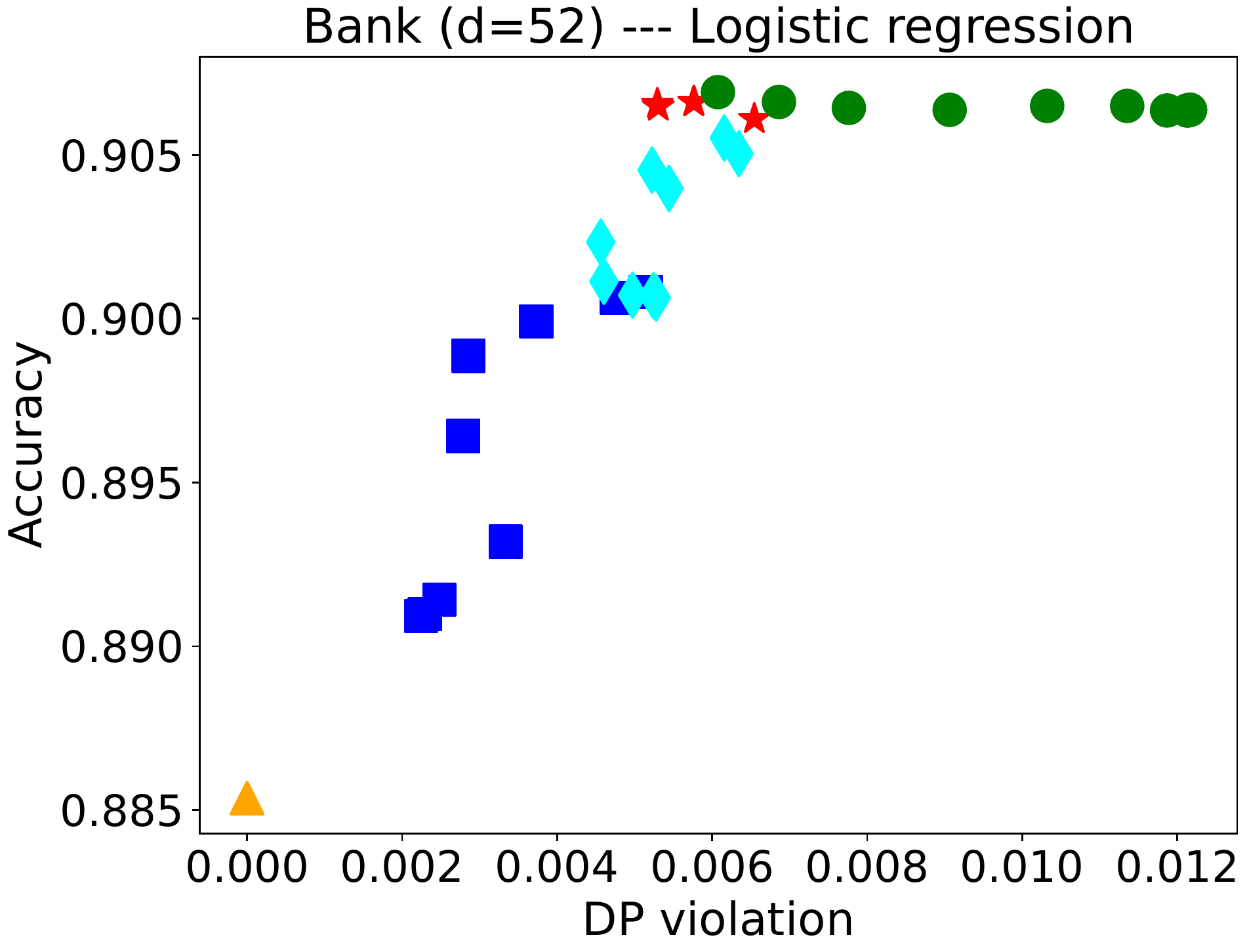}
    \includegraphics[scale=\scaleBiasMitigation]{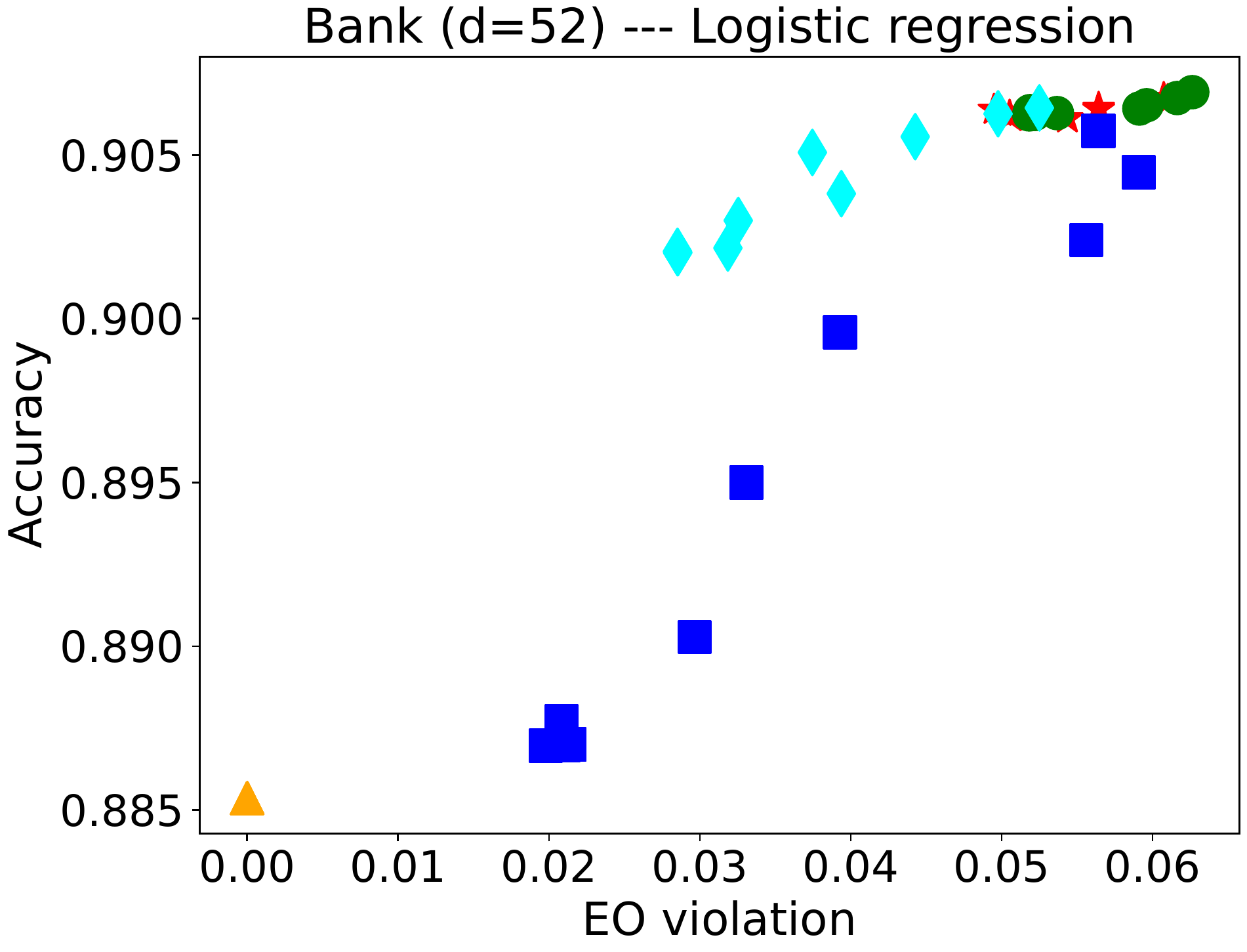}
    \caption{Comparison with the state-of-the-art reductions approach of \citet{agarwal_reductions_approach} 
    when training a logistic regression classifier on the Adult Income dataset (first and second plot) and  
    the Bank Marketing dataset (third and fourth plot). Compared to the plots in Figure~\ref{fig:bias_mitigation_exp}, these plots also show the results for Fair PCA-S and fair kernel PCA.}
    \label{fig:bias_mitigation_exp_Bank_Marketing}
\end{figure*}

\newcommand{\scaleBiasMitigationRuntime}{0.3}

\begin{figure*}[t!]
    \centering
    \includegraphics[scale=\scaleBiasMitigationRuntime]{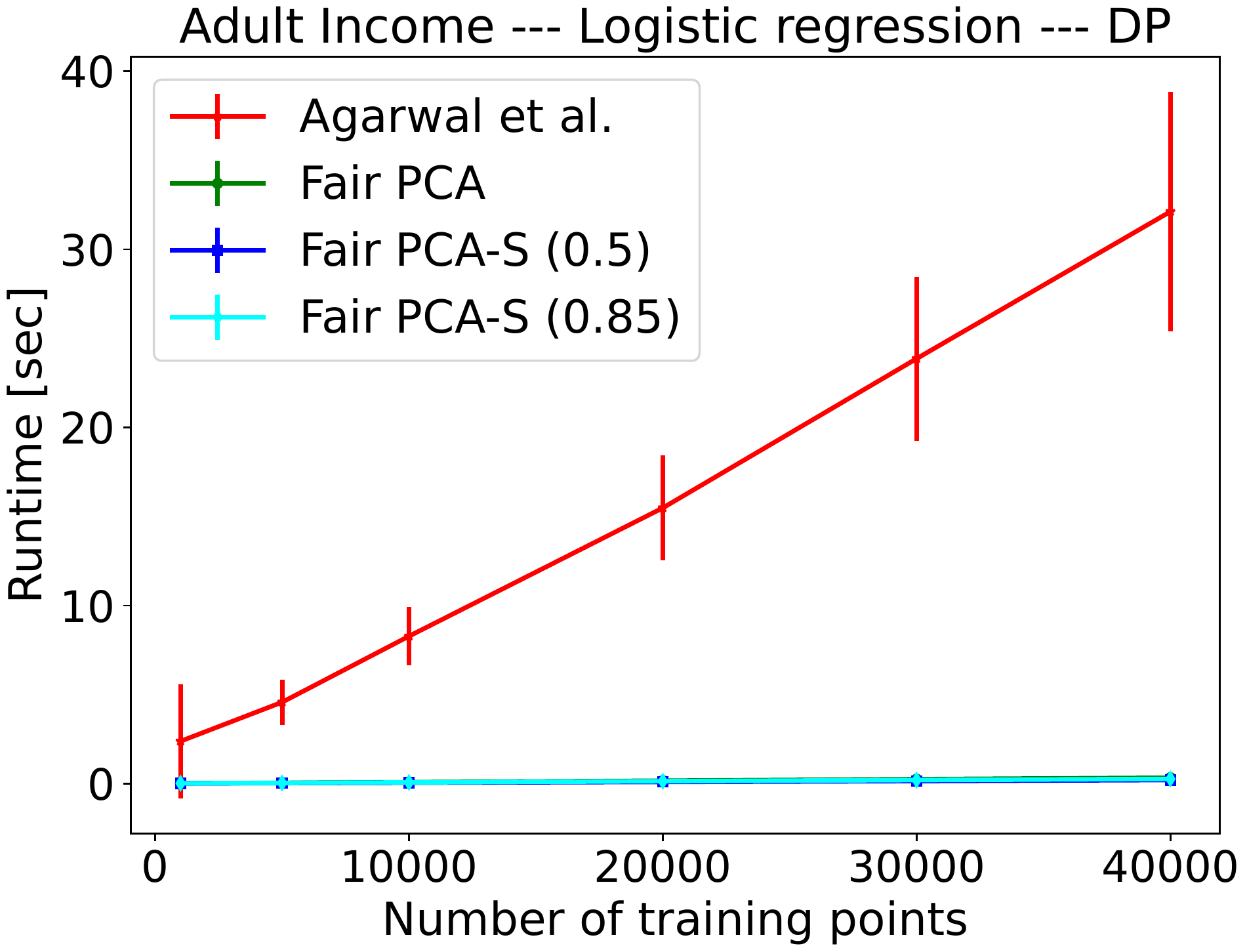}
    \hspace{1cm}
    \includegraphics[scale=\scaleBiasMitigationRuntime]{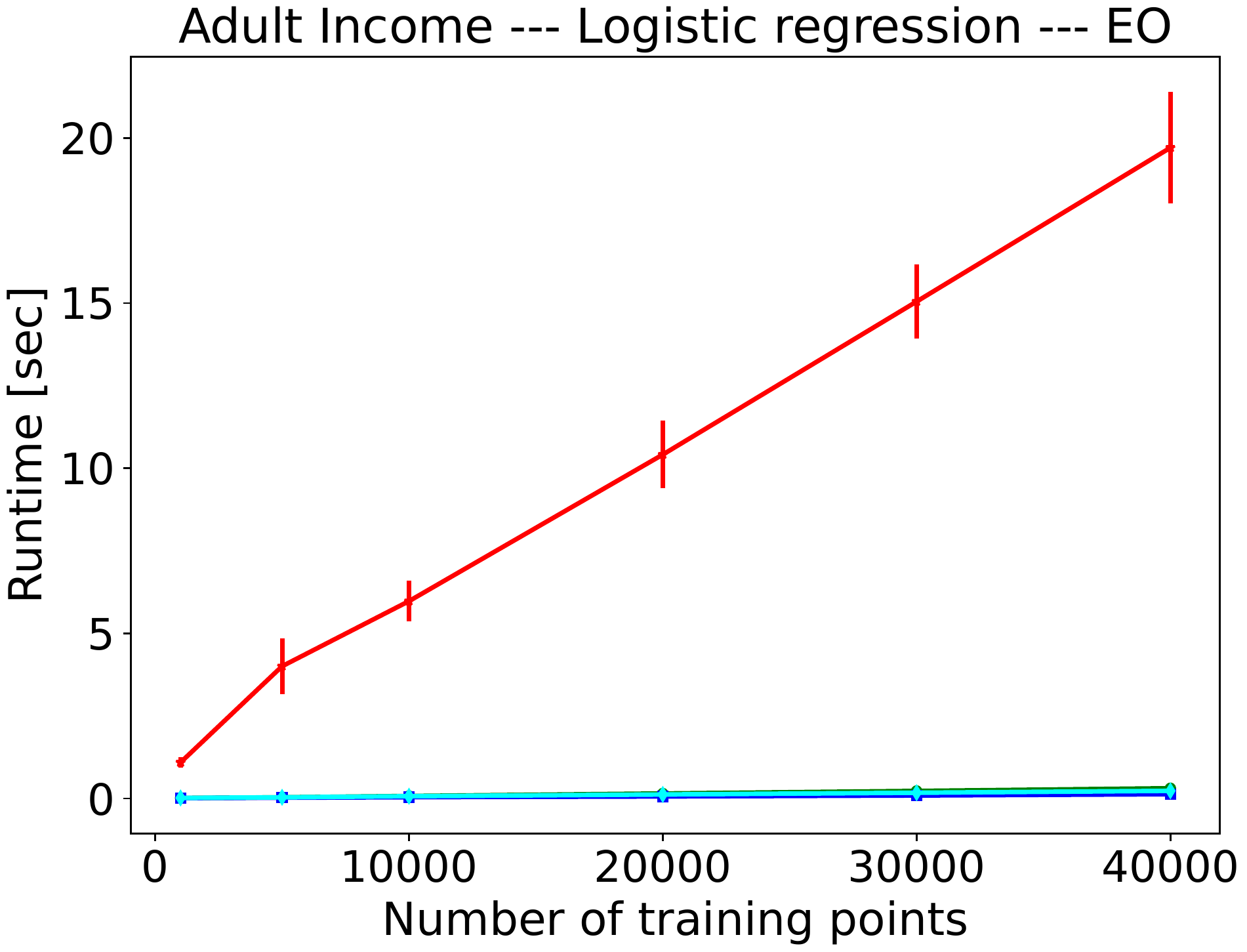}
    
    \caption{Runtime comparison between our methods and the method of \citet{agarwal_reductions_approach}. For our methods, the running time includes the time it takes to fit the logistic regression classifier on top of the fair representation.}
    \label{fig:bias_mitigation_exp_runtime}
\end{figure*}

\subsection{Comparison with 
\citet{agarwal_reductions_approach}}\label{appendix_agarwal_addendum}

Figure~\ref{fig:bias_mitigation_exp_Bank_Marketing} shows the results of the comparison with the reductions approach of \citet{agarwal_reductions_approach} when training a logistic regression classifier
for Fair PCA-S and fair kernel PCA (next to the results for fair PCA and the method of \citeauthor{agarwal_reductions_approach}, which we have already seen in Figure~\ref{fig:bias_mitigation_exp} in Section~\ref{subsec:experiments_bias_mitigation}).
We see that Fair PCA-S produces smooth trade-off curves and can achieve lower fairness violation than fair PCA or the method of \citeauthor{agarwal_reductions_approach} in some cases. 
However, the representation learned by fair kernel PCA only allows for a constant logistic regression classifier (with zero fairness violation and an accuracy equaling the probability of the predominant label---cf. Appendix~\ref{app:details_about_datasets}).

The plots of Figure~\ref{fig:bias_mitigation_exp_runtime} show the running time of the various methods as a function of the number of training points 
on the Adult Income dataset and when training a logistic regression classifier. 
The curves show the average over the eleven values of the fairness parameter / the parameter~\texttt{difference\_bound} (cf.~Appendix~\ref{app:implementation_details}) and over ten random draws of training  data together with the standard deviation as error bars. Note that for our methods the running time includes the time it takes to train the classifier on a representation produced by our method. While none of our methods ever runs for more than 0.5 seconds, the method of \citeauthor{agarwal_reductions_approach}, on average, runs for more than 32 seconds when training with 40000 datapoints and aiming for DP. 
 The plots do not show curves for fair kernel PCA, which we cannot simply apply to this large number of training points due to its cubic running time  in the number of datapoints. We leave it as an interesting question for future work to develop scalable approximation techniques for fair kernel PCA similarly to those that have been developed for standard kernel PCA or other kernel methods \citep[e.g.,][]{williams2000,Kim2005,chin2006}.

\clearpage

\renewcommand{\scaleCelebA}{1.4cm}

\newcommand{\attribute}{glasses}
\begin{figure}
    \centering
    \begin{turn}{90} 
     \begin{minipage}{\scaleCelebA}
    \begin{center}
    \begin{small}
    Original
    \\
    dim$=$6400
    \end{small}
    \end{center}
    \end{minipage}
    \end{turn}
     \hspace{-1pt}
    \includegraphics[height=\scaleCelebA]{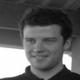}
    \includegraphics[height=\scaleCelebA]{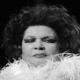}
    \includegraphics[height=\scaleCelebA]{CelebA_experiment/original/053950.jpg}
    \includegraphics[height=\scaleCelebA]{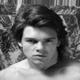}
    \includegraphics[height=\scaleCelebA]{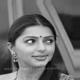}
    \includegraphics[height=\scaleCelebA]{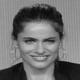}
    \includegraphics[height=\scaleCelebA]{CelebA_experiment/original/050288.jpg}
    \includegraphics[height=\scaleCelebA]{CelebA_experiment/original/016531.jpg}
    \includegraphics[height=\scaleCelebA]{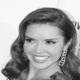}
    \includegraphics[height=\scaleCelebA]{CelebA_experiment/original/094247.jpg}

    \vspace{0mm}
    \par\noindent\rule{\linewidth}{0.6pt}
    \vspace{0mm}
    
    \centering
    \begin{turn}{90} 
    \begin{minipage}{\scaleCelebA}
    \begin{center}
    \begin{small}
    Fair PCA
    \\
    $k=6399$
    \end{small}
    \end{center}
    \end{minipage}
    \end{turn}
    \hspace{-1pt}
    \includegraphics[height=\scaleCelebA]{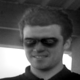}
    \includegraphics[height=\scaleCelebA]{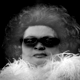}
    \includegraphics[height=\scaleCelebA]{CelebA_experiment/FairPCA_dim6399/\attribute/53950.png}
    \includegraphics[height=\scaleCelebA]{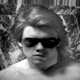}
    \includegraphics[height=\scaleCelebA]{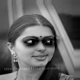}
    \includegraphics[height=\scaleCelebA]{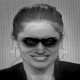}
\includegraphics[height=\scaleCelebA]{CelebA_experiment/FairPCA_dim6399/\attribute/50288.png}
    \includegraphics[height=\scaleCelebA]{CelebA_experiment/FairPCA_dim6399/\attribute/16531.png}
    \includegraphics[height=\scaleCelebA]{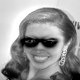}
    \includegraphics[height=\scaleCelebA]{CelebA_experiment/FairPCA_dim6399/\attribute/94247.png}

     \vspace{2mm}
    \centering
    \begin{turn}{90} 
    \begin{minipage}{\scaleCelebA}
    \begin{center}
    \begin{small}
    Fair PCA
    \\
    $k=400$
    \end{small}
    \end{center}
    \end{minipage}
    \end{turn}
    \hspace{-1pt}
     \includegraphics[height=\scaleCelebA]{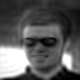}
    \includegraphics[height=\scaleCelebA]{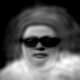}
    \includegraphics[height=\scaleCelebA]{CelebA_experiment/FairPCA_dim400/\attribute/53950.png}
    \includegraphics[height=\scaleCelebA]{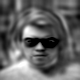}
    \includegraphics[height=\scaleCelebA]{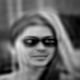}
    \includegraphics[height=\scaleCelebA]{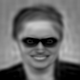}
\includegraphics[height=\scaleCelebA]{CelebA_experiment/FairPCA_dim400/\attribute/50288.png}
    \includegraphics[height=\scaleCelebA]{CelebA_experiment/FairPCA_dim400/\attribute/16531.png}
    \includegraphics[height=\scaleCelebA]{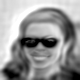}
    \includegraphics[height=\scaleCelebA]{CelebA_experiment/FairPCA_dim400/\attribute/94247.png}

    \caption{
    Fair PCA 
    applied to the CelebA dataset 
    to erase the concept of ``\attribute''. 
    }
    \label{fig:celeba_experiment_\attribute}
\end{figure}

\vspace{1cm}
\renewcommand{\attribute}{mustache}

\vspace{1cm}
\renewcommand{\attribute}{beard}

\vspace{1cm}
\renewcommand{\attribute}{smiling}

\vspace{1cm}
\renewcommand{\attribute}{bald}

\vspace{1cm}
\renewcommand{\attribute}{hat}

\clearpage

\begin{figure}
    \centering
    \includegraphics[height=\scaleCelebA]{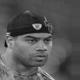}
    \includegraphics[height=\scaleCelebA]{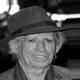}
    \includegraphics[height=\scaleCelebA]{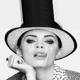}
    \includegraphics[height=\scaleCelebA]{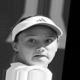}
    \includegraphics[height=\scaleCelebA]{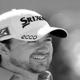}
    \includegraphics[height=\scaleCelebA]{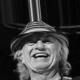}
    \includegraphics[height=\scaleCelebA]{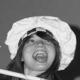}
    \includegraphics[height=\scaleCelebA]{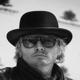}
    \includegraphics[height=\scaleCelebA]{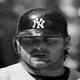}
    \includegraphics[height=\scaleCelebA]{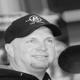}

    \vspace{3mm}
    \includegraphics[height=\scaleCelebA]{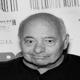}
    \includegraphics[height=\scaleCelebA]{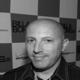}
    \includegraphics[height=\scaleCelebA]{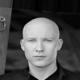}
    \includegraphics[height=\scaleCelebA]{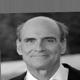}
    \includegraphics[height=\scaleCelebA]{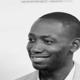}
    \includegraphics[height=\scaleCelebA]{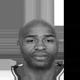}
    \includegraphics[height=\scaleCelebA]{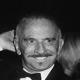}
    \includegraphics[height=\scaleCelebA]{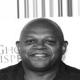}
    \includegraphics[height=\scaleCelebA]{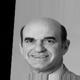}
    \includegraphics[height=\scaleCelebA]{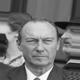}
\caption{Examples of faces in the CelebA dataset that feature a hat (top row) or baldness (bottom row). We can see that the hats are highly diverse. 
Note that 
both types of faces are rare in the dataset: only 4.8\% of the faces feature a hat, and only 2.2\% feature baldness (cf. Figure~\ref{fig:CelebA_distribution_attributes}).}\label{fig:celebA_hat_examples}
\end{figure}

\end{document}